\documentclass[backref=page]{amsart}

\newcommand{\druckVersion}{0} 
\usepackage[english]{babel}
\usepackage{amssymb}
\usepackage{mathtools}
\usepackage{amsmath}
\usepackage{bbm}
\usepackage{graphicx}
\usepackage{caption}
\usepackage{listings}
\usepackage{color}
\usepackage{transparent}
\usepackage{xcolor}
\usepackage{textcomp}
\usepackage[numbers]{natbib}
\usepackage[shortlabels]{enumitem} 
\usepackage[autostyle=false, style=english]{csquotes}
\MakeOuterQuote{"}
\usepackage{mathrsfs}
\usepackage{eurosym}
\usepackage{textcomp}
\usepackage{amsthm}
\usepackage{mdframed}
\usepackage{xpatch}
\usepackage{xifthen}
\usepackage{pdfpages}
\usepackage[percent]{overpic}
\usepackage{braket}
\usepackage{hyperref}
\usepackage[nameinlink]{cleveref}
\usepackage{tablefootnote}
\usepackage[usestackEOL]{stackengine}
\usepackage{graphicx}
\usepackage{etoolbox}

\usepackage{crossreftools} 

\definecolor{meinCiteGruen}{rgb}{0.0, 0.7, 0.05}
\definecolor{meineLinkFarbe}{rgb}{0.0, 0.0, 0.61}
\definecolor{orange}{rgb}{1, 0.5, 0.00}

\ifthenelse{\equal{\druckVersion}{1}}{
	\newcommand*{\meinGreyscale}{0.5}
	\newcommand*{\meineTranzparenz}{0.5} 
	\definecolor{hellgrau}{rgb}{\meinGreyscale, \meinGreyscale, \meinGreyscale}
	\hypersetup{hidelinks}
}{
	\newcommand*{\meinGreyscale}{0.6}
	\newcommand*{\meineTranzparenz}{0.4}
	\definecolor{hellgrau}{rgb}{\meinGreyscale, \meinGreyscale, \meinGreyscale}
	\hypersetup{
		colorlinks   = true, 
		urlcolor     = blue, 
		linkcolor    = meineLinkFarbe, 
		citecolor   = meinCiteGruen, 
	}
}

\hypersetup{bookmarksdepth=3, %
	bookmarksopenlevel=2} 
	
\makeatletter 
\AfterEndEnvironment{mdframed}{%
	\tfn@tablefootnoteprintout%
	\gdef\tfn@fnt{0}%
}
\makeatother

\makeatletter
\patchcmd{\@makefnmark}{\fontsize}{\check@mathfonts\fontsize}{}{}
\makeatother

\let\oldnormalcolor\normalcolor
\renewcommand{\normalcolor}{\oldnormalcolor\transparent{1}}

\crefalias{subequation}{equation}
\mathtoolsset{centercolon}

\newcommand{\mydot}{\,{\color{hellgrau}.}}
\newcommand{\mycomma}{\,{\color{hellgrau},}}

\newcommand*{\meinfullref}[1]{\hyperref[{#1}]{\Cref*{#1} \nameref*{#1}}}
\newcommand*{\unameref}[2][{}]{\hyperref[{#2}]{ \textcolor{black}{\nameref*{#2}{#1}}}}
\newcommand{\meqref}[1]{\labelcref{#1}}

\newcommand*{\meinref}[2][{}]{\hyperref[{#2}]{{#1}~\labelcref*{#2}}}

\newcommand*\colvec[3][]{
	\begin{pmatrix}\ifx\relax#1\relax\else#1\\\fi#2\\#3\end{pmatrix}
}
\newcommand*\colvecs[3][]{
	\begin{psmallmatrix}\ifx\relax#1\relax\else#1\\\fi#2\\#3\end{psmallmatrix}
}

\let\emptyset\varnothing

\newcommand{\unimportant}[1]{{\transparent{\meineTranzparenz}{#1}}}

\newcommand{\R}{\ensuremath{\mathbb{R}}}
\newcommand{\C}{\ensuremath{\mathcal{C}}}
\newcommand{\U}{\ensuremath{\mathcal{U}}}
\newcommand{\Rp}{\ensuremath{\R_{>0}}}
\newcommand{\Rpz}{\ensuremath{\R_{\geq 0}}}
\newcommand{\N}{\ensuremath{\mathbb{N}}}

\newcommand{\B}{\ensuremath{\mathfrak{B}}}
\newcommand{\E}[1][]{\ensuremath{\mathbb{E}\ifthenelse{\equal{#1}{}}{}{\left[ {#1}\right] }}}
\newcommand{\Eco}[2]{\ensuremath{\E\left[{#1}\middle|{#2}\right]}} 
\newcommand{\PP}[1][]{\ensuremath{\mathbb{P}\ifthenelse{\equal{#1}{}}{}{\left[ {#1}\right] }}}
 
\newcommand{\jac}{\ensuremath{J}}
\newcommand{\jacj}[1][\cdot]{\ensuremath{\jac^{j}\left(#1\right)}}
\newcommand{\jacjat}[2][\cdot]{\ensuremath{\jac^{j}\left(#1\right)\Big|_{{#2}}}}
\newcommand{\X}{\ensuremath{\mathcal{X}}}
\newcommand{\Y}{\ensuremath{\mathcal{Y}}}

\newcommand{\om}{\ensuremath{{\color{hellgrau}\omega}}}
\newcommand{\Om}{\ensuremath{\Omega}}
\newcommand{\omb}{\ensuremath{{\color{hellgrau}(\om)}}}
\newcommand{\Wkp}[3][K]{\ensuremath{\href{https://en.wikipedia.org/w/index.php?title=Sobolev_space&oldid=910223537}{W^{{#2},{#3}}({#1})}}}
\newcommand{\Woi}[1][K]{\Wkp[#1]{1}{\infty}}
\newcommand{\Wop}[1][{K,\nu}]{\Wkp[#1]{1}{p}}

\newcommand{\Wt}{W^2}
\newcommand{\Wti}[1][K]{\Wkp[{#1}]{2}{\infty}}
\newcommand{\Hk}[2][K]{\ensuremath{\href{https://en.wikipedia.org/w/index.php?title=Sobolev_space&oldid=910223537}{H^{{#2}}({#1})}}}
\newcommand{\WkpShort}[2]{\ensuremath{\href{https://en.wikipedia.org/w/index.php?title=Sobolev_space&oldid=910223537}{W^{{#1},{#2}}}}}
\newcommand{\iid}{\href{https://en.wikipedia.org/w/index.php?title=Independent_and_identically_distributed_random_variables&oldid=910267759}{\oldnormalcolor i.i.d.}}
\newcommand{\dx}[1][x]{\,d{#1}}
\newcommand{\din}{{d}}  
\newcommand{\Rdin}{\ensuremath{\mathbb{R}^{\din}}}
\newcommand{\dout}{{d_{\tiny \text{out}}}}
\newcommand{\Rdout}{\ensuremath{\mathbb{R}^{\dout}}}
\newcommand{\LtwoS}[1][{\Woi[{\tilde{K}}]}]{\ensuremath{L^2(\Sd,{#1})}}
\newcommand{\Ltwo}[1][{\tilde{K}}]{\ensuremath{L^2({#1})}}

\newcommand{\fao}{\ensuremath{\forall\omega\in\Om}}
\newcommand{\faog}{\ensuremath{{\color{hellgrau}\fao}}}

\newcommand{\faxd}{\ensuremath{\forall x\in\Rdin}}
\newcommand{\faxdg}{\ensuremath{{\color{hellgrau}\faxd}}}
\newcommand{\fromto}[2][1]{\ensuremath{\left\lbrace {#1},\dots,{#2}\right\rbrace }} 
\newcommand{\allIndi}[3][1]{\ensuremath{\forall {#3}\in\fromto[{#1}]{#2}}} 

\newcommand{\relu}[1][\cdot]{\ensuremath{\max\left( 0,{#1}\right) }}
\newcommand{\ind}{\ensuremath{\mathbbm{1}}}
\newcommand{\sobnorm}[1][\cdot]{\ensuremath{\left\| {#1}\right\|_{\Woi}}}
\newcommand{\sobnormmulti}[1][\cdot]{\ensuremath{\left\| {#1}\right\|_{\Woi[K,\Rdout]}}}
\newcommand{\sobnormop}[1][\cdot]{\ensuremath{\left\| {#1}\right\|_{\Wop}}}
\newcommand{\norm}[2][\cdot]{\ensuremath{\left\| {#1}\right\|_{{#2}}}}
\newcommand{\supnorm}[1][\cdot]{\ensuremath{\left\| {#1}\right\|_{L^\infty(K)}}}
\newcommand{\supnormon}[2][\cdot]{\ensuremath{\left\| {#1}\right\|_{L^\infty(#2)}}}
\newcommand{\supnormmulti}[1][\cdot]{\ensuremath{\supnormon[{#1}]{K,\Rdout}}}
\newcommand{\twonorm}[1][\cdot]{\ensuremath{\left\| {#1}\right\|_2}}

\newcommand{\Ltwonormons}[2][\cdot]{\ensuremath{\left\| {#1}\right\|_{L^2(#2)}^2}}
\newcommand{\LtwonormonSd}[2][\cdot]{\ensuremath{\left\| {#1}\right\|_{\LtwoS[{#2}]}}}
\newcommand{\plim}[1][n\to\infty]{\ensuremath{\Plim_{#1}}}
\newcommand{\nlim}[1][n\to\infty]{\ensuremath{\lim\limits_{#1}}}
\newcommand{\Frobnorm}[1][\cdot]{\ensuremath{\left\| {#1}\right\|_F}}

\newcommand{\tlim}{\ensuremath{\lim_{T\to\infty}}}

\newcommand{\skprod}[2]{\ensuremath{\left\langle {#1},{#2} \right\rangle}}
\newcommand{\notationOf}[2][ImplRegPart1V3]{\ensuremath{\left( {#2} \right)_{\text{\cite{#1}}}}}
\newcommand{\Linfty}[1][K]{\ensuremath{L^\infty(#1)}}
\newcommand{\Linftymultidim}[1][K,\Rdout]{\ensuremath{\Linfty[#1]}}
\newcommand{\Wtimultidim}[1][K,\Rdout]{\ensuremath{\Wti[#1]}}
\newcommand{\Woimultidim}[1][K,\Rdout]{\ensuremath{\Woi[#1]}}
\newcommand{\T}{\ensuremath{\hyperref[def:adaptedIGAM]{\mathcal{T}}}}
\newcommand{\Tw}{\ensuremath{\hyperref[eq:Tw]{\tilde{\mathcal{T}}}}}

\newcommand{\Cgl}[1][g]{\ensuremath{\hyperref[rem:compactSupp]{C_{#1}^\ell}}}
\newcommand{\Cgu}[1][g]{\ensuremath{\hyperref[rem:compactSupp]{C_{#1}^u}}}

\newcommand{\xtr}{\ensuremath{x^{\text{\tiny train}}}}

\newcommand{\ytr}{\ensuremath{y^{\text{\tiny train}}}}

\newcommand{\lw}{\ensuremath{\hyperref[thm:ridgeToaIGAM]{\tilde{\lambda}}}}
\newcommand{\lwnl}{\ensuremath{\tilde{\lambda}}}
\newcommand{\Ltr}{\ensuremath{\hyperref[as:generalloss]{L}}}

\newcommand{\ltr}{\ensuremath{l}} 
\newcommand{\ltri}[1][i]{\ensuremath{{l_{#1}}}}

\newcommand{\Ltrb}[1]{\ensuremath{\Ltr\left( {#1}\right) }}

\newcommand{\ltricon}[1][i]{\ensuremath{\hyperref[as:lossconcat]{l_{#1}\textsuperscript{con}}}}

\newcommand{\Pgpmm}[1][g]{\ensuremath{\hyperlink{eq:Pgpmm}{P_{\pm}^{{\color{hellgrau} {#1}}}}}}

\newcommand{\Plpm}[1][g]{\ensuremath{\hyperref[eq:Plpm]{P_{+-}^{{\color{hellgrau}\lambda\ifthenelse{\equal{#1}{}}{}{,} {#1}}}}}}

\newcommand{\PgSp}[1][g]{\ensuremath{\hyperlink{eq:PgSp}{P_{\phi}^{{\color{hellgrau} {#1}}}}}}
\newcommand{\PgSpmb}[2][g]{\ensuremath{\PgSp[{#1}]\left( {#2}\right) }}
\newcommand{\PgSpm}[1][g]{\ensuremath{\PgSp[{#1}] }}
\newcommand{\Kw}{\ensuremath{{\tilde{{K}}}}}
\newcommand{\Fn}[1][n]{\ensuremath{\hyperref[def:ridgeNet]{F_{#1}^{{\color{hellgrau}\lw}}}}} 
\newcommand{\Fnb}[2][n]{\ensuremath{\Fn[{#1}]\left( {#2}\right) }}
\newcommand{\Fl}[1][]{\ensuremath{\hyperref[def:splineReg]{F^{{\color{hellgrau}\lambda}\ifthenelse{\equal{#1}{}}{}{,} {#1}}}}} 
\newcommand{\Flpmm}[1][g]{\ensuremath{\hyperref[def:adaptedSplineReg]{F_{\pm}^{{\color{hellgrau}\lambda\ifthenelse{\equal{#1}{}}{}{,} {#1}}}}}}

\newcommand{\Flpm}[1][g]{\ensuremath{\hyperref[def:adaptedSplineReg]{F_{+-}^{{\color{hellgrau}\lambda\ifthenelse{\equal{#1}{}}{}{,} {#1}}}}}}

\newcommand{\Flpmasym}[1][g_+,g_-]{\ensuremath{\hyperref[def:asymadaptedSplineReg]{F_{+-}^{{\color{hellgrau}\lambda\ifthenelse{\equal{#1}{}}{}{,} {#1}}}}}}

\newcommand{\wR}{\ensuremath{\hyperref[def:ridgeNet]{w^{*{\color{hellgrau},\lwnl}}}}}
\newcommand{\wRk}[1][k]{\ensuremath{\hyperref[def:ridgeNet]{w_{#1}^{*{\color{hellgrau},\lwnl}}}}}
\newcommand{\wRl}[1][\lwnl]{\ensuremath{\hyperref[def:ridgeNet]{w^{*{\color{hellgrau},{#1}}}}}}
\newcommand{\wRo}{\ensuremath{\wR\omb}}

\newcommand{\ww}{\ensuremath{\hyperref[eq:ww]{\tilde{w}}}}
\newcommand{\wwo}{\ensuremath{\ww\omb}}
\newcommand{\wwp}{\ensuremath{\hyperref[par:wwp]{\ww^+}}}
\newcommand{\wwm}{\ensuremath{\hyperref[par:wwp]{\ww^-}}}

\newcommand{\wwstar}{\ensuremath{\hyperref[eq:wwstar]{\tilde{w}^{*{\color{hellgrau},\lwnl}}}}}

\newcommand{\wt}[1][T]{\ensuremath{\hyperref[def:GDsolution]{w^{#1}}}}
\newcommand{\wto}[1][T]{\ensuremath{\hyperref[def:GDsolution]{w^{{#1}}}\omb}}
\newcommand{\wth}[1][T]{\ensuremath{\hyperref[def:GDsolution]{\check{w}^{#1}}}}
\newcommand{\wtho}[1][T]{\ensuremath{\hyperref[def:GDsolution]{\check{w}^{{#1}}}\omb}}

\newcommand{\vw}{\ensuremath{\tilde{v}}}

\newcommand{\RN}{\ensuremath{\hyperref[def:RSNN]{\mathcal{RN}}}}
\newcommand{\RNnl}{\ensuremath{\mathcal{RN}}}
\newcommand{\RNw}[1][w]{\ensuremath{\RN_{#1}}}
\newcommand{\RNwo}[1][w]{\ensuremath{\RN_{{#1}{\color{hellgrau},\om}}}}
\newcommand{\RNR}[1][{\lwnl}]{\ensuremath{\hyperref[def:ridgeNet]{\RNnl^{*{\color{hellgrau},{#1}}}}}}
\newcommand{\RNRo}[1][{\lwnl}]{\ensuremath{\hyperref[def:ridgeNet]{\RNnl_{\om}^{*{\color{hellgrau},{#1}}}}}}

\newcommand{\FDRNR}[1][{\lwnl}]{\ensuremath{\hyperref[def:FDridgeNet]{\RNnl_{\ShdSymb}^{*{\color{hellgrau},{#1}}}}}}
\newcommand{\FDRNRo}[1][{\lwnl}]{\ensuremath{\hyperref[def:FDridgeNet]{\RNnl_{\ShdSymb,\om}^{*{\color{hellgrau},{#1}}}}}}
\newcommand{\FDRNw}[1][w]{\ensuremath{\RN_{\ShdSymb,{#1}}}}
\newcommand{\FDRNwo}[1][w]{\ensuremath{\RN_{\ShdSymb,{#1}{\color{hellgrau},\om}}}}

\newcommand{\fl}[1][\lambda]{\ensuremath{\hyperref[def:splineReg]{f^{*{\color{hellgrau},{#1}}}}}}
\newcommand{\flg}{\ensuremath{\hyperref[def:splineReg]{f_g^{*{\color{hellgrau},\lambda}}}}}

\newcommand{\flp}[1][g]{\ensuremath{\hyperref[def:adaptedSplineReg]{f^{*{\color{hellgrau},\lambda}}_{{#1},+}}}}
\newcommand{\flm}[1][g]{\ensuremath{\hyperref[def:adaptedSplineReg]{f^{*{\color{hellgrau},\lambda}}_{{#1},-}}}}
\newcommand{\flpm}[1][g]{\ensuremath{\hyperref[def:adaptedSplineReg]{f^{*{\color{hellgrau},\lambda}}_{{#1},\pm}}}}

\newcommand{\fLpm}[1][\lambda]{\ensuremath{\hyperref[def:adaptedSplineReg]{f^{*{\color{hellgrau},{#1}}}_{g,\pm}}}}
\newcommand{\flpasym}[1][g_+,g_-]{\ensuremath{\hyperref[def:asymadaptedSplineReg]{f^{*{\color{hellgrau},\lambda}}_{{#1},+}}}}
\newcommand{\flmasym}[1][g_+,g_-]{\ensuremath{\hyperref[def:asymadaptedSplineReg]{f^{*{\color{hellgrau},\lambda}}_{{#1},-}}}}
\newcommand{\flpmasym}[1][g_+,g_-]{\ensuremath{\hyperref[def:asymadaptedSplineReg]{f^{*{\color{hellgrau},\lambda}}_{{#1},\pm}}}}
\newcommand{\gammaAsym}[1][g_+,g_-]{\ensuremath{\hyperref[def:asymadaptedSplineReg]{\gamma^{*{\color{hellgrau},\lambda}}_{{#1}}}}}

\newcommand{\phin}{\ensuremath{\phi^n}}
\newcommand{\fn}{\ensuremath{f^n}}
\newcommand{\un}{\ensuremath{u^n}}

\newcommand{\Sd}[1][d-1]{\ensuremath{\hyperref[def:Sd]{S^{#1}}}}
\newcommand{\Shd}[1][d-1]{\ensuremath{\check{S}^{#1}}}
\newcommand{\sh}{\ensuremath{\check{s}}}
\newcommand{\SdSymb}[1][d-1]{\ensuremath{\circ}}
\newcommand{\ShdSymb}[1][d-1]{\ensuremath{\check{\circ}}}
\newcommand{\Th}{\ensuremath{\hyperref[def:adaptedGAM]{\T_{\ShdSymb}}}}

\newcommand{\gbar}{\ensuremath{\hyperlink{eq:gbar}{\bar{g}}}}
\newcommand{\gbarhat}{\ensuremath{\hyperlink{eq:gbarhat}{\check{\bar{g}}}}}

\newcommand{\gxis}{\ensuremath{\hyperref[def:kinkPosDenscond]{g_{\xi|s}}}}
\newcommand{\gsr}[2][r]{\ensuremath{
\newcommand{\gs}[1][s]{\ensuremath{
\newcommand{\ghsr}[2][r]{\ensuremath{
\newcommand{\ghs}[1][\sh]{\ensuremath{

\newcommand{\flGAM}[1][\lambda]{\ensuremath{\hyperref[def:GAM]{f_{\ShdSymb}^{*{\color{hellgrau},{#1}}}}}}
\newcommand{\flIGAM}[1][\lambda]{\ensuremath{\hyperref[def:IGAM]{f_{\SdSymb}^{*{\color{hellgrau},{#1}}}}}}
\newcommand{\FlIGAMm}{\ensuremath{\hyperref[def:GAM]{F_{\SdSymb}^{{\color{hellgrau}\lambda}}}}}
\newcommand{\FlGAMm}{\ensuremath{\hyperref[def:GAM]{F_{\ShdSymb}^{{\color{hellgrau}\lambda}}}}}
\newcommand{\FlIGAMmb}[1]{\ensuremath{\FlIGAMm\left( {#1}\right) }}
\newcommand{\FlGAMmb}[1]{\ensuremath{\FlGAMm\left( {#1}\right) }}
\newcommand{\PIGAMm}{\ensuremath{\hyperlink{eq:PIGAMm}{P_{\SdSymb}}}}
\newcommand{\PGAMm}{\ensuremath{\hyperlink{eq:PGAMm}{P_{\ShdSymb}}}}

\newcommand{\Cgsl}[1][\gs]{\ensuremath{\hyperref[rem:compactSupp]{C_{#1}^\ell}}}
\newcommand{\Cgsu}[1][\gs]{\ensuremath{\hyperref[rem:compactSupp]{C_{#1}^u}}}

\newcommand{\psr}[2][r]{\ensuremath{
\newcommand{\ps}[1][s]{\ensuremath{
\newcommand{\FlSm}[1][g]{\ensuremath{\hyperref[def:adaptedGAM]{F_{\SdSymb}^{{\color{hellgrau}\lambda\ifthenelse{\equal{#1}{}}{}{,} {#1}}}}}}
\newcommand{\FlShm}[1][\check{g}]{\ensuremath{\hyperref[def:adaptedGAM]{F_{\ShdSymb}^{{\color{hellgrau}\lambda\ifthenelse{\equal{#1}{}}{}{,} {#1}}}}}}
\newcommand{\FlSmb}[2][g]{\ensuremath{\FlSm[{#1}]\left( {#2}\right) }}
\newcommand{\FlShmb}[2][\check{g}]{\ensuremath{\FlShm[{#1}]\left( {#2}\right) }}
\newcommand{\PgSm}[1][g]{\ensuremath{\hyperlink{eq:PgSm}{P_{\SdSymb}^{{\color{hellgrau} {#1}}}}}}
\newcommand{\PgShm}[1][\check{g}]{\ensuremath{\hyperlink{eq:PgShm}{P_{\ShdSymb}^{{\color{hellgrau} {#1}}}}}}
\newcommand{\flS}[1][g]{\ensuremath{\hyperref[def:adaptedIGAM]{f^{*{\color{hellgrau},\lambda}}_{{#1},\SdSymb}}}}
\newcommand{\fzS}[1][g]{\ensuremath{\hyperref[def:adaptedGAM]{f^{*{\color{hellgrau},0+}}_{{#1},\SdSymb}}}}
\newcommand{\flSh}[1][\check{g}]{\ensuremath{\hyperref[def:adaptedGAM]{f^{*{\color{hellgrau},\lambda}}_{{#1},\ShdSymb}}}}

\newcommand{\Ush}{\ensuremath{U(\sh)}}
\newcommand{\mUsh}[1][\sh]{\ensuremath{\mu\left(U({#1})\right)}}
\newcommand{\psrstar}[2][r]{\ensuremath{
\newcommand{\psstar}[1][s]{\ensuremath{
\newcommand{\fcw}[1][g]{\ensuremath{\hyperref[def:aIGAMtoaGAM]{\tilde{\check{f_{#1}}}^{\color{hellgrau}\lambda}}}}
\newcommand{\psrw}[2][r]{\ensuremath{
\newcommand{\psw}[1][s]{\ensuremath{
\newcommand{\fw}{\ensuremath{\hyperref[def:aGAMtoaIGAM]{\tilde{f}_{{\color{hellgrau}\ghs[],\ShdSymb}}^{{\color{hellgrau}\lambda}}}}}
\newcommand{\psrcw}[2][r]{\ensuremath{
\newcommand{\pscw}[1][s]{\ensuremath{
\newcommand{\psrcstar}[2][r]{\ensuremath{
\newcommand{\pscstar}[1][s]{\ensuremath{

\newcommand{\link}{\hyperref[def:linkfunction]{\ensuremath{\ell}}}
\newcommand{\linkinv}{\ensuremath{\hyperref[def:linkfunction]{\ell}^{-1}}}
\newcommand{\nua}{\ensuremath{\nu_{\text{a}}}}
\newcommand{\nuc}{\ensuremath{\nu_{\text{c}}}}

\newcommand{\mfootmark}{\text{\footnotemark}}
\newcommand{\mfootref}[1]{\text{{ushape\textsuperscript{\ref{#1}}}}}

\newcommand{\regSpl}{\hyperref[def:splineReg]{regression spline}}
\newcommand{\regSplf}{\hyperref[def:splineReg]{regression spline~\ensuremath{\fl}}}
\newcommand{\wregSpl}{\hyperref[def:splineReg]{weighted regression spline}}
\newcommand{\wregSplf}{\hyperref[def:splineReg]{weighted regression spline~\ensuremath{\flg}}}
\newcommand{\aregSpl}{\hyperref[def:adaptedSplineReg]{adapted regression spline}}
\newcommand{\asymaregSpl}{\hyperref[def:asymadaptedSplineReg]{asymmetric adapted regression spline}}
\newcommand{\aregSplf}{\hyperref[def:adaptedSplineReg]{adapted regression spline~\ensuremath{\flpm}}}
\newcommand{\RSN}{\hyperref[def:RSNN]{RSN}}
\newcommand{\RSNlong}{\hyperref[def:RSNN]{randomized shallow neural network}}
\newcommand{\RSNlongAndShort}{\hyperref[def:RSNN]{\textbf{r}andomized \textbf{s}hallow neural \textbf{n}etwork (\RSN)}}
\newcommand{\wRRSN}{wR\RSN}
\newcommand{\wRRSNlong}{wide ReLU \hyperref[def:RSNN]{randomized shallow neural network}}
\newcommand{\wRRSNlongAndShort}{\textbf{w}ide \textbf{R}eLU \hyperref[def:RSNN]{\textbf{r}andomized \textbf{s}hallow neural \textbf{n}etwork} (\wRRSN)}
\newcommand{\wlargeRRSNlongAndShort}{\textbf{w}ide \unimportant{(large number of neurons~$n$)} \textbf{R}eLU \hyperref[def:RSNN]{\textbf{r}andomized \textbf{s}hallow neural \textbf{n}etwork} (\wRRSN)}
\newcommand{\ridgeRSN}{\hyperref[def:ridgeNet]{$\ell_2$-regularized RSN}}
\newcommand{\FDridgeRSN}{\hyperref[def:FDridgeNet]{FD $\ell_2$-regularized RSN}}
\newcommand{\aGAM}{\hyperref[def:adaptedGAM]{adapted generalized additive model}}
\newcommand{\aIGAM}{\hyperref[def:adaptedIGAM]{adapted infinite generalized additive model}}
\newcommand{\aIGAMf}{\hyperref[def:adaptedIGAM]{adapted infinite generalized additive model \flS}}
\newcommand{\GAM}{\hyperref[def:GAM]{generalized additive model}}
\newcommand{\IGAM}{\hyperref[def:IGAM]{infinite generalized additive model}}
\newcommand{\IGAMf}{\hyperref[def:IGAM]{infinite generalized additive model \flIGAM}}
\newcommand{\GAMs}{\hyperref[def:GAM]{GAM}}
\newcommand{\IGAMs}{\hyperref[def:IGAM]{IGAM}}
\newcommand{\aGAMs}{\hyperref[def:adaptedGAM]{adapted GAM}}
\newcommand{\aIGAMs}{\hyperref[def:adaptedIGAM]{adapted IGAM}}
\newcommand{\proofInSec}[2]{{\transparent{\meineTranzparenz}
		\begin{proof}
			The \hyperlink{proof:#1}{proof of \Cref*{#1}} is formulated in \Cref{#2}.
\end{proof}}}
\newcommand{\proofInSource}[2]{{\transparent{\meineTranzparenz}
		\begin{proof}
			For a proof of \Cref*{#1} see {#2}.
\end{proof}}}
\mdfdefinestyle
{meinTheoremstyle}{
	linewidth=0pt,
	backgroundcolor =  hellgrau,
	innerleftmargin=0pt,
	innerrightmargin=0pt,
	innertopmargin=-5pt,
	innerbottommargin=0pt}
\theoremstyle{plain}
\newtheorem{theorem}{Theorem}[section]
\newmdtheoremenv[style=meinTheoremstyle]{satz}[theorem]{Satz}
\newtheorem{lemma}[theorem]{Lemma}

\theoremstyle{remark}
\newtheorem{remark}[theorem]{Remark}
\theoremstyle{definition}
\newtheorem{definition}[theorem]{Definition}
\newtheorem{example}[theorem]{Example}

\newtheorem{assumption}{Assumption}
\crefname{assumption}{assumption}{assumptions}

\crefname{paradox}{Paradox}{Paradoxes}

\crefname{problem}{Problem}{Problems}
\DeclareMathOperator*{\argmin}{arg\,min}
\DeclareMathOperator*{\argmax}{arg\,max}
\DeclareMathOperator*{\supp}{supp}
\DeclareMathOperator*{\sgn}{sgn}
\DeclareMathOperator*{\var}{Var}
\DeclareMathOperator*{\Plim}{\PP-\lim}
\DeclareMathOperator*{\BigO}{\mathcal{O}}
\DeclareMathOperator*{\PBigO}{\PP-\mathcal{O}}

\newcommand\scaledinset[6]{%
	\setbox0=\hbox{#6}%
	\stackinset{#1}{#2\wd0}{#3}{#4\ht0}{#5}{#6}
} 

\setlength{\textwidth}{\paperwidth}
\addtolength{\textwidth}{-3in}
\calclayout

\begin{document}

\title[\resizebox{4.7in}{!}{How (Implicit) Regularization of ReLU Neural Networks Characterizes the Learned Function --- Part II}]{How (Implicit) Regularization of ReLU Neural Networks Characterizes the Learned Function \\ Part II: \\ the multi-D Case of Two Layers with Random First Layer}
\author{Jakob Heiss\,$^*$, Josef Teichmann and Hanna Wutte\,$^*$}
\address{ETH Z\"urich, D-Math, R\"amistrasse 101, CH-8092 Z\"urich, Switzerland}
\email{jakob.heiss@math.ethz.ch, jteichma@math.ethz.ch, hanna.wutte@math.ethz.ch}
\thanks{ $*$. Equal contribution. \\
The authors gratefully acknowledge the support from ETH-foundation.}
\curraddr{}
	\begin{abstract}
	
Randomized neural networks (randomized NNs), where only the terminal layer's weights are optimized constitute a powerful model class to reduce computational time in training the neural network model. At the same time, these models generalize surprisingly well in various regression and classification tasks.
In this paper, we give an exact macroscopic characterization (i.e., characterization in function space) of the generalization behavior of randomized, shallow NNs with ReLU activation (RSNs). We show that \RSN s correspond to a Generalized Additive Model (GAM)-typed regression in which infinitely many directions are considered: the \emph{Infinite Generalized Additive Model (IGAM)}. The IGAM is formalized as solution to an optimization problem in function space for a specific regularization functional and a fairly general loss.
This work is an extension to multivariate NNs of the results in \cite{ImplRegPart1V3}, where we showed how  wide RSNs with ReLU activation behave like spline regression under certain conditions and if the input dimension $d=1$.
	\end{abstract}
\keywords{}
\subjclass[]{}

\maketitle
	
	\section{Introduction}\label{se:Introduction}
	 The number of parameters in many applied neural network models (NNs) is steadily increasing. In some fields, randomized NNs (also referred to as random feature models or extreme learning machines) have proven to succeed in significantly cutting training time while achieving performance on par to fully-trained models \cite{HUANG2006489,CAO2018278,rahimi2007random,herrera2021optimal}.
	 The most striking property of these types of networks is that after random initialization only the terminal layer is trained. In particular, the first-layer parameters are chosen randomly at initialization and remain untrained. Thus, training such a randomized NN in essence amounts to (generalized) linear regression, for which a speed-up in computational time seems plausible.
	 
	 Most interestingly however, it has been observed that in many cases such randomized NNs still provide good generalization.
	This might be surprising, since one might expect that the randomness in these models decreases the regularity of the learned function, but in fact the effect is quite the opposite: as we will thoroughly discuss, the learned function will be especially smooth because of this randomness.
	
	In this paper, we specifically treat the generalization properties of randomized NNs with a single hidden layer and ReLU activation (\RSN{}s). We give an exact mathematical characterization \emph{in function space} of how these \RSN{}s generalize under both \emph{implicit} and \emph{explicit} regularization \emph{in parameter space}. Throughout this paper, explicit regularization in parameter space means weight-decay, or $\ell_2$-regularization. Implicit regularization of NNs is generally known as the regularization that arises from the optimization process, rather than from explicit terms added to a loss function. In this paper, we specifically study the implicit regularization effect of a gradient descent method with small weight initialization.
	Both of these explicit and implicit regularization techniques are standard tools commonly applied in day-to-day neural network training. It is therefore particularly exciting to discuss how these specific regularization techniques translate to regularity of the resulting functions in function space. 
	
	This paper is an extension to our previous work \cite{ImplRegPart1V3}.
	In \cite{ImplRegPart1V3}, we discussed how precisely implicit (or explicit) regularization on parameter space characterizes the learned function {{in terms of a regularization functional \notationOf{\Pgpmm}}} in function space
	in the case of \emph{wide}
	\RSN{}s with one dimensional in- and output (i.e., 1-dimensional \wRRSN s). While most of the theorems therein hold true for the $d$-dimensional case with $d\ge 2$ as well, the main contribution of \cite{ImplRegPart1V3}, i.e. \cite[Theorem 3.9]{ImplRegPart1V3} that macroscopically (i.e., in function space) characterizes the learned function was derived for input dimension $d=1$ only. In this paper, we first formulate a regularization functional~\PgSm\ in \Cref{def:adaptedIGAM} as the multidimensional generalization of \notationOf{\Pgpmm}. Second, we formulate and prove \Cref{thm:ridgeToaIGAM} as the multidimensional generalization of \cite[Theorem 3.9]{ImplRegPart1V3}.
	\subsection{Related Work}
	Understanding the surprisingly good generalization properties of neural network models is of much interest throughout the literature.
	For univariate RSNs, \citet{ImplRegPart1V3} and \citet{williams2019gradient} both characterize the regularization in function space induced by regularization in parameter space. In our previous work \cite{ImplRegPart1V3}, we showed that explicit $\ell_2$-regularization corresponds in function space to regularizing the estimate's second derivative for fairly general loss functionals.
	Moreover for least squares regression, we showed that the trained network converges to the smooth spline interpolation of the training data as the number of hidden nodes tends to infinity. Concurrently, \citet{williams2019gradient} established a function space equivalent of the $\ell_2$ regularized objective in parameter space. Compared to \cite{williams2019gradient}, in \cite{ImplRegPart1V3}, we i.a. also discuss early stopping, and treat more general loss functionals. 
	
	In this paper, we discuss regularization induced in function space for \emph{multivariate} RSNs.
	Recently, 	\cite{ImplicitBiasRadonTransform} gave an alternative description in function space for the implicit bias of gradient descent for (univariate and multivariate) wide \RSN{}s. This bias is expressed in terms of Radon transforms of a power of the negative Laplacian. By contrast, in \cite{ImplRegPart1V3} as well as in this paper, we establish how the implicit bias (of both univariate and multivariate \RSN{}s) compares to generalized additive models. Moreover, while \cite{ImplicitBiasRadonTransform} give results for specific initializations and various activation functions, they only discuss the implicit regularization effects for GD on least-squared loss. By contrast, we additionally derive the implicit bias in function space for wide, explicitly $\ell_2$-regularized RSNs for very general loss functionals.

	The implicit regularization of multivariate \RSN{}s is also studied by \citet{pmlr-v119-jacot20a} (called \enquote{Random Feature Models} therein). They do however not study the implicit regularization induced by gradient descent, but the implicit regularization induced by averaging over infinitely many ensemble members of which each ensemble member only has a small number of random neurons and the final layer's parameters are perfectly optimized with respect to $\ell_2$-regularization. In particular no link to gradient descent training is established.

	We refer to \cite[Section 1.4.]{ImplRegPart1V3} for an overview of further related works.

	\subsection{Preliminaries\label{sec:Setting}}
	Throughout this work, as in \cite{ImplRegPart1V3}, we compare the general tasks of minimizing a 
	loss functional $\Ltr$ over a parametric class of functions $\Set{f_\theta : \theta\in\Theta}=:\mathcal{H}_{\Theta}\subset\Set{\mathcal{X}\to\mathcal{Y}}$ via a gradient method or explicit $\ell_2$-regularization on parameter space, i.e.,
	\begin{align}
	\text{gradient descent on } \min_{\theta\in\Theta}\Ltr(f_{\theta})\text{, or}\label{eq:minL}\\
 \min_{\theta\in\Theta}\Ltr(f_{\theta})+\lw\twonorm[\theta]^2, \label{eq:minLL2}
 \end{align}
 to the task of minimizing the same loss functional~$\Ltr$ and a different regularization functional~$P$ over a suitable (non-parametric) function class $\mathcal{H}\subset\Set{\mathcal{X}\to\mathcal{Y}}$, i.e., 
 \begin{align}
 \min_{f\in\mathcal{H}}\Ltr(f)+\lambda P(f)\label{eq:minLP}
    .\end{align}
	Input and output spaces are $\X\subseteq\Rdin$ respectively $\Y\subseteq\Rdout$ with input and output dimension $\din\in\N$ and $\dout\in\mathbb{N}$, the parameter space $\Theta\subset\R^{\tilde{n}}$, for some $\tilde{n}\in\N$, and hyper-parameters $\lambda>0$, $\lw>0$.
	These types of problems for instance include the basic regression and classification settings.

More precisely in this paper, we macroscopically analyze the functions $\hat{f}$ obtained from \cref{eq:minL,eq:minLL2} when the model class $\mathcal{H}_\Theta$ consists of \wRRSN s (i.e., wide ReLU \RSN s). In particular, we characterize the regularity of these functions obtained by 
\begin{enumerate}[a)]
    \item  perfectly optimizing \cref{eq:minLL2} with explicit $\ell_2$-regularization, and by
    \item approximately solving \cref{eq:minL} via training the \wRRSN s' terminal-layer parameters with standard gradient descent algorithms without any explicit regularization.
\end{enumerate}
The main contribution of this paper is to deduce the regularity of these learned functions on function space in terms of \cref{eq:minLP} from both the \emph{explicit and implicit regularizations} on the parameters (i.e., \cref{eq:minL,eq:minLL2}), \emph{in the infinite-width limit}.
In particular, we contribute to the understanding of why training \emph{wide} \RSN s with \emph{and without} explicit regularization leads to surprisingly regular functions~$\hat{f}$, which are \enquote{desirable} from a Bayesian point of view (see \cite{ImplRegPart1V3} for a detailed introduction to the paradox of implicit regularization in the training of neural networks).


	\subsubsection{Our Contribution}
	In this paper, we make the following contributions.
	\begin{itemize}
	    \item We precisely formulate the \aIGAMs, a weighted generalization to the generalized additive model. 
	    \item We show in \Cref{thm:ridgeToaIGAM} that the \aIGAMs\ is a mathematically precise characterization of explicitly $\ell_2$-regularized \RSN s in the infinite-width limit. This result extends the existing literature in particular due to the generality of the loss functional involved. In particular, this loss functional can be non-convex,  it can depend on the first derivative or include integrals over functions (e.g., the loss functionals of \cite{NOMUICML,Cuchiero_2020}).
	    \item We connect this result in \Cref{thm:ridgeToaIGAM} to deduce that the same regularization effects can be \emph{implicitly} achieved for wide \RSN s by gradient-descent-training on least-squared loss without any explicit regularization.
	\end{itemize}

\subsection{Notation}	
	\begin{remark}[General Notation]
	    Within this work, $\skprod{\cdot}{\cdot}$ denotes the standard inner product in $\R^d$. 
	    We write $\Sd$ for the $(d{-}1)$-dimensional unit-sphere (cp. \Cref{def:Sd}) and consider $\Shd\subseteq\Sd$, $|\Shd|<\infty$ to be a finite set of directions in $\R^d$. Moreover throughout this paper, we use the convention $\min\{\emptyset\}:=\inf\{\emptyset\}:=+\infty$. We remark that throughout this work, derivatives are to be understood in the weak sense \cite{Adams:SobolevSpaces1990498}. Moreover, we denote by $\Wt$ the set of twice weakly differentiable functions. 
	    
	\end{remark}
		\begin{remark}[The norm $\sobnormmulti$]\label{rem:sobnormoi}
	    Throughout this paper, we consider
	    \begin{equation}
	        \sobnorm[f]:=\max\{\sup_{x\in K}|f(x)|,\sup_{x\in K}|f^{'}(x)| \},
	    \end{equation}
	    for every $f\in\mathcal{C}(\R)$ with piece-wise continuous derivative $f^{'}$, where we assume w.l.o.g.~that $f^{'}$ is left continuous (i.e.~$\relu[x]^{'}=\ind_{(0,\infty)}(x)$). Furthermore, by $\|\cdot\|_{\Wkp[K]{2}{2}}$ we denote the Sobolev norm using weak (second) derivatives.
	    Convergence in $\sobnormmulti$ is then defined as component-wise convergence in $\sobnorm$.
	\end{remark}
	
	The remainder of this paper is structured as follows.
	In \Cref{sec:IGAM}, we introduce specific generalizations to the additive model, the \GAM\ (\GAMs) and \IGAM\ (\IGAMs). The \IGAMs\ is of particular interest in this work, as we show how it corresponds to an $\ell_2$-regularized wide randomized shallow network (\Cref{thm:ridgeToaIGAM}).
	In \Cref{sec:RSNN}, we recall the definition of said type of neural network: the \emph{d-dimensional \wRRSN}.
	
	This paper's main results are then presented in \Cref{sec:RidgeToaIGAM,sec:GradientToRidge}.
	We conclude with summarizing the implications of these results in \Cref{sec:conclusion}, and \cref{eq:conclusionExplicit,eq:conclusion} in particular. 
	
	
	\section[Infinite Generalization of GAMs (IGAMs)]{Infinite Generalization of Generalzied Additive Models (IGAMs)}\label{sec:IGAM}

	As we will show in a central part of this work, the function obtained by training a \wRRSN\ using a standard gradient descent method on parameter space corresponds to a particular generalization of the additive model\footnote{In the framework of regression, the additive model is a well-known function class used to capture the relation between input $x\in\mathcal{X}$ and output $y\in\mathcal{Y}$ via
	\[ \hat{y}=\hat{f}(x)=\sum_{i=1}^d \varphi_i(x^{i}),\]	with $x\in\Rdin, \din\in\N$ and $\varphi_i(\cdot)$, $i=1,\ldots d$ \enquote{smooth}, non-paramteric functions applied to the i\textsuperscript{th} component of the input $x$.}, the \aIGAM\ (\aIGAMs). Throughout this section, we introduce two extensions of the additive model that the \aIGAMs\ then builds upon.\newline  
	We start by defining what we term a \GAM\ (\GAMs) in \Cref{def:GAM}. A \GAMs, assumes that an estimator $\hat{f}$ for a minimizer of \eqref{eq:minLP} is represented as the sum of univariate smooth functions $\ps[\sh] \in\Wt(\R,\Rdout)$ along certain directions $\sh\in\Sd$ in the input space, i.e. that it is of the form
	\[\hat{f}(\cdot)=\link^{-1}\left(\sum_{\sh\in\Shd}\psr[\langle \sh,\cdot\rangle]{\sh}\right),\]
	for a finite number of directions $\Shd\subseteq\Sd$ and $\varphi\in{\Wt(\R,\Rdout)}^{\Shd}$ (in other words $\varphi=(\varphi^{1},\ldots,\varphi^{\dout}) $ and $\ps[\sh]^{k}\in\Wt(\R)\ \unimportant{\forall \sh \in \Shd, \forall k=1,\ldots,\dout}$). Here, smoothness is characterized by the regularization functional $\PgShm$ in \Cref{def:GAM}, that is we consider $\varphi\in{\Wt(\R,\Rdout)}^{\Shd}$ with minimal squared $L_2-$norm of the second derivative. As in generalized linear models (cp. \citet{10.2307/2344614}), the link function $\link$ acts as mediator between the distribution of output and input variables. 
	\begin{definition}[Link Function]\label{def:linkfunction}
	An invertible function $\link:\Y\to\Rdout$, with Lipschitz continuous inverse $\linkinv$ is called \emph{link function}.
	\end{definition}
	
	Popular choices of $\link$ include the identity in the case of ordinary regression and the logit function $\link(x):=\ln{x/(1-x)}$ with which one can link binary outputs to continuously valued additive predictors.
	
 	\begin{remark}[Related Generalizations]
	    Note that \Cref{def:GAM} is a slightly modified version of the popular \href{https://en.wikipedia.org/wiki/Generalized_additive_model}{\oldnormalcolor generalized additive model} as it was originally introduced by \citet{hastie1986}. In that model, 
	\[\hat{f}(x)=\link^{-1}\left(\sum_{i=1}^d \varphi_i(x^{i})\right),\]
	for $x \in \Rdin$, $\din\in\N$ and $\varphi_i(\cdot)$, $i=1,\ldots d$ \enquote{smooth}, typically non-paramteric functions and a suitable link function $\link$. Here, $x^{i}$ denotes the i\textsuperscript{th} component of $x\in\Rdin$.
	
	Another extension of the additive model similar to our \GAMs\ has been introduced by \citet{friedman1981projection} as the \href{https://de.wikipedia.org/wiki/Projection_Pursuit}{\oldnormalcolor projection pursuit model} \[\hat{f}(x)=\sum_{i=1}^p \varphi_i(\skprod{a_i}{x})\mycomma\quad\faxdg\mydot\] The directional vectors $a_i\in\Rdin$, $i=1,\ldots,p$ are optimized numerically. By contrast, the directions $\Shd$ in \Cref{def:GAM} remain unchanged.
	\end{remark}

		\begin{definition}[\GAMs]\label{def:GAM}
		Let $\lambda \in \Rp$ and \link\ be a link function.
		Then the \textit{\GAM}~$\flGAM$ is defined %
		as
		\begin{equation}\label{eq:GAM}
		\flGAM 
		\in
		\argmin_{f\in \Wt(\Rdin,\Rdout)}\underbrace{\Ltrb{f}+\lambda \PGAMm (f)}_{=:\FlGAMmb{f}},    
		\end{equation}
		with
		\hypertarget{eq:PGAMm}{\begin{equation*}
		\PGAMm(f):=  \min_{\substack{\varphi\in{\Wt(\R,\Rdout)}^{\Shd} \\ \link\circ f=\sum_{\sh\in\Shd}\psr[\langle \sh,\cdot\rangle]{\sh}}} \left(
		\sum_{\sh\in\Shd}\int_{\R} \twonorm[{\psr{\sh}}^{''}]^2 \dx[r]		\right).
		\end{equation*}}

	\end{definition}
\begin{remark}\label{rem:notUniversalGAM}
    In \Cref{def:GAM} we optimize~\meqref{eq:GAM} over all $f\in\Wt(\Rdin,\Rdout)$. Note however, that \flGAM only attains functions from a much smaller subspace, because $\PGAMm$ assigns $\infty$ to most $\Wt$-functions~f\unimportant{, since most functions~$f$ can not be expressed as a sum $\sum_{\sh\in\Shd}\psr[\langle \sh,\cdot\rangle]{\sh}$ of finitely many one-directional functions $\psr[\langle \sh,\cdot\rangle]{\sh}$} and $\min(\emptyset):=\infty$.
\end{remark}

	\begin{definition}[Sphere]\label{def:Sd} The $(d{-}1)$-dimensional (unit-)sphere is defined as:
\begin{equation}
    \Sd:=\Set{x\in\Rdin | \twonorm[x]=1}.
\end{equation}
\end{definition}

If one considers only a fixed finite number of directions $\sh\in\Shd\subseteq\Sd$ the class of \GAMs{}s is not universal (cp. \Cref{rem:notUniversalGAM}). Thus, we proceed to define a universal\footnote{Every $L_2$-function can be approximated arbitrarily well on any compactum by an \IGAMs, an \aIGAMs, a sufficiently large \RSN, and by a \GAMs\ with sufficiently many directions \citep{CybenkoUniversalApprox1989,HornikUniversalApprox1991251}.} infinite generalization: the \IGAMs.
\IGAMs{}s are obtained when taking the number of directions~$\sh$ to infinity:

	\begin{definition}[IGAM]\label{def:IGAM}
		Let $\lambda \in \Rp$ and \link\ be a link function.
		Then the \textit{\IGAM}~$\flIGAM$ is defined %
		as
		\begin{equation}\label{eq:IGAM}
		\flIGAM 
		\in\argmin_{f\in \Wt(\Rdin,\Rdout)}\underbrace{ \Ltrb{f}+\lambda \PIGAMm(f)  }_{=:\FlIGAMmb{f}},    
		\end{equation}
		with\footnote{For any $\varphi\in  {\Wt(\R,\Rdout)}^{\Sd}$ for which the objective of \eqref{eq:PfuncIGAM} is not defined we set it to infinity.}
		\hypertarget{eq:PIGAMm}{\begin{equation}\label{eq:PfuncIGAM}
		\PIGAMm(f):=  \min_{\substack{\varphi\in  {\Wt(\R,\Rdout)}^{\Sd}  \\ \link\circ f=\int_{\Sd}\psr[\langle s,\cdot\rangle]{s}\, ds}} \left(
		\int_{\Sd}\int_{\R} \twonorm[ {\psr{s}}^{''} ]^2 \dx[r]\dx[s]		\right).
		\end{equation}}

	\end{definition}
	
	\section{Randomized Shallow Neural Networks (RSNs)}\label{sec:RSNN}
	Within this paper, we consider randomized shallow, feed-forward neural networks (also known as random feature models). These neural networks consist of one hidden layer and randomly fixed (i.e., non-trainable) first-layer parameters.
		\begin{definition}[Randomized shallow neural network]\label{def:RSNN}
		Let $\left( \Om, \Sigma, \PP\right) $ be a probability space. For a given link function $\link$, let
	     the activation functions $\sigma:\mathbb{R}\to\mathbb{R}$ and $\linkinv:\R^\dout\to\Y$ be Lipschitz continuous and non-constant. Then, a \textit{randomized shallow neural network} is defined as $\RNwo:\,\Rdin\to \Y$ s.t.
		\begin{equation}\label{eq:RSN}
		\RNwo(x):=\linkinv\left(\sum_{k=1}^{n}w_{\cdot,k}\,\sigma\left({b_k\omb}+\sum_{j=1}^{d}{v_{k,j}\omb}x_j\right) \right)\quad\faog\ \faxdg
		\end{equation}
		with\footnote{One could include an additional bias~$c\in\Rdout$ to the last layer too, however in the limit $n\to\infty$ this last-layer bias $c$ does not change the behavior of the trained network-functions \RNw[\wt] or \RNR.}
		\begin{itemize}
			\item number of neurons~$n\in\N$ and input dimension~$\din\in \mathbb{N}$,
			\item trainable weights~$w_k:=w_{\cdot,k}\in\Rdout$, $k=1,\dots, n$,
			\item random biases $b_k\begingroup\transparent{\meineTranzparenz}:(\Om, \Sigma) \to (\R, \B)\endgroup$ \iid\ real valued random variables k=1,\dots,n,
			\item random weights $v_k:=v_{k,\cdot}\begingroup\transparent{\meineTranzparenz}:(\Om, \Sigma) \to (\Rdin, \B^d)\endgroup$ \iid~$\Rdin$-valued random vectors k=1,\dots,n.
		\end{itemize}
		
	\end{definition}

The main assumptions we require to hold are made precise in \Cref{as:mainAssumptions} below.
	\begin{assumption}\label{as:mainAssumptions} Using the notation from \Cref{def:RSNN}:	
		\begin{enumerate}[a)]
			\item The activation function $\sigma(\cdot)=\relu$ is ReLU.
			\item\label{item:densityExists} For every $s\in\Sd$, both the distribution of $s_k:=\frac{v_k}{\twonorm[v_k]}$ and the conditional distribution of $\xi_k:=\frac{-b_k}{\twonorm[v_k]}$ conditioned on $s_k=s$ admit probability density functions $p$ respectively $\gxis$ with respect to the Lebesgue-measure.%
			\footnote{\Cref{as:mainAssumptions}\ref{item:densityExists} holds for any distribution typically used in practice. Moreover, it implies that $\PP[v_k=0]=0 \quad\forall k \in \fromto{n}$. Note, that \Cref{as:mainAssumptions}\ref{item:densityExists} is required in order to exclude certain degenerate cases of \RSN s such as those with constant weights and biases $w_k,b_k, k=1,\ldots,n$, and could in fact be weakened.
			} 
		\end{enumerate}
	\end{assumption}
	
	We henceforth require \Cref{as:mainAssumptions} to be in place. For later uses, we further introduce the notions of kink positions which describe the normal distance between the kink-hyperplanes and zero.   
	
	\begin{definition}[kink positions~$\xi$]\label{def:kinkPos}
		The \textit{kink positions}~$\xi_k:=\frac{-b_k}{\twonorm[v_k]}$ are defined using the notation of \Cref{def:RSNN} under the \Cref{as:mainAssumptions}.
	\end{definition}
	
		\begin{definition}[kink directions~$s$]\label{def:kinkDir}
		The \textit{kink directions}~$s_k:=\frac{v_k}{\twonorm[v_k]}$ are defined using the notation of \Cref{def:RSNN} under the \Cref{as:mainAssumptions}.
	\end{definition}
	
	\begin{definition}[conditioned kink position density~\gxis]\label{def:kinkPosDenscond}
		The \textit{probability density function~$\gxis:\R\to\Rpz$ of the kink position}~$\xi_k:=\frac{-b_k}{\twonorm[v_k]}$ conditioned on $\frac{v_k}{\twonorm[v_k]}=s$ is defined in the setting of \Cref{def:kinkPos}.
	\end{definition}
		\begin{definition}[kink direction density $p$]\label{def:kinkDirectionDens}
		The \textit{probability density function~$p:\Rdin\to\Rpz$ of the kink direction}~$s_k:=\frac{v_k}{\twonorm[v_k]}$ is defined in the setting of \Cref{def:kinkDir}.
	\end{definition}

	
	\section{Main Theorems}\label{sec:theorem}
	
	We now proceed to discuss this work's main contributions to understanding how the implicit and explicit regularization on parameter space translates to the resulting function when the models considered in the minimization of \cref{eq:minL,eq:minLL2} are wide ReLU RSNs. We start by introducing the notion of $\ell_2$-regularized RSNs.

	\begin{definition}[$\ell_2$-regularized \RSN]\label{def:ridgeNet}
		Let $\RNwo$ be a randomized shallow network as introduced in \Cref{def:RSNN}. The $\ell_2$-regularized \RSN is defined as
		\begin{equation}\label{eq:ridgeRSN}
		\RNRo := \RN_{\wRo,\om} \quad\faog\mycomma
		\end{equation}
		with $\wRo$ such that
		\begin{equation}
		\wRo \in \argmin_{w\in\R^{\dout\times }}\underbrace{{\Ltr\left( \RNwo\right) } +\lw||w||_2^2}_{\Fnb{\RNwo}} \quad\faog\mydot
		\end{equation}
		
	\end{definition}
	 
	First and foremost, we show in \Cref{sec:RidgeToaIGAM} how the \wRRSN\ with $\ell_2$-regularized, optimal terminal-layer parameters relates to an \aIGAMs\ (with regularization parameters $\lw>0$ and $\lambda>0$ respectively). We will define the \aIGAMs\ in \Cref{def:adaptedIGAM} as an adapted version of the \IGAMs\ from \Cref{def:IGAM}. More precisely, we show that as the number of hidden nodes $n$, i.e., the dimension of the hidden layer tends to infinity the optimal $\ell_2$-regularized, ReLU-activated network converges to an \aIGAMs\ in probability with respect to a certain Sobolev norm. By \Cref{as:mainAssumptions,as:generalloss}, we prove this correspondence for randomized shallow networks with arbitrary input dimension $\din\in\N$ and a fairly general loss functional.
	
	Thereafter, we restate \cite[Theorem 3.20]{ImplRegPart1V3} noting that for the particular choice of squared loss (see \Cref{as:squaredloss}) the (suitably initialized) gradient flow for optimizing the parameters of an \RSN\ ad infinitum leads to the same solution as performing $\ell_2$-regularized regression with diminishing regularization on the parameters of the \RSN's terminal layer. Note, that this holds true for any \emph{fixed} number of hidden layers $n\in\N$.

	Combining these results we conclude that for the specific choice of squared-loss, training randomized shallow networks using gradient descent implicitly corresponds to solving the regularized problem corresponding to a certain (adapted) \IGAMs .

	\subsection{\texorpdfstring{\hyperref[def:ridgeNet]{$\ell_2$-regularized RSN}}{L2-regularized RSN} \texorpdfstring{$\to$}{to} \texorpdfstring{\IGAMs}{IGAM} 
	}\label{sec:RidgeToaIGAM}
	Throughout this section, we rigorously derive the correspondence between a certain \IGAMs\ and an $\ell_2$-regularized randomized shallow network with hyper parameters $\lambda>0$ and $\lw>0$ respectively. For giving a detailed description of the convergence behavior, we introduce adapted versions of the \GAMs\ and \IGAMs, the \aGAMs~\flSh and \aIGAMs~\flS. These are modified versions of the \GAMs- respectively \IGAMs- with weighting functions introduced in the respective penalisation. Depending on the distribution of the random weights~$w_k$ and biases~$w_b$, the random network~\RNR\ will converge to such an adapted version~\flS\ of the \IGAMs.\newline
	Moreover, the \aIGAMs\ can be obtained in the limit of \aGAMs, as the number of directions $|\Shd|\to\infty$ (see \Cref{le:aGAMtoaIGAM}).

	\begin{definition}[adapted GAM]\label{def:adaptedGAM}
		Let $\lambda \in \Rp$ and \link\ be a link function.
		Then for a given family of functions $\ghs:\R\to\Rpz$, $\sh\in\Shd$, the \textit{\aGAM}~$\flSh$ is defined%
		\footnote{\label{footnote:uniqueAGAM}
		The \aGAMs\ exists for $L$ fulfilling \Cref{as:generalloss}\ref{item:continuousL}, if~$g$ is compactly supported and continuous on $\supp(g)$
		\unimportant{and $\gbarhat\neq 0$}. It is uniquely defined in case we additionally assume that $L$ is convex. (This follows analogously to \cite[Lemma A.25]{ImplRegPart1V3} for finitely many directions $\sh\in\Shd$).
		}
		as
		\begin{equation}\label{eq:adaptedGAM}
		\flSh 
		\in
		\argmin_{f\in \Wt(\Rdin,\Rdout)}\underbrace{\Ltrb{f}+\lambda \PgShm (f)}_{=:\FlShmb{f}},    
		\end{equation}
		with
		\hypertarget{eq:PgShm}{\begin{equation*}
		\PgShm(f):=  \gbarhat\min_{\substack{\varphi\in\Th \\ \link\circ f=\sum_{\sh\in\Shd}\psr[\langle \sh,\cdot\rangle]{\sh}}} \left(
		\sum_{\sh\in\Shd}\int_{\supp (\ghs)} \frac{\twonorm[ {\psr{\sh}}^{''} ]^2}{\ghsr{\sh}} \dx[r]		\right),
		\end{equation*}}
		\hypertarget{eq:gbarhat}{\begin{equation*}
		    \unimportant{\gbarhat\overset{\footnotemark}{:=}\sum_{\sh\in\Shd}\ghsr[0]{\sh}}
		\end{equation*}}\footnotetext{$\gbarhat\in\Rp$ is just a constant scaling factor (cp. \cref{footnote:gbarnecessary}).}
		and
		\begin{align*}
		\Th:=\bigg\{\varphi\in{\Wt(\R,\Rdout)}^{\Shd}\bigg| \forall \sh\in\Shd :& \supp(\ps[\sh]^{''})\subseteq\supp(\ghs),\\
		& \lim_{r\to -\infty} \psr{\sh} =0 \text{ and } \lim_{r\to -\infty} \frac{\partial}{\partial r}\psr{\sh} =0\bigg\}.
		\end{align*}
	\end{definition}

	\begin{definition}[adapted IGAM]\label{def:adaptedIGAM}
		Let $\lambda \in \Rp$ and \link\ be a link function.
		Then for a given family of functions $\gs:\R\to\Rpz$, $s\in\Sd$, the \textit{\aIGAM}~$\flS$ is defined%
		\footnote{\label{footnote:uniqueAIGAM}	The \aIGAMs\ exists for $L$ fulfilling \Cref{as:generalloss}\ref{item:continuousL}, if~$g$ is compactly supported and continuous on $\supp(g)$
		\unimportant{and $\gbar\neq 0$}. 
		 The minimum is attained (uniquely if we additionally assume a convex loss $L$) by \Cref{rem:PhilevelSolutionExistence}.}
		as
		\begin{equation}\label{eq:adaptedIGAM}
		\flS 
		\in\argmin_{f\in\Wt(\Rdin,\Rdout)}\underbrace{ \Ltrb{f}+\lambda \PgSm(f)  }_{=:\FlSmb{f}},    
		\end{equation}
		with\footnote{If for some $f\in W^2$, the set ${\T}_f:=\T\cap\{\varphi\in\T|\link\circ f=\int_{\Sd}\psr[\langle s,\cdot\rangle]{s}\, ds\}$ is non-empty, then the unique minimum over $\varphi\in\T$ is attained: this follows from \cite[Lemma A.24]{ImplRegPart1V3}, setting $\mathcal{X}:={\T}_f$ ($\T$ is a Hilbert space, thus complete, and since ${\T}_f$ is a closed subset of $\T$, it is complete as well), $||\cdot||:=||\cdot||_{ L^{2}(\Sd ,{\Hk[\tilde{K},\Rdout]{2}})}$.}
		\hypertarget{eq:PgSm}{\begin{equation*}
		\PgSm(f):=  \gbar\min_{\substack{\varphi\in\T \\ \link\circ f=\int_{\Sd}\psr[\langle s,\cdot\rangle]{s}\, ds}} \left(
		\int_{\Sd}\int_{\supp (\gs)} \frac{\twonorm[ {\psr{s}}^{''} ]^2}{\gsr{s}} \dx[r]\dx[s]		\right),
		\end{equation*}}
		\hypertarget{eq:gbar}{\begin{equation*}
		    \unimportant{\gbar\overset{\footnotemark}{:=}\int_{s\in\Sd}\gsr[0]{s}}
		\end{equation*}}\footnotetext{$\gbar\in\Rp$ is just a constant scaling factor (cp. \cref{footnote:gbarnecessary}).}
		and
		\begin{align*}
		\T:=\bigg\{ \varphi\in L^{2}(\Sd ,{\Hk[\tilde{K},\Rdout]{2}}) \bigg| \forall s\in\Sd :&\supp(\ps^{''})\subseteq\supp(\gs),\\
		&\lim_{r\to -\infty} \psr{s}=0 \text{ and } \lim_{r\to -\infty} \frac{\partial}{\partial r}\psr{s}=0\bigg\}.
		\end{align*}
		\unimportant{Here, $\tilde{K}:=K'\cup\overline{\bigcup_{s\in\Sd}\supp (\gs)}$ for a compact interval $K'\subset\R$ and
		\[\Hk[\tilde{K},\Rdout]{2}:=\left(\left\{h\in\Wt(\R,\Rdout)\middle|\|h\|_{\Hk[\tilde{K}]{2}}<\infty, \supp (h^{''})\subset\tilde{K}\right\},\|\cdot\|_{\Hk[\tilde{K}]{2}}\right).\footnote{Note that this is an isometric copy of the Sobolev space ${\Hk[\tilde{K},\Rdout]{2}}$ as it is usually defined, i.e., $\left(\left\{h\in\Wt(\tilde{K},\Rdout)\middle|\|h\|_{\Hk[\tilde{K}]{2}}<\infty\right\}, \|\cdot\|_{\Hk[\tilde{K}]{2}}\right)$ and therefore a Hilbert space.}\]}
		
	\end{definition}
	\begin{remark}\label{rem:compactSupp}
		If for the weighting functions $\gs$, $s\in\Sd$, it holds that $\supp(\gs)$ is compact (cp. \Cref{as:truncatedg}\ref{item:truncatedg}), we define for any $s\in\Sd$
		\begin{equation}
		\Cgsl:= \min (\supp(\gs)) \quad \text{and}\quad \Cgsu:=\max(\supp(\gs)).
		\end{equation}
		Furthermore in that case, the set $\T$ can be rewritten: From $\supp (\ps'')\subseteq \supp (\gs)$ it follows that $\ps'\in\Hk[\R]{1}$ is constant on $(-\infty, \Cgsl]$. With $\lim_{r\to -\infty} \psr{s}'=0$ we obtain that $\psr{s}'=0$ $\forall r\le\Cgsl$. By the same argument we obtain $\psr{s}=0$ $\forall s\le\Cgsl$. Hence altogether we have

		\begin{align*}
		\T=\bigg\{ \varphi\in L^{2}(\Sd ,{\Hk[\tilde{K},\Rdout]{2}}) \bigg| \forall s\in\Sd:
		&\supp (\ps'')\subseteq \supp (\gs),\\
		&\, \forall r\le\Cgsl: \psr{s}=0=\psr{s}' \bigg\}.
		\end{align*}
        Analogous derivations can be made to rewrite $\Th$ in the discrete case.
		
	\end{remark}
	
	\begin{remark}[Connection to part I \cite{ImplRegPart1V3}]
	In the case of input and output dimension $\din=1=\dout$ and $\link=\text{id}_{\Rdout}$ the results of this paper are consistent with those of \cite{ImplRegPart1V3} in the following way:
	\begin{align*}
	\Sd&\overset{d=1}{=}\{+1,-1\}
	=\left(\{+,-\}\right)_{\text{\cite{ImplRegPart1V3}}}\\
	\psr[\skprod{+1}{x}]{+1}&\overset{d=1}{=}\notationOf{f_+(x)}\\
	\unimportant{\psr[\skprod{-1}{x}]{-1}}
	&\transparent{\meineTranzparenz}\overset{d=1}{=}\notationOf{f_-(x)}\\
	g_{+1}(x)
	&\overset{d=1}{=}\notationOf{g_+(x)}
	\overset{\text{\cite[Assumption 3]{ImplRegPart1V3}}}{=}\notationOf{g(x)}\overset{\text{\cite[Assumption 3]{ImplRegPart1V3}}}{=}\notationOf{g(-x)}\\
	g_{-1}(-x)
	&\overset{d=1}{=}\notationOf{g_-(x)}
	\unimportant{\overset{\text{\cite[Assumption 3]{ImplRegPart1V3}}}{=}\notationOf{g(x)}
	\overset{\text{\cite[Assumption 3]{ImplRegPart1V3}}}{=}\notationOf{g(-x)}}\\
	\gbar&\overset{\substack{d=1\\\text{\cite[Assumption 3]{ImplRegPart1V3}}}}{=}\notationOf{2g(0)}\\
	    \PgSm&\overset{\substack{d=1\\\text{\cite[Assumption 3]{ImplRegPart1V3}}}}{=}\notationOf{\Pgpmm}
	\end{align*}
	    
	\end{remark}
	
	The following assumptions are technicalities that facilitate the \hyperlink{proof:thm:ridgeToaIGAM}{proof} of our core result \Cref{thm:ridgeToaIGAM} and could be weakened (see \crefrange{footnote:first:as:truncated}{footnote:last:as:truncated}).
	\begin{assumption}\label{as:truncatedg} Using the notation from \Cref{def:RSNN,def:kinkPosDenscond} the following assumptions extend \Cref{as:mainAssumptions}. For every $s\in\Sd$ we assume:
		\begin{enumerate}[a)]
			\item\label{item:truncatedg} The probability density function~$\gxis$ of the kinks~$\xi_k$ conditioned on the kink direction $s$ has compact support~$\supp(\gxis)$.%
			\footnote{\label{footnote:first:as:truncated}We believe that \Cref{as:truncatedg}\ref{item:truncatedg} can be weakened quite extensively.\unimportant{This assumption facilitates our proofs, but we believe it could be replaced by the condition of finite second moment without reformulating the theorem. 
			}} 
			\item\label{item:densityIsSmooth} The density~$\left.\gxis\right|_{\supp (\gxis)}$ is uniformly continuous on $\supp (\gxis)$.%
			\footnote{\label{footnote:densityIsSmooth}One could think of replacing \Cref{as:truncatedg}\ref{item:densityIsSmooth} by the weaker assumption that $\gxis$ is (improper) Riemann-integrable, however almost all distributions which are typically used in practice satisfy \Cref{as:truncatedg}\ref{item:densityIsSmooth}.}
			\item\label{item:reziprokdensityIsSmooth} The reciprocal density~$\left.\frac{1}{\gxis}\right|_{\supp (\gxis)}$ is uniformly continuous on $\supp (\gxis)$.%
			\footnote{\Cref{as:truncatedg}\ref{item:reziprokdensityIsSmooth} implies that $\min_{x\in\supp (\gxis)}\gxis >0$. Similarly to \cref{footnote:densityIsSmooth}, this assumption might be weakened in a way allowing $\gxis$ to have finitely many jumps and  $\min_{x\in\supp (\gxis)}\gxis$ to be zero.}
			\item\label{item:vkDistrSmooth} The conditioned distribution~$\mathcal{L}(v_k|\xi_k=x, s_k=y)$ of $v_k$ is uniformly continuous in $(x,y)$ on $\supp(\gxis)\times\Sd$.%
			\footnote{\label{footnote:last:as:truncated}Similarly to \cref{footnote:densityIsSmooth}, \Cref{as:truncatedg}\ref{item:vkDistrSmooth} might be attenuated.}
			\item\label{item:vk:finiteSecondMoment}$\E[{\twonorm[v_k]^2}]<\infty$.\footnote{\Cref{as:truncatedg}\ref{item:vk:finiteSecondMoment} always holds in typical scenarios. \Cref{as:truncatedg}\ref{item:vk:finiteSecondMoment} together with \Cref{as:truncatedg}\ref{item:truncatedg} and \ref{item:vkDistrSmooth} implies that $\Eco{{\twonorm[v_k]^2}}{\xi_k=x,s_k=y}$ is bounded on $\supp(\gxis)$.}
		\end{enumerate}
	\end{assumption}
The following technical \Cref{as:easyReadable} is made in order to consistently extend the theory derived in \cite{ImplRegPart1V3} to general dimensions. In particular, for the one dimensional case and $\mathcal{L}(v_k)=\mathcal{L}(-v_k)$ we recover the setting of \cite[Theorem 3.9]{ImplRegPart1V3}.
	\begin{assumption}\label{as:easyReadable} Using the notation from \Cref{def:RSNN,def:kinkPosDenscond} the following assumption extends \Cref{as:mainAssumptions}:
		\begin{enumerate}[a)]
			\item\label{item:gbar} $\gbar\neq 0$.%
			\footnote{\label{footnote:gbarnecessary}\Cref{as:easyReadable}\ref{item:gbar} has to be satisfied due to the way \Cref{def:adaptedIGAM} and \Cref{thm:ridgeToaIGAM} are formulated, although the theory could be easily reformulated if \Cref{as:easyReadable}\ref{item:gbar} were not satisfied. The theorems presented would hold as well if $\gbar$ were replaced by $\int_{s\in\Sd}\gsr[{r_{\text{mid}}}(s)]{s}\,ds$ for a value~$r_{\text{mid}}(s)$ depending on the direction $s\in\Sd$, or by e.g. $\int_{s\in\Sd}\frac{1}{2}\int_{-1}^1 \gsr{s} \,dr\,ds$. However, the results are more easily interpreted if $s \, r_{\text{mid}}(s)\in\Rdin$ is located somewhere \enquote{in the middle} of the training data. \Cref{thm:ridgeToaIGAM} would even hold true if $\gbar:=1$. It is only necessary to consistently pick the same definition $\gbar\neq 0$ everywhere it appears---e.g. in the definition of \lw\ in \Cref{thm:ridgeToaIGAM} and in the definition of regularization-functional \PgSm.}

		\end{enumerate}
	\end{assumption}
	
	\begin{assumption}[Loss]\label{as:generalloss}
There exist $p\in[1,\infty)$ and a finite Borel measure $\nu=\nuc+\nua$, where $\nuc$ is absolutely continuous w.r.t.\ the Lebesgue measure and $\nua$ is supported on a finite (possibly empty) subset $\{x_1,\ldots,x_N\}\subset\X$ such that the loss functional $\Ltr:L^\infty(K)\to\Rpz$ is
\begin{enumerate}[a)]
\item\label{item:continuousL} continuous w.r.t.\ $\sobnormop[\cdot]$ for some compact $K\subset\R^\din$ and
\item\label{item:lipschitzL} 
Lipschitz\footnote{We think uniformly continuous should be sufficient, but would make the proof more cumbersome.} continuous w.r.t.\ $\sobnormop[\cdot]$ on $\{f: \Ltr(f)<\Ltr(0)+\epsilon\}$ for some $\epsilon>0$.
\end{enumerate}
	\end{assumption}
	
	
	\begin{example}[Classification]
	    Let $\allIndi{N}{i}: \xtr_i\in\mathcal{X}, \ytr_i \in [0,1]$. If we set  \[\Ltr(f):=\sum_{i=1}^N{\ltri\left(f(\xtr_i)\right)},\] 
	    with cross-entropy losses $\ltri$ and softmax as final activation function $\link^{-1}$, \Cref{as:generalloss} is fulfilled.
	\end{example}
	\begin{example}[Regression]\label{ex:Regression}
	    Let $\allIndi{N}{i}: \xtr_i\in\mathcal{X}, \ytr_i \in \mathcal{Y}$. If we set  \begin{equation}\label{eq:squaredLoss}\Ltr(f):=\sum_{i=1}^N{||\ytr_i-f(\xtr_i)||_2^2},\end{equation} and use the identity as final activation function $\link^{-1}$, \Cref{as:generalloss} is fulfilled. \unimportant{Of course, almost any other classical regression loss such as $\Ltr(f):=\sum_{i=1}^N{||\ytr_i-f(\xtr_i)||_1}$ fulfills \Cref{as:generalloss} as well.}
	\end{example}
	 Note that in contrast to many other results in the literature, \Cref{as:generalloss} does not require the loss to be convex. In fact, the loss functional can be quite general, it can depend on the first derivative or include integrals over functions, as we see in the following example.
	
\begin{example}
	    Let $\allIndi{N}{i}: \xtr_i\in\mathcal{X}, \ytr_i \in \mathcal{Y}$. For any loss functional of the form  \[\Ltr(f):=\sum_{i=1}^N{\ltri\left(f(\xtr_i)\right)}+\int_{K}u(f(x),f'(x))\,dx,\] 
	    with losses $\ltri:\mathcal{Y}\to\Rpz$ that are Lipschitz-continuous on the sub-level set \begin{align*}\Set{y\in\Y:\ltr_i(y)\leq L(0)+\epsilon},\end{align*} a Lipschitz-continuous function $u:\mathcal{Y}\to\Rpz$ and a compact set $K\subset\R$, then \Cref{as:generalloss} is fulfilled.
     For instance, the loss in \cite{NOMUICML} can be written in this form if one chooses a Lipschitz integrand (as done in \cite{weissteiner2023BOCA}).
	\end{example}
 \begin{proof}
 Let  $\nu_c(E):=\lambda_{\R^{\din}}(E\cap K)$ be the Lebesgue measure restricted to $K$ and let $\nu_a(E):=\#(E\cap \Set{\xtr_i : i \in \fromto{N}})$ be the counting measure restricted to the training data points.
 \end{proof}

Based on these assumptions, we now formulate this paper's main results.	
	\begin{theorem}[$\ell_2$-regularized \RSN\ corresponds to adapted IGAM]\label{thm:ridgeToaIGAM}
		Using the notation from \Cref{def:RSNN,def:kinkPosDenscond,def:kinkDirectionDens,def:ridgeNet,def:adaptedIGAM} let%
		\footnote{Since all $v_k$ are identically distributed and all $\xi_k$ are identically distributed as well, the conditioned expectation \Eco{{\twonorm[v_k]^2}}{\xi_k=r, s_k=s} does not depend on the choice of $k\in\fromto{n}$. Therefore, we will sometimes use the following notation~$\Eco{v}{\xi=r, s=s}:=\Eco{v_k}{\xi_k=r, s_k=s}$}
		\[\forall s\in\Sd, r\in\R:\quad \gsr{s}:=p(s)\gxis(r) \Eco{{\twonorm[v_k]^2}}{\xi_k=r, s_k=s},\] and $\lw:=\lambda n \gbar$. Let further the \Cref{as:mainAssumptions,as:truncatedg,as:easyReadable}, as well as \Cref{as:generalloss} for some compact $K\subset\Rdin$ be satisfied. Then the following statement holds: for every $\left(\RNR\right)_{n\in\N}$ 
		\begin{equation}\label{eq:ridgeToaIGAM}
		\plim d_{\Woi[K,\Rdout]}\left(\RNR,
	\argmin\FlSm\right) =0.
\footnote{Using the definition of the $\Plim$, equation~\meqref{eq:ridgeToaIGAM} reads as: $\forall\left(\RNR\right)_{n\in\N}\in\prod_{n\in\N}\argmin\Fn:\forall\epsilon>0:\forall\rho\in(0,1):\exists n_0\in\N:\forall n>n_0$:
	\begin{equation*}
	    \PP\left[\exists\flS\in\argmin\FlSm:\sobnormmulti[{\RNR-\flS }]<\epsilon\right]>\rho.
	\end{equation*}
}
		\end{equation}
	\end{theorem}
	\proofInSec{thm:ridgeToaIGAM}{subsubsec:proof:ridgeToaIGAM}

	The following lemma shows that \aGAMs s converge to \aIGAMs s as you let the number of directions go to infinity (more precisely as you let the gaps between the directions go to zero).
	\begin{lemma}[adapted GAM converges to adapted IGAM]\label{le:aGAMtoaIGAM}
	    Let $\lambda \in \Rp$. Furthermore, let $\flS$ be an \aIGAMs~ with \[\gs(\cdot):=g(s,\cdot)~p(s),\, s\in\Sd,\] for some weighting function $g:\Sd\times\R\to\Rpz$ and $p$ the kink direction density (see \Cref{def:kinkDirectionDens}). Consider the \aGAMs~$\flSh$ for $n=|\Shd|$ finite directions, with weighting function $\ghs$ given as \[\ghs(\cdot):=g(\sh, \cdot)\int_{U(\sh)}p(s)\dx[s],~ \sh\in\Shd,\] where $\mathcal{U}:=\{U(\sh)\}_{\sh\in\Shd}$ are disjoint environments such that $\dot\bigcup_{\sh\in\Shd}U(\sh)=\Sd$. We further denote the maximal distance of the partition $\mathcal{U}$ w.r.t. the Euclidean norm $\twonorm$ by
	    \begin{displaymath}
	    |\mathcal{U}|:=\max_{\sh\in\Shd}\left(\sup_{s\in U(\sh) }\twonorm[s-\sh]\right).
	    \end{displaymath}
	    Let further the \Cref{as:mainAssumptions,as:truncatedg,as:easyReadable}, as well as \Cref{as:generalloss} for some compact $K\subset\Rdin$ be satisfied. Then, if $|\mathcal{U}|\overset{n\to\infty}{\longrightarrow}0$,
	    we have 
	    that
	    \begin{equation}\label{eq:aGAMtoaIGAM}
	    \lim_{n\to\infty}d_{\Woi[K,\Rdout]}\left(\flSh,
	\argmin\FlSm\right) =0, \quad\forall\flSh\in\argmin\FlShm.
	    \end{equation}
	\end{lemma}
\proofInSec{le:aGAMtoaIGAM}{sec:proof:RidgeToaIGAM}

{
		\subsection{RSN and Gradient Descent \texorpdfstring{$\to$}{to} \texorpdfstring{$\ell_2$}{L2}-regularized Network 
		}\label{sec:GradientToRidge}

We now move on to restate the convergence result
\cite[Theorem 3.20]{ImplRegPart1V3} in \Cref{thm:GDRidge}, which, for the \emph{specific choice of squared loss} (from \Cref{ex:Regression}), characterizes the convergence for infinite training time of the solution obtained by gradient descent to the $\ell_2$-regularized network within this paper's setting, i.e., for general output dimension $\dout\in\N$.\footnote{This correspondence is well known in the literature. However, we refer to \Cref{thm:GDRidge} for a formulation of the result in the context of randomized shallow neural networks mapping from $\Rdin$ to $\Rdout$.} Throughout this section, we consider the setting of supervised learning with squared loss, i.e., we require \Cref{as:squaredloss} to hold true. We begin by briefly recalling the notion of 'time-$T$-solution', i.e., the trained \RSN s~\unimportant{\RNw[\wt]} obtained by pursuing the gradient flow w.r.t.\ this choice of training loss starting in the origin~\unimportant{\wt[0]=0} in parameter space up to time $T$.
	
	\begin{assumption}\label{as:squaredloss}
	The loss functional is given as $\Ltr(f):=\sum_{i=1}^N{||\ytr_i-f(\xtr_i)||_2^2}$ for some training data $\xtr_i\in\mathcal{X}=\Rdin, \ytr_i \in \mathcal{Y}=\Rdout, i=1,\ldots,N$.
	\end{assumption}
	
	\begin{definition}[time-$T$ solution]\label{def:GDsolution}
		Let $\allIndi{N}{i}: (\xtr_i, \ytr_i)\in \R^{\din+\dout}$ for some $N,\din\in\mathbb{N}$ and $\RNw$ be a randomized shallow neural network with $n\in\mathbb{N}$ hidden nodes. For any $\omega \in \Omega$ and $T>0$, the time-$T$ solution to the problem
		\begin{equation}\label{eq:GDproblem}
		\min_{w\in\R^{\dout\times n}} \underbrace{\sum_{i=1}^N\twonorm[{\RNwo(\xtr_i)-\ytr_i}]^2}_{\Ltrb{\RNwo}}
		\end{equation}
		is defined as $\RNwo[\wto]$, with weights $\wto\in\R^{\dout\times n}$ obtained by taking the gradient flow
		\begin{align*}
		d\wt[t]&=-\nabla_w \Ltrb{\RNw[{\wt[t]}]}\, dt,\tag{GF}\label{eq:GDflow}\\
		\wt[0]&=0,
		\end{align*}
		corresponding to \meqref{eq:GDproblem} up to time $T$.
	\end{definition}
	
	\begin{remark}\label{rem:discreteGD}
		In practice, the weights $\wt$ of the time-$T$ solution as introduced in \Cref{def:GDsolution} are approximated by taking 
		$\tau:=T/\gamma$ steps of size $\gamma>0$ according to the Euler discretization 
		\begin{align*}\label{eq:GDdiscrete}
		\wth[t+\gamma]&=\wth[t]-\gamma\nabla_w L(\RNw[{\wth[t]}]),\tag{GD}\\
		\wth[0]&=0,
		\end{align*}
		corresponding to \eqref{eq:GDflow}. 
	\end{remark}

	\begin{theorem}\label{thm:GDRidge}
		Let $\RNw[{\wt}]$ be the time-$T$ solution and consider for $\lw=\frac{1}{T}$ the corresponding $\ell_2$-regularized network $\RNR[\frac{1}{T}]$ (cp. \Cref{def:GDsolution,def:ridgeNet}). We then have that
		\begin{equation}\label{eq:thm:GDRidge}
		\fao:\quad \tlim\sobnorm[{\RNRo[\frac{1}{T}]-\RNwo[\wt\omb]}]=0.
		\end{equation}
	\end{theorem}
	{\transparent{\meineTranzparenz}
    \begin{proof}
        The proof follows analogously to the one of \cite[Theorem 3.20]{ImplRegPart1V3}.
    \end{proof}
    }
    
    \begin{remark}[Early Stopping] Note that all the results about early stopping from \cite[Subsection 3.2.1 Early Stopping]{ImplRegPart1V3} hold true in case of arbitrary input-dimension~$\din\in\N$. I.e., if we use the equation $\lw=\frac{1}{2(e-1)T}$, then the solution $\RNwo[\wt\omb]$ obtained from an early stopped gradient flow~\eqref{eq:GDflow} without any explicit regularization is usually quite close to the perfectly optimized $\ell_2$-regularized solution $\RNRo$.
    \end{remark}
}

\section{Conclusion}\label{sec:conclusion}
We are now ready to summarize our main results. 	The notation $\overset{\to}{\approx}$ in equations \eqref{eq:conclusionExplicit} and \eqref{eq:conclusion} below corresponds to a mathematically proven exact limit in the very strong\footnote{Convergence in \sobnorm[\cdot] implies uniform convergence on $K$ or convergence in $\Wkp[K]{1}{p}$. Even stronger Sobolev-convergence, such as convergence w.r.t. $\WkpShort{2}{p}$, cannot be shown since $\RN_w\notin \Wkp[K]{2}{p}$.} Sobolev-Norm~\sobnorm[\cdot] {\color{hellgrau}(in probability in the case of $\scriptstyle \overset{\underset{n\to\infty}{\PP}}{\approx}$)}.
\paragraph{\textbf{Explicit regularization}} By \Cref{thm:ridgeToaIGAM}, a wide ReLU $\ell_2$-regularized \RSN\ (cp. \Cref{def:ridgeNet}) corresponds to an \IGAMs\ in the following sense:
	\begin{equation}\label{eq:conclusionExplicit}
	\RNR[\lw]
	\underset{\text{\Cref{thm:ridgeToaIGAM}}}{\overset{\underset{n\to\infty}{\PP}}{\approx}}\flS
	.
	\end{equation}
 For convex loss functionals, both $\RNR[\lw]$ and $\flS$ are unique (see \Cref{rem:PhilevelSolutionExistence}). In this case, \cref{eq:conclusionExplicit} means that the $\ell_2$-regularized \RSN~$\RNR[\lw]$ converges in probability to the \aIGAMs~$\flS$ w.r.t.\ the Sobolev norm $\sobnormmulti$.
 However, depending on the choice of loss functional $\Ltr$, neither $\RNR[\lw]$ nor $\flS$ must be unique. In that general case, the approximation in \cref{eq:conclusionExplicit} should be more generally interpreted as in \Cref{thm:ridgeToaIGAM}.
\paragraph{\textbf{Implicit regularization}}
 Combining \Cref{thm:ridgeToaIGAM,thm:GDRidge} we can characterize the effect of implicit regularization when minimizing the squared loss \eqref{eq:squaredLoss} over the function class of randomized shallow networks with $n$ hidden nodes. Then, for 
 large training time~$T$ and a large number of neurons~$n$, the obtained network
	\begin{equation}\label{eq:conclusion}
	\RN_{\wt[T]}
	\underset{\text{\Cref{thm:GDRidge}}}{\overset{\substack{T\to\infty\\
	\unimportant{\ltri{(y)}=\twonorm[y-\ytr_i]^2}}}{\approx}}\RNR[\frac{1}{T}]
	\underset{\text{\Cref{thm:ridgeToaIGAM}}}{\overset{\underset{n\to\infty}{\PP}}{\approx}}\fzS
	\end{equation}
	is very close to the \aIGAMs\ $\flS$ with diminishing penalization.%
	\footnote{\label{footnote:extensionToLefteqConcl}One can extend \cref{eq:conclusion} to the left-hand side to standard discretized gradient descent~\eqref{eq:GDdiscrete} via
	\begin{equation*}
	    {\transparent{\meineTranzparenz}\RN_{\wth[{T,\wth[0]}]}
	\overset{\wth[0]\to 0}{\approx}}\RN_{\wth[T]}
	\overset{\gamma\to0}{\approx}\RN_{\wt[T]}
	\end{equation*} analogously to \cite[eq. (32) and (33)]{ImplRegPart1V3}
	.}

	In applications however, both the number of hidden nodes and training steps are finite. Hence, it is particularly interesting to note that in typical settings for \emph{arbitrary} training time $T\in\Rp$ (including early stopping, i.e. $T\ll\infty$) the same relation approximately holds true. We refer to \cite[Sections 3.2.1 and 4]{ImplRegPart1V3} for a discussion of early stopping and finite numbers of nodes.

	\subsection{Empirical illustration} We investigate a randomized shallow neural network trained to approximate the function $f:\R^2\to\R, x\mapsto \twonorm[x]^2$, given $N=64$ noisy data points $(x_i, f(x_i)+\epsilon_i)\in\R^3$, where $x_i$, $i=1,\ldots,N$ are realisations of independent centred Gaussians with covariance matrix $\Sigma=\text{I}_2$, scaled to fit the two-dimensional $[-1,1]$-cube. The i.i.d. noise terms $\epsilon_i$, $i=1,\ldots,N$ are Gaussian as well, with standard deviation $\sigma=0.05$. The RSN was chosen to consist of $n=2^{12}$ hidden nodes with first-layer weights and biases sampled from a Uniform distribution on $[-0.05, 0.05]$.\newline
	\begin{figure}[!h]
		\centering

            \includegraphics[width=.8\linewidth]{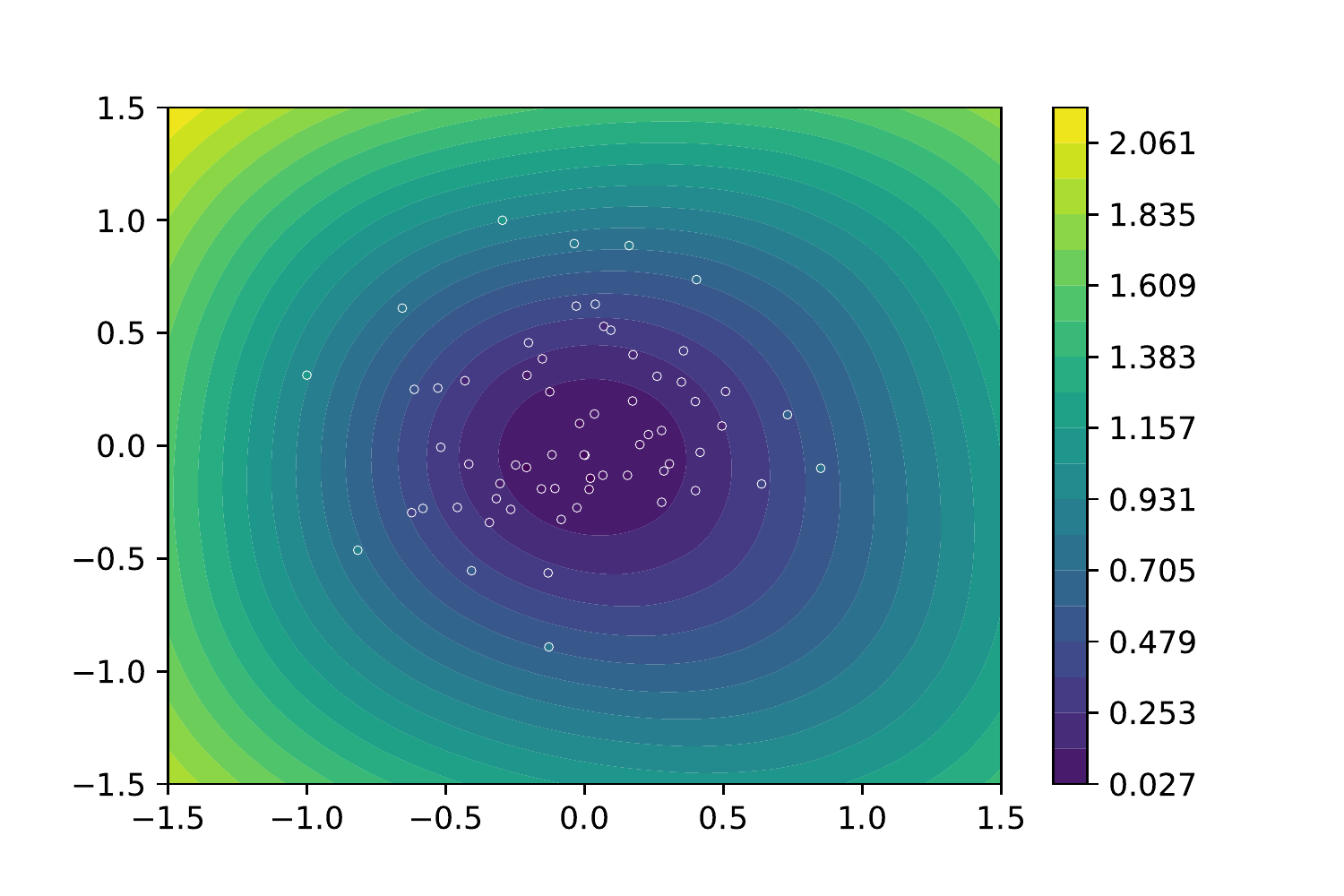}

            \centering
            \includegraphics[width=.8\linewidth]{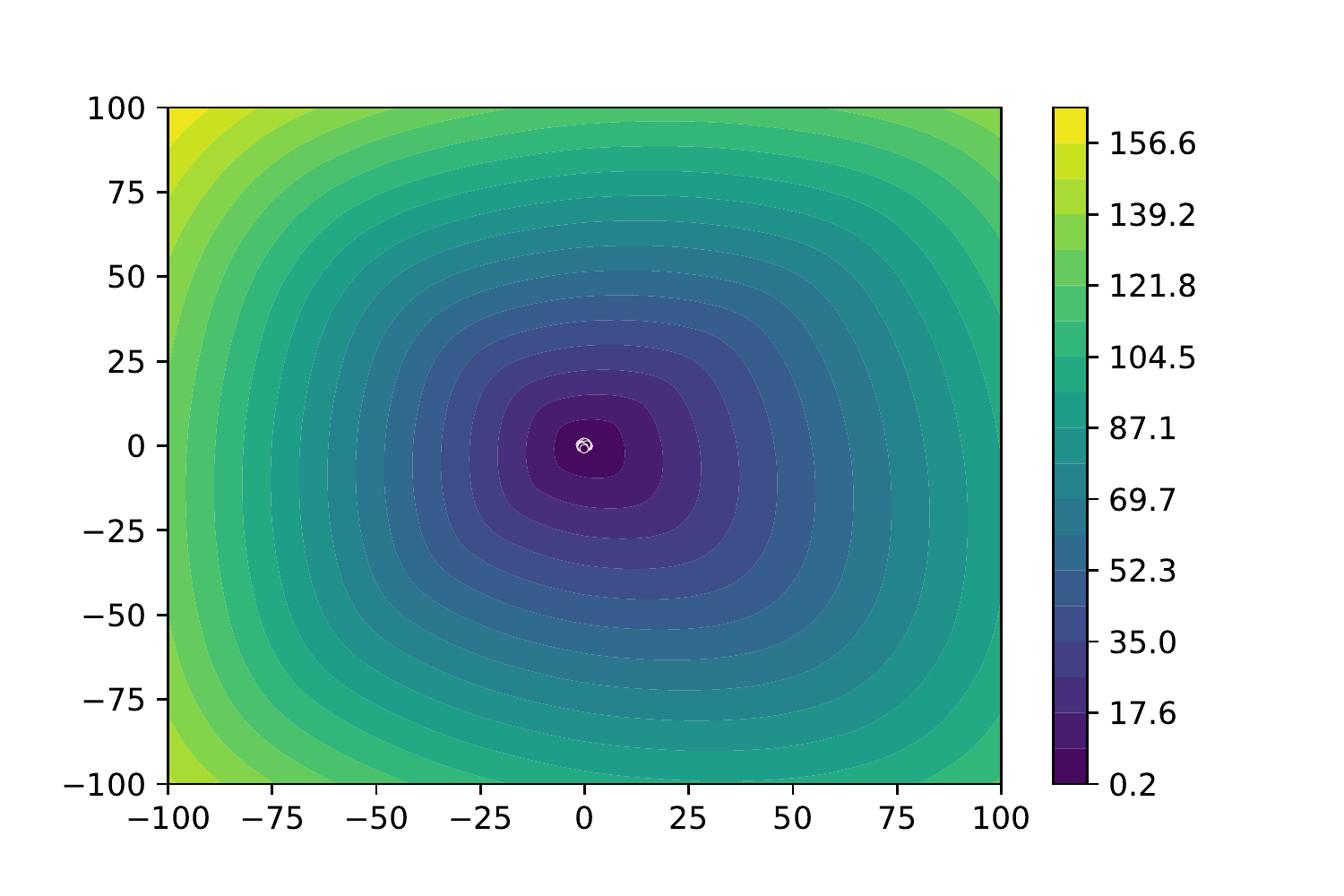}

        \caption{The solution function obtained from performing gradient descent to train an RSN to fit through the data points (white circles) visualized on different scales.}
        \label{fig:RSNSGD}
	\end{figure}
    \Cref{fig:RSNSGD} shows the contours of the solution function obtained by using a standard implementation of gradient descent with step size $\gamma=2^{-15}$ 
		for $\tau=2^{15}$ epochs. Within this paper's setting, this corresponds to the time-T solution for $T=1$.\newline
	In \Cref{fig:RSNGRAD_zoomed_in,fig:RSNGRAD_zoomed_out} we visualize the the components of gradient $\nabla \RNw[\wt]$ on differing scales. \Cref{fig:RSNGRAD_zoomed_in} shows the gradient's contours close to the data. As these plots show, $\nabla \RNw[\wt](x)\approx\nabla f(x)=2x$ for $x\in[-1,1]^2$, i.e. close to the data, the trained RSN approximates well the unknown function $f$.\newline
	On regions far from the $[-1,1]$-cube, the trained network displays a cone-like shape, with $\nabla \RNw[\wt]\approx x/\twonorm[x]$ (cp. \Cref{fig:RSNGRAD_zoomed_out}). 
\begin{figure}[!h]\hspace{-2.5cm}
\begin{minipage}{0.5\textwidth}
            \includegraphics[width=1.3\linewidth]{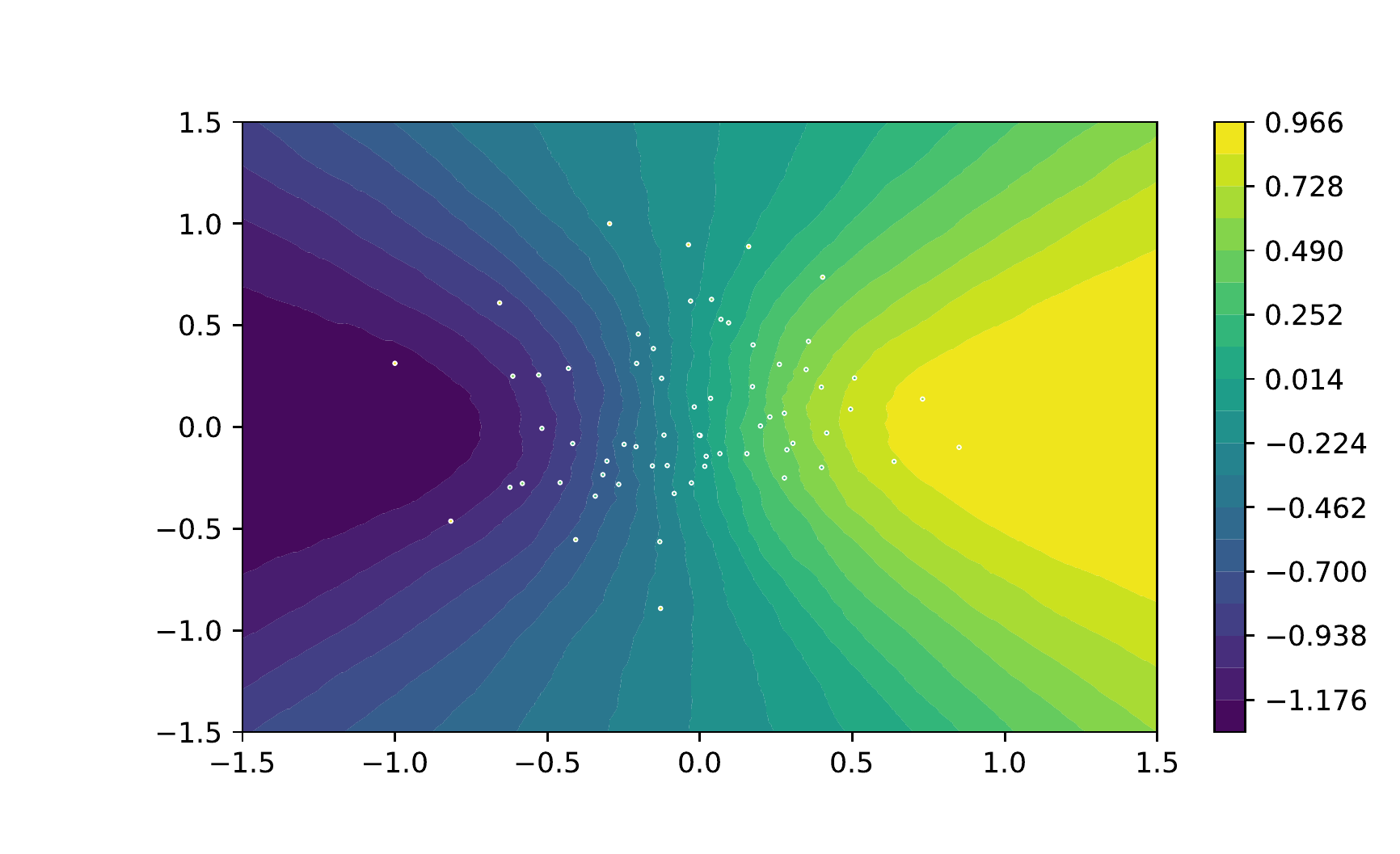}    
\end{minipage}\hspace{1cm}
\begin{minipage}{0.5\textwidth}
            \includegraphics[width=1.3\linewidth]{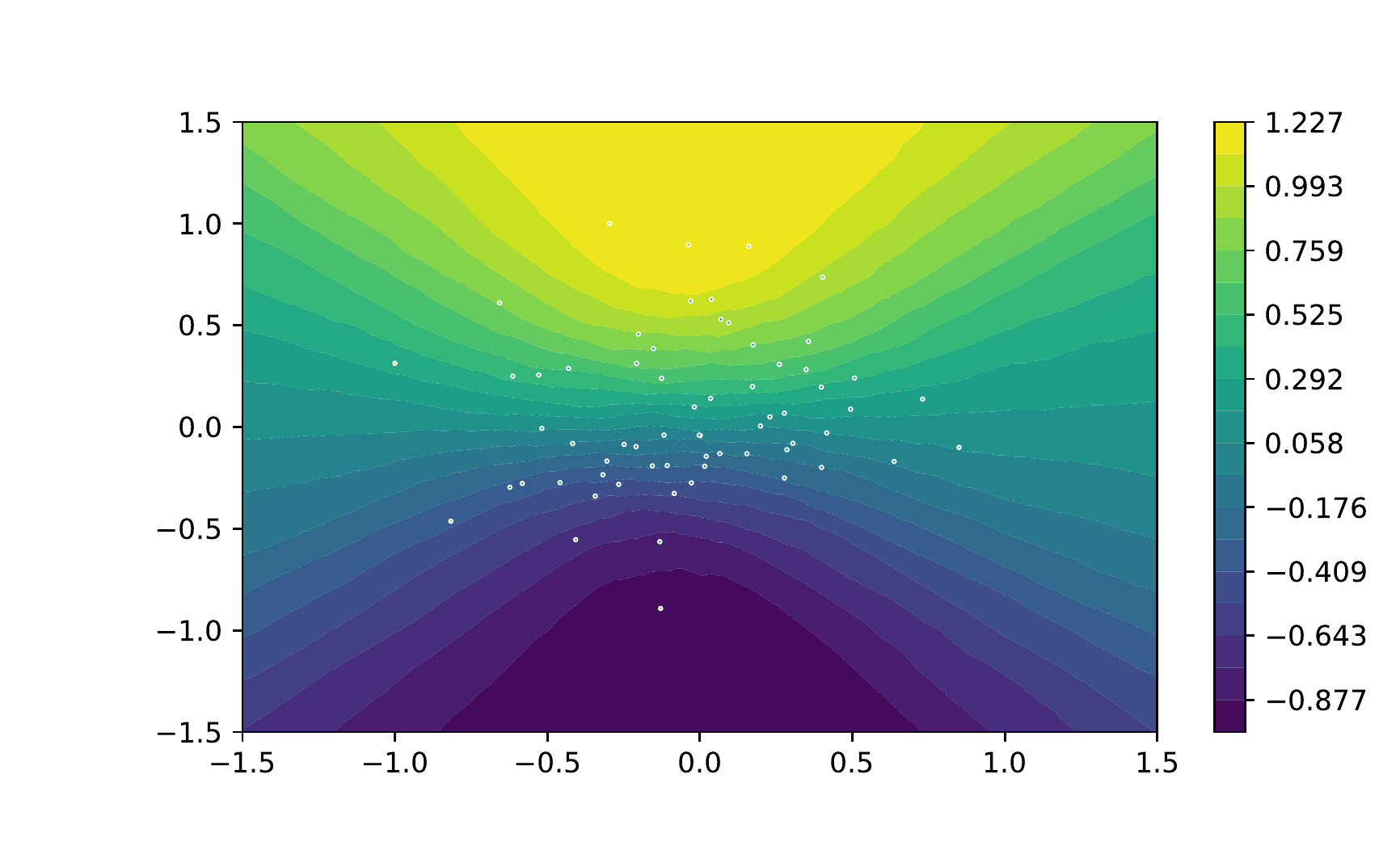}
\end{minipage}

        \caption{The contours of $\frac{\partial}{\partial x_1} \RNw[\wt]$ and $\frac{\partial}{\partial x_2} \RNw[\wt]$ for values $x$ close to the training data (white dots). In particular for $x\in[-1,1]^2$, the contours suggest that $\frac{\partial}{\partial x_i} \RNw[\wt](x)\approx\frac{\partial}{\partial x_i} f(x)=2x_i$, $i=1,2$.}
        \label{fig:RSNGRAD_zoomed_in}
	\end{figure}
	
		\begin{figure}[!h]\hspace{-1.5cm}
\begin{minipage}{0.5\textwidth}
            \includegraphics[width=1.1\linewidth]{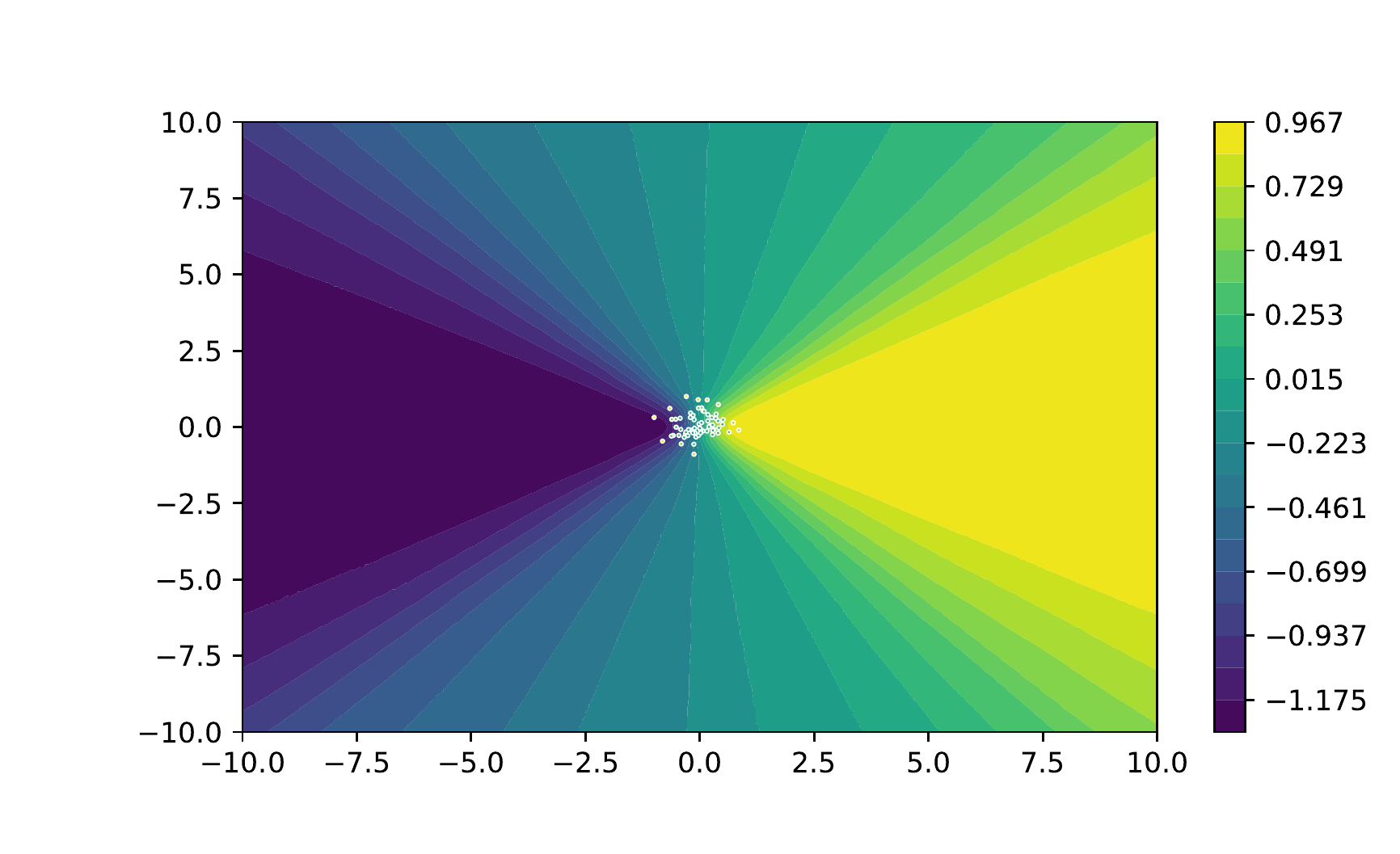}    
\end{minipage}
\begin{minipage}{0.5\textwidth}
            \includegraphics[width=1.1\linewidth]{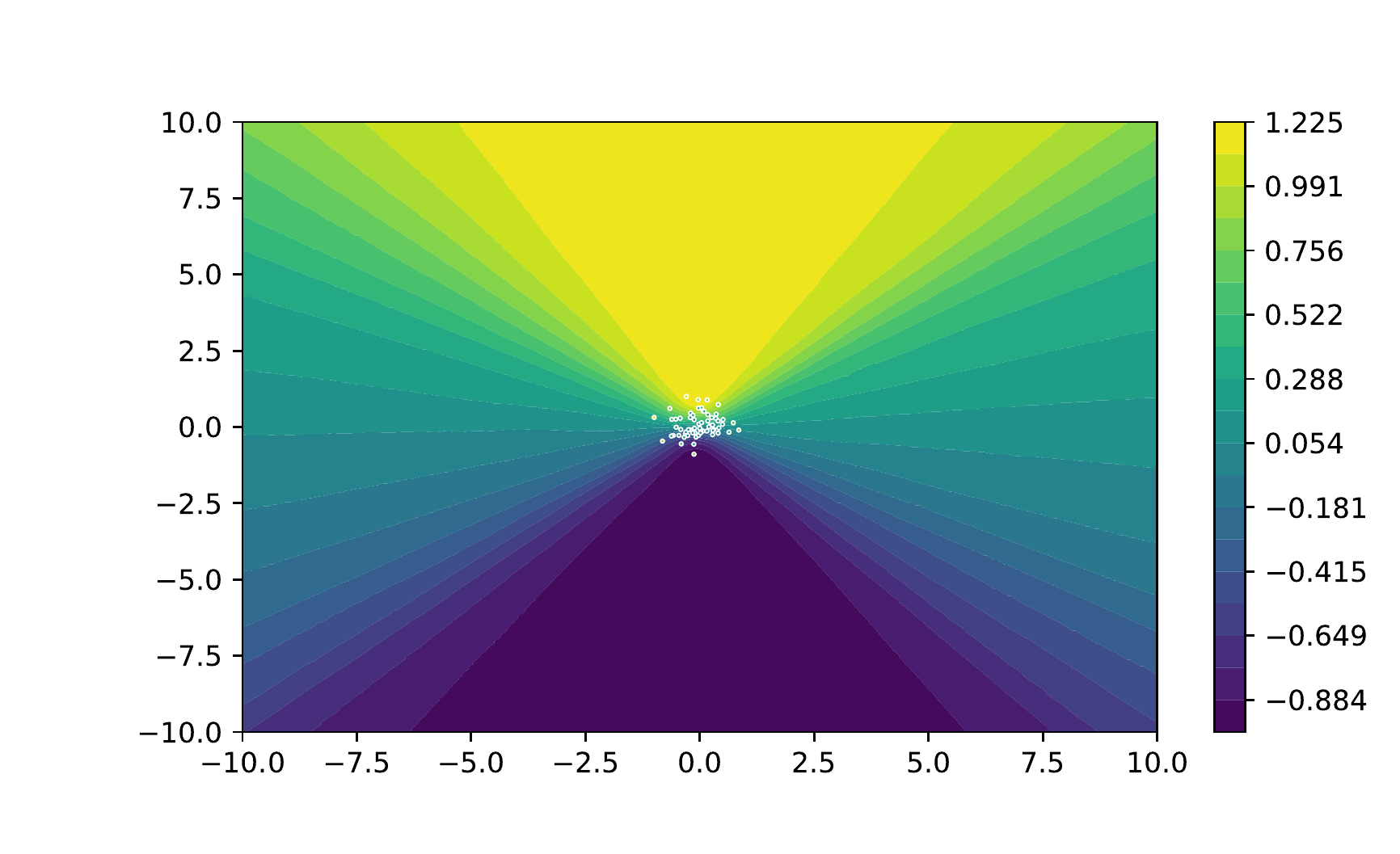}
\end{minipage}

	
\hspace{-1.5cm}
\begin{minipage}{0.5\textwidth}
            \includegraphics[width=1.1\linewidth]{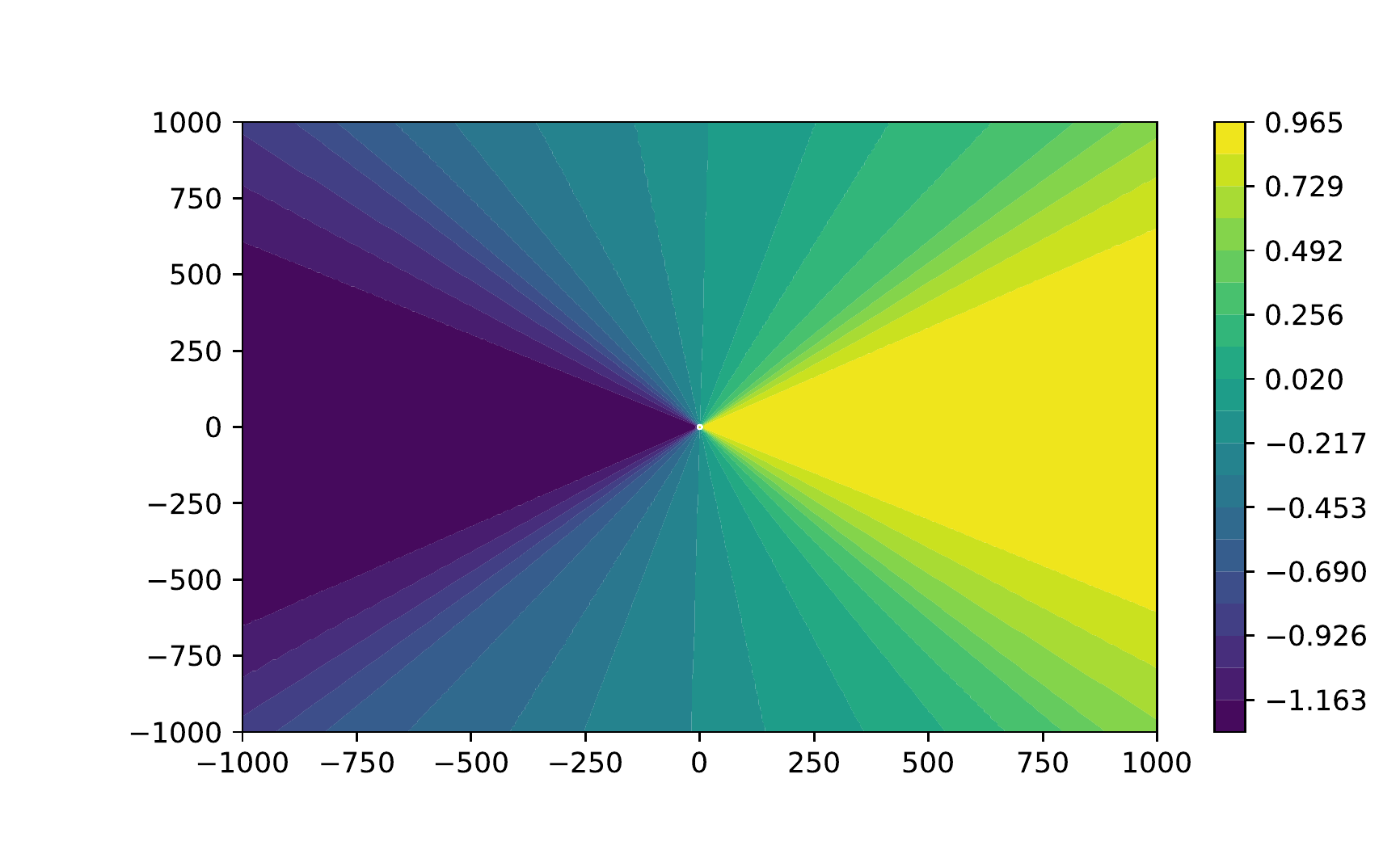}    
\end{minipage}
\begin{minipage}{0.5\textwidth}
            \includegraphics[width=1.1\linewidth]{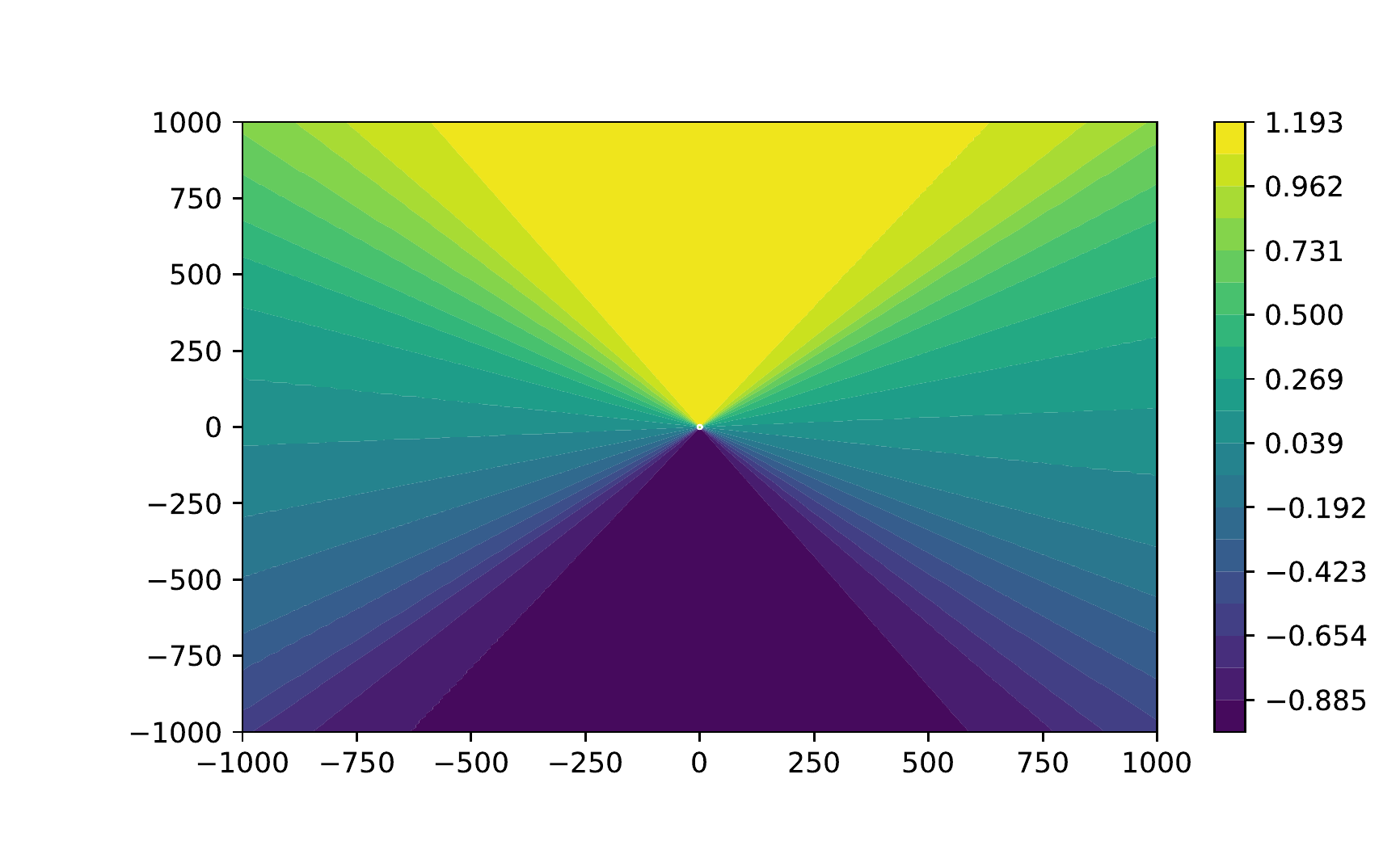}
\end{minipage}

        \caption{The contours of $\frac{\partial}{\partial x_1} \RNw[\wt]$ and $\frac{\partial}{\partial x_2} \RNw[\wt]$ on two different scales. We observe, that far from the training data the trained RSN is akin to a cone.}
        \label{fig:RSNGRAD_zoomed_out}
	\end{figure}

	\subsection{Ongoing Work and Further Extensions}
Within this paper, we presented a first generalization to the work \cite{ImplRegPart1V3} on (implicit) regularization of \RSN s to multidimensional input and output $\X=\Rdin, \Y=\Rdout$. In \cite{implRegPart3V4}, we elaborate on explicit regularization effects in terms of regularization functionals for fully trained, shallow as well as deep neural networks. The long-term objective of our research is to extend this theory to various network architectures. Ideally, this would lead to a thorough understanding of the regularizing constraints (in terms of regularizing functionals) that one implicitly imposes by deciding for the structure, algorithm and meta-parameters of the neural network considered in supervised learning. In other words, our vision is to on a large scale link network architectures to prior beliefs.

\newpage
\bibliographystyle{plainnat}
\bibliography{MeineLiteratur4_general}

	\appendix
	\section{Proofs}\label{sec:proofs}
	In the following, we rigorously prove the results presented within this paper.
	
	In \Cref{subsec:existenceAndEquivalentProblem}, we discuss regularity properties of solutions to the optimization problems defining the \aGAMs\ and \aIGAMs\ and prove their existence.
	Thereafter in \Cref{sec:proof:RidgeToaIGAM}, we prove this paper's main results \Cref{le:aGAMtoaIGAM} (\Cref{sec:proofle:aGAMtoaIGAM}) and \Cref{thm:ridgeToaIGAM} (\Cref{subsubsec:proof:ridgeToaIGAM}) respectively.  
		\begin{center}
		\boxed{\text{
				Throughout this section, we henceforth require \Cref{as:mainAssumptions,as:truncatedg,as:easyReadable,as:generalloss} to be in place.}}
	\end{center}
	
	\subsection{Existence and Regularity of aGAM and aIGAM}\label{subsec:existenceAndEquivalentProblem}
	We start by showing specific regularity properties of the \aGAMs\ \flSh\ and \aIGAMs\ \flS\ in \Cref{le:0.1NEW}.
	Thereafter, we give an equivalent formulation of the optimization problem defining the \aIGAMs\ \flS\ in \Cref{def:Philevelproblem}.
	We use this equivalent formulation to derive existence (and uniqueness, in case of convex loss $\Ltr$) of minimizers $\varphi^*$ (respectively $\flS$) of \Cref{def:Philevelproblem} (respectively \Cref{def:adaptedIGAM}) in \Cref{rem:PhilevelSolutionExistence}.

		\begin{lemma}\label{le:0.1NEW}
	   Let $\flS$ be an \aIGAMs\ as introduced in \Cref{def:adaptedIGAM}, with link function \link, $\tilde{K}\subset{\R}$ compact, and $\varphi^*$ s.t. \[\link\circ\flS(x)=\int_{\Sd}\psrstar[\skprod{s}{x}]{s}\dx[s], \quad \faxd.\] Then
	    \begin{enumerate}
	        \item\label{item:phistarSecondderivCont} $\frac{\partial^{2}\psrstar[r]{s}}{\partial r^{2}}:\Sd\times\R\to \Rdout$ is continuous.
	        \item\label{item:phistarInL2W2infty}  $\psstar[]\in\LtwoS[{\Wtimultidim[{\tilde{K}}]}]$. 
	        \item\label{item:phistarContOnRdin} Moreover,
	        \begin{equation*}
	        \Sd\to \Wtimultidim, s\mapsto\psrstar[\skprod{s}{\cdot}]{s}
	    \end{equation*} is continuous for any compact $K\subset\Rdin$ s.t. $\bigcup_{s\in\Sd}\Set{\skprod{s}{x}\mid x\in K}\subset\tilde{K}$.
	        \item\label{item:phistarinW12}  $\psstar[]\in \Wkp[{\Sd, \Woi[{\tilde{K}}]}]{1}{2}$ and there exists a constant $C>0$ such that the first (weak) derivative $\partial\varphi^*/\partial_s$ is bounded by $C$ in $\LtwoS[{\Woi[\Kw]}]$.
	    \end{enumerate} 
	    
	    Let $\flSh$ be an \aGAMs\ as introduced in \Cref{def:adaptedGAM}, with link function \link, and $\check{\varphi}^*$ s.t. \[\link\circ\flSh(x)=\sum_{\Shd}\psrcstar[\skprod{\sh}{x}]{\sh}, \quad \faxd.\] Then
	    \begin{enumerate}[resume]
	        \item\label{item:phistarhatSecondderivCont} $\frac{\partial^{2}\psrcstar[r]{\sh}}{\partial r^2}:\R\to \Rdout$ is continuous for every $\sh\in\Shd$.
	        \item\label{item:phistarhatContOnRdin} the map 
	    \begin{equation}\label{eq:discretescontinuous}
	        \Sd\to \Wtimultidim, s\mapsto\psrcstar[\skprod{s}{\cdot}]{\sh}
	    \end{equation} is continuous for any compact $K\subset\Rdin$, $\sh\in\Shd$ s.t. $\bigcup_{\sh\in\Shd}\Set{\skprod{\sh}{x}\mid x\in K}\subset\tilde{K}$.
	    \end{enumerate}
	\end{lemma}

	\begin{proof}
	\begin{enumerate}
	    \item 

	We start by showing $\frac{\partial^{2}\psrstar[r]{s}}{\partial r^{2}}\in\C(\Sd\times\R,\Rdout)$ ( i.e., that there exists a continuous representative in the equivalence class of $\frac{\partial^{2}\psrstar[r]{s}}{\partial r^{2}}\in L^2(\Sd\times\R,\Rdout)$). Let $(s_0, r_0)\in \Sd\times\R$ and $\epsilon>0$ be fixed. Furthermore, define
	\[\overline{\varphi^{{*}^{''}}}:=\frac{1}{\mu({U_{\delta}(s_0,r_0)})}\int_{U_{\delta}(s_0,r_0)}\frac{\partial^{2}\psrstar[r]{s}}{\partial r^{2}}\,d\mu(s,r),\]
	for some open environment $U_\delta(s_0,r_0)$, $\delta>0$ around $(s_0, r_0)$. By optimality of $\psrstar[r]{s}$, we have that for some small $\delta>0$
	\begin{equation}\label{eq:mu=0}
	    \mu\left(\underbrace{\left\{(s,r)\in U_{\delta}(s_0,r_0)\bigg|\twonorm[{\psrstar[r]{s}^{''}-\overline{\varphi^{*^{''}}}}]\ge\epsilon/2\right\}}_{=:N}\right)=0.
	\end{equation}
	To see this, assume $\mu(N)>0$. Define for $\Delta:=\epsilon/2>0$
	\begin{align}\label{eq:varphiDelta}
	    \varphi^{\Delta^{''}}(s,r):=\left\{\begin{matrix}
	    \psrstar[r]{s}^{''}+\Delta& (s,r) \in N^-,\\
	    \psrstar[r]{s}^{''}-\Delta&(s,r) \in N^+, \\ \psrstar[r]{s}^{''} & \text{else},\end{matrix}\right.
	\end{align}
	with $N=N^+\cup N^-$ and 
	\begin{align*}
	    N^+:=\left\{(s,r)\in U_{\delta}(s_0,r_0)\bigg|\psrstar[r]{s}^{''}\ge\overline{\varphi^{*^{''}}}+\Delta \right\},\\
	    N^-:=\left\{(s,r)\in U_{\delta}(s_0,r_0)\bigg|\psrstar[r]{s}^{''}\le\overline{\varphi^{*^{''}}}-\Delta \right\}.
	\end{align*}
	Moreover, 
	\begin{equation}
	    f^\Delta(\cdot) :=\int_{\Sd}\varphi^\Delta(s, \skprod{s}{\cdot})\dx[s].
	\end{equation}
	We remark that 
	\begin{enumerate}
	    \item by \Cref{as:generalloss} there exists $K\subset\Rdin$ compact s.t. \begin{align*}\sobnormmulti[f^\Delta-\flS]< 4\Delta \delta^2|\Sd|\end{align*} implies that $|\Ltr(f^\Delta)-\Ltr(\flS)|<\epsilon_L$. Therefore, we proceed with showing \begin{align*}\sobnormmulti[f^\Delta-\flS]< 4\Delta \delta^2|\Sd|.\end{align*}\\
	    Indeed, for every $s\in\Sd$ we have $\nlim[r\to-\infty]\varphi^{\Delta^{'}}(s,r)=0$ and $\nlim[r\to-\infty]\varphi^{\Delta}(s,r)=0$ since $\varphi^*\in\T$, and hence
	    \begin{align*}
	        \varphi^{\Delta^{'}}(s,r)&=\int_{-\infty}^{r}\varphi^{\Delta^{''}}(s,x)\dx[x]
	        = \psrstar[r]{s}^{'}\mp\Delta| N_s^{\pm }\cap(-\infty, r]|,\\
	        \implies \varphi^\Delta (s,r) &= \int_{-\infty}^{r}\varphi^{\Delta^{'}}(s,y)\dx[y]=\psrstar[r]{s}\mp \Delta\int_{-\infty}^r| N_s^{\pm }\cap(-\infty, y]|\dx[y]\\
	        &=\psrstar[r]{s}\mp \Delta\int_{r_{0}-\delta}^{r\wedge (r_{0}+\delta)}| N_s^{\pm }\cap(-\infty, y]|\dx[y],
	        	    \end{align*}
This implies
\begin{align*}
  \|\varphi^{\Delta^{'}} (s,r)-\psrstar[r]{s}^{'}\|& \le 2\Delta\delta\quad\forall(s,r)\in\Sd\times\R,\\
	        \|\varphi^\Delta (s,r)-\psrstar[r]{s}\|& \le 4\Delta\delta^2\quad\forall(s,r)\in\Sd\times\R,
	        \end{align*}
	        and further we get that first,
  \begin{align*}
	        f^{\Delta}(x)= \int_{\Sd}\psrstar[\skprod{s}{x}]{s}\pm 4\Delta\delta^2\dx[s]= \flS (x)\pm 4\Delta\delta^2|\Sd|\quad \forall x\in\Rdin.
	    \end{align*}
	    with $N_s^{\pm}:=\{r\in\R|(s,r)\in N^{\pm}\}$. 
	    Therefore, $\supnorm[{f^\Delta-\flS}]< 4\Delta \delta^2|\Sd|$.
	   Second,
	   \begin{align*}
	        \|f^{\Delta^{'}}(x)-{\flS}^{'}(x)\|&\le \int_{\Sd}\|\varphi^{\Delta}(s,\skprod{s}{x})^{'}-\psrstar[\skprod{s}{x}]{s}^{'}\|\|s\|\dx[s]\\
	        &=\int_{\Sd}\|\varphi^{\Delta}(s,\skprod{s}{x})^{'}-\psrstar[\skprod{s}{x}]{s}^{'}\|\dx[s]\\
	        &\le 2\Delta \delta|\Sd|\quad \forall x\in\Rdin.
	    \end{align*}
 Therefore, $\left\|{{f^\Delta}^{'}-{\flS}^{'}}\right\|_{L^{\infty}(K)}< 2\Delta \delta|\Sd|$.
	    
	    Since by \Cref{as:generalloss} it holds that $\Ltr$ is Lipschitz continuous with constant $C_L$ w.r.t $\Wkp[K]{1}{p}$ on a level set of $L$, and $f^\Delta$ lies in that level set for $\epsilon$ small enough, we obtain by the above derivations
	    \begin{align*}
	        |\Ltr(f^\Delta)-\Ltr(\flS)|&\le\|f^\Delta-\flS\|_{\Wkp[K]{1}{p}}C_L\\
	        &\le\sobnormmulti[f^\Delta-\flS]C_L\\
	        &\le 4C_L\Delta\delta^2|\Sd|.  
	    \end{align*}
	    Therefore, for $\delta:=\left(\frac{\epsilon_{L}}{4\Delta C_L|\Sd|}\right)^{1/2}$ we have $|\Ltr(f^\Delta)-\Ltr(\flS)|<\epsilon_L$.
	    \item due to the quadratic form of the regularization, there exists $C_\Delta>0$ s.t. $\PgSm(f^\Delta)\le\PgSm(\flS)-\Delta^2C_\Delta$.\\
	    Indeed,
	    \begin{align*}
	        \int_{\Sd}\int_{\R}\frac{\twonorm[{ \varphi^{\Delta^{''}}(s,r)}]^2}{\gsr{s}}\dx[r]\dx[s]=&\int_{\Sd}\int_{\R}\frac{\twonorm[{ \varphi^{\Delta^{''}}(s,r)\mp\Delta\ind_{\N^{\pm}}(s,r)}]^2}{\gsr{s}}\dx[r]\dx[s]\\
	        =&\int_{\Sd}\int_{\R}\frac{\twonorm[{ \psrstar[r]{s}^{''}}]^2}{\gsr{s}}\dx[r]\dx[s]\\
	        &+2\left[ \int_{\Sd}\int_{N_s^-}\frac{\Delta\overbrace{\psrstar[r]{s}^{''}}^{<\bar{\varphi}^{{*}^{''}}-\Delta}}{\gsr{s}}\dx[r]\dx[s]-\int_{\Sd}\int_{N_s^+}\frac{\Delta\overbrace{\psrstar[r]{s}^{''}}^{>\bar{\varphi}^{{*}^{''}}+\Delta}}{\gsr{s}}\dx[r]\dx[s]\right]\\
	        &+\Delta^2\underbrace{\int_{\Sd}\int_{N_s^-\cup N_s^+}\frac{1}{\gsr{s}}\dx[r]\dx[s]}_{=:G}\\
	        <&\int_{\Sd}\int_{\R}\frac{\twonorm[{ \psrstar[r]{s}^{''}}]^2}{\gsr{s}}\dx[r]\dx[s]\\
	        &+2\left[\left(\Delta\bar{\varphi}^{{*}^{''}}-\Delta^2\right)\underbrace{\int_{\Sd}\int_{N_s^-}\frac{1}{\gsr{s}}\dx[r]\dx[s]}_{=:G^-}\right.\\
	        &\qquad\left.-\left(\Delta\bar{\varphi}^{{*}^{''}}-\Delta^2\right)\underbrace{\int_{\Sd}\int_{N_s^+}\frac{1}{\gsr{s}}\dx[r]\dx[s]}_{=:G^+}\right]\\
	        &+\Delta^2 G\\
	        =&-2\Delta^2(G^-+G^+)+\Delta^2 G=-\Delta^2 G,
	    \end{align*}
	    assuming $G^-=G^+$ in the penultimate equation.
	    \end{enumerate}
	    Hence, choosing $\delta<\left(\frac{\Delta C_\Delta}{4 C_L|\Sd|}\right)^{1/2}$ we obtain $\epsilon_L<\Delta^2C_L$ and thus $\FlSmb {f^\Delta}<\FlSmb {\flS}$ contradicting its optimality.
	Thus, by \eqref{eq:mu=0}, for every $(s,r)$ in $U_\delta(s_0, r_0)$, we have by the triangle inequality that $\mu$-a.e.
	\[\twonorm[{\psrstar[r]{s}^{''}-\psrstar[r_0]{s_0}^{''}}]<\epsilon.\]
	Hence the first claim follows.
	
\item	Since $\frac{\partial^{2}\psrstar[r]{s}}{\partial r^{2}}\in\C(\Sd\times\R,\Rdout)$, we have $\supnormon[\frac{\partial^{2}\psrstar[r]{s}}{\partial r^{2}}]{\Sd\times\tilde{K}, \Rdout}<\infty$ for every $\tilde{K}\subset\R$ compact. Thus $\frac{\partial^{2}\psrstar[r]{s}}{\partial r^{2}}\in L^\infty(\Sd\times\tilde{K},\Rdout)\subset\LtwoS[{\Linfty[{\tilde{K},\Rdout}]}]$.
	Since $\varphi^*\in\T$, \Cref{le:poincare} (item \eqref{item:PoincareLiLiProductSpace}) implies that also 
	$\frac{\partial\psrstar[r]{s}}{\partial r}\in L^\infty(\Sd\times\tilde{K},\Rdout)\subset\LtwoS[{\Linfty[{\tilde{K},\Rdout}]}]$.
	Analogously, we get that $\psrstar[r]{s}\in L^\infty(\Sd\times\tilde{K},\Rdout)\subset\LtwoS[{\Linfty[{\tilde{K},\Rdout}]}]$.
	Thus we have $\varphi^*\in L^2\left(\Sd,\Wtimultidim[{\tilde{K},\Rdout}]\right)$.
	In particular, 
	\begin{displaymath}
	    \varphi^*:\Sd\to\Wtimultidim[{\tilde{K},\Rdout}],
	\end{displaymath}
	is continuous.
 \item Since for any compact $K\subset\Rdin$, the map $\Sd\to C^\infty(K, \R):s\mapsto\skprod{s}{\cdot}$ is uniformly continuous, with \eqref{item:phistarInL2W2infty} we get continuity of $\Sd\to \Wtimultidim, s\mapsto\psrstar[\skprod{s}{\cdot}]{s}$. 
	\item To show that $\psstar[]\in \Wkp[{\Sd, \Woi[{\tilde{K}}]}]{1}{2}$, we first note that \[\psstar[]\in \T\subset\LtwoS[{\Wkp[{\tilde{K}}]{2}{2}}]\subset\LtwoS[{\Wkp[{\tilde{K}}]{1}{\infty}}]\] and that the Hilbert space ${\Wkp[{\tilde{K}}]{1}{\infty}}$ has the Radon Nikodym property. Now, by \cite[Thm. 2.5.]{kreuter2019vector} if $\varphi^*$ has $\LtwoS[{\Woi[{\tilde{K}}]}]$-bounded difference quotients, we obtain that $\varphi^*\in\Wkp[{S,\Woi[\tilde{K}]}]{1}{2}$.
	Since by item (2), $\Sd\to\Wkp[\tilde{K}]{1}{\infty}$ is continuous, the difference quotient
	\begin{align*}
	    \frac{\|\varphi^*_{s+he_{i}}-\varphi^*_{s}\|_{\Wkp[\tilde{K}]{1}{\infty}}}{|h|} 
	\end{align*}
	is bounded by a constant $C$ on $\Sd$. Thus we get that
	\begin{align*}
	    \LtwonormonSd[{\varphi^*_{s+he_{i}}-\varphi^*_{s}}]{\Wkp[\tilde{K}]{1}{\infty}}\le|h||\Sd|^{\frac{1}{2}}C.
	\end{align*}
	Thus, we can apply \cite[Thm. 2.5.]{kreuter2019vector} as mentioned.
	Moreover by \cite[Thm. 2.5.]{kreuter2019vector}, the first (weak) derivative $\partial\varphi^*_s/\partial_s$ is bounded by $C$ in $\LtwoS[{\Woi[\Kw]}]$.
		\end{enumerate}
	
	Analogously, it follows that the maps $\frac{\partial^{2}\psrcstar[r]{\sh}}{\partial r^2}:\R\to \Rdout$, $\sh\in\Shd$, as well as
	 for any compact $K\subset\Rdin$ s.t. $\bigcup_{\sh\in\Shd}\Set{\skprod{\sh}{x}\mid x\in K}\subset\tilde{K}$ and $\sh\in\Shd$ the maps \eqref{eq:discretescontinuous} are continuous.
	\end{proof}
	
	We now state an equivalent formulation of the optimization problem of \Cref{def:adaptedIGAM} in \Cref{def:Philevelproblem}.
	
			\begin{definition}\label{def:Philevelproblem}
		Let 
		$\lambda \in \Rp$.
		Then for a given function $g:\R\to\Rpz$ we define
		\begin{equation}\label{eq:PhilevelSolution}
		\phi^*\in \argmin_{\phi\in \T}\underbrace{\Ltrb{f}+\lambda \PgSpmb{\phi} }_{=:F^\phi(\phi)},    
		\end{equation}
		with $\link\circ f=\int_{s\in\Sd}\phi_s(\skprod{s}{\cdot})\dx[s]$, and
		\hypertarget{eq:PgSpmb}{\begin{equation}\label{eq:PgSpGeneralLoss}	\PgSpmb{\phi}:=\gbar\int_{\Sd}\int_{\supp (\gs)} \frac{\twonorm[{{\phi_{s}(r)}^{''} }]^2}{\gsr{s}} \dx[r]\dx[s].
		\end{equation}}

	\end{definition}

\begin{remark}[Equivalence of Optimization problems]\label{rem:PhilevelSolutionEquiv}

	    Note that every solution $\phi^*$ to \meqref{eq:PhilevelSolution}  yields a solution~$\flS:=\int_{s\in\Sd}\phi^*_s(\skprod{s}{\cdot})\dx[s]$ to \meqref{eq:adaptedIGAM}.
	    Conversely, for any solution~$\flS$ to \meqref{eq:adaptedIGAM} there exists a solution~$\phi^*$ to \meqref{eq:PhilevelSolution} such that $\flS:=\int_{s\in\Sd}\phi^*_s(\skprod{s}{\cdot})\dx[s]$.
	    Thus, $\phi^*$ is a solution to \meqref{eq:PhilevelSolution} if and only if  $\flS:=\int_{s\in\Sd}\phi^*_s(\skprod{s}{\cdot})\dx[s]$ is a solution to \meqref{eq:adaptedIGAM}. Moreover, $\min F^\phi=\min\FlSm$.
	    \end{remark}
 We elaborate more on the existence of the minimizers in \cref{eq:PhilevelSolution} (and by \Cref{rem:PhilevelSolutionEquiv} also the existence of \flS) in \Cref{rem:PhilevelSolutionExistence}.
	
		\begin{remark}[Existence and Uniqueness of the \aIGAMs]\label{rem:PhilevelSolutionExistence}
	 	By \Cref{le:0.1NEW}, we know that for any optimal $\varphi^*$ s.t. \[\link\circ\flS(x)=\int_{\Sd}\psrstar[\skprod{s}{x}]{s}\dx[s], \quad \faxd\]
	with \flS\ as in \Cref{def:adaptedIGAM}, we have that $\varphi^*\in \LtwoS[{\Hk[\tilde{K}]{2}}]\cap\Wkp[{\Sd,\Woi[\tilde{K}]}]{1}{2}$. Therefore, it suffices to consider the restricted space \begin{equation}\label{eq:Tw}
	    \tilde{\T}:=\T\cap\Wkp[{\Sd,\Woi[\tilde{K}]}]{1}{2}
	\end{equation} in the definition of \PgSm. By \Cref{rem:PhilevelSolutionEquiv}, it thus also suffices to consider $\tilde{\T}$ in \cref{eq:PhilevelSolution}, i.e., we have
	\[\argmin_{\phi\in\T}F^\phi(\phi)=\argmin_{\phi\in\tilde{\T}}F^\phi(\phi).\]
	We apply \cite[Lemma A.23]{ImplRegPart1V3} to show that 
	\[\argmin_{\phi\in\tilde{\T}}F^\phi(\phi)\neq\emptyset\]
	as follows. We set $\mathcal{X}:=\tilde{\T}$ ($\tilde{\T}$ is a closed subset of the Hilbert space $\LtwoS[{H^2(\tilde{K},\Rdout)}]$, thus complete), $||\cdot||:=||\cdot||_{ L^{2}(\Sd ,{\Hk[\tilde{K},\Rdout]{2}})}$ and $\tau$ the topology induced by $||\cdot||_{ L^{2}(\Sd ,{\Woimultidim[\tilde{K}]})}$. By \Cref{le:PfuncStronglyConvex}, $\PgSpm$ is strongly convex w.r.t. $||\cdot||$, and by \Cref{le:LossContinuous,le:f_continuous_in_phi} $L$ is continuous w.r.t. $\tau.$
	Moreover, let $\lambda>0$ and define $\mathcal{K}:=\{\phi\in\tilde{\T}:\PgSpmb{\phi}\le \frac{L(0)}{\lambda}+\PgSpm(0)\}$. By \Cref{le:compactnessoflevelset}, $\mathcal{K}$ is sequentially compact w.r.t. $\tau$.
	
	In case $L$ is convex, an analogous application of \cite[Lemma A.24]{ImplRegPart1V3} yields that $\argmin_{\phi\in\tilde{\T}}F^\phi(\phi)$ is a singleton.
	\end{remark}

	\subsection{Proof of \Cref{thm:ridgeToaIGAM}}\label{sec:proof:RidgeToaIGAM}
	
	A number of lemmata are required for both the \hyperlink{proof:le:aGAMtoaIGAM}{proof} of \Cref{le:aGAMtoaIGAM} and the \hyperlink{proof:thm:ridgeToaIGAM}{proof} of \Cref{thm:ridgeToaIGAM}. These will be given below.
	We start by defining two specific aIGAM and aGAM that will be needed for the following proofs.

	\begin{definition}[aIGAM2aGAM]\label{def:aIGAMtoaGAM}
		Let for some $\lambda>0$, $\flS$ be an \aIGAM\ as introduced in \Cref{le:aGAMtoaIGAM}, with link function \link,
		\[\gs(\cdot):=g(s,\cdot)~p(s),\, s\in\Sd,\] for some weighting function $g:\Sd\times\R\to\Rpz$ and $\mathcal{U}:=\{\Ush\}_{\sh\in\Shd}$ are disjoint environments such that $\dot\bigcup_{\sh\in\Shd}\Ush=\Sd$. Furthermore, let $\varphi^*$ s.t. \[\link\circ\flS(x)=\int_{\Sd}\psrstar[\skprod{s}{x}]{s}\dx[s], \quad \faxd.\] The \emph{aIGAM2aGAM}~\fcw\ w.r.t. the aIGAM $\flS$ is given by
		\begin{equation}
		\link\circ\fcw(x)=\sum_{\sh\in\Shd}\psrcw[\skprod{\sh}{x}]{\sh},\quad \faxd,
		\end{equation}
		with \[\psrcw[r]{\sh}:=\psrstar[r]{\sh}\, \mUsh[\sh], \quad \forall(\sh,r)\in\Shd\times\R.\]
    Here, $\mu$ denotes the $(d-1)$-dimensional surface measure.
	\end{definition}
	
	\begin{definition}[aGAM2aIGAM]\label{def:aGAMtoaIGAM}
    Let for $n=|\Shd|$ finite directions $\Shd\subset\Sd$, $\flSh$ be an \aGAM\ as introduced in \Cref{le:aGAMtoaIGAM}, with link function \link, \[\ghs(\cdot):=g(\sh, \cdot)\int_{\Ush}p(s)\dx[s],~ \sh\in\Shd,\] for some weighting function $g:\Sd\times\R\to\Rpz$ and $\mathcal{U}:=\{\Ush\}_{\sh\in\Shd}$ are disjoint environments such that $\dot\bigcup_{\sh\in\Shd}\Ush=\Sd$. Furthermore, let $\check{\varphi}^*$ s.t. \[\link\circ\flSh(x)=\sum_{\Shd}\psrcstar[\skprod{\sh}{x}]{\sh}, \quad \faxd.\]
    The \emph{aGAM2aIGAM}~\fw\ w.r.t. the aGAM $\flSh$ is given by
		\begin{equation}
		\link\circ\fw(x)=\int_{s\in\Sd}\psrw[\skprod{s}{x}]{s}\dx[s], \quad \faxd,
		\end{equation}
		with
		\[\psrw[r]{s}:=\sum_{\sh\in\Shd}\frac{\psrcstar[r]{\sh}}{\mUsh}\ind_{\Ush}(s), \quad (s,r)\in\Sd\times\Rdin.\]
	Here, $\mu$ denotes the $(d-1)$-dimensional surface-measure.
	\end{definition}

	\subsubsection{Proof of \Cref{le:aGAMtoaIGAM}}\label{sec:proofle:aGAMtoaIGAM}
	Throughout this subsection~\labelcref{sec:proofle:aGAMtoaIGAM} we will assume $|\mathcal{U}|\overset{n\to\infty}{\longrightarrow}0$.
	\begin{proof}[\hypertarget{proof:le:aGAMtoaIGAM}{Proof of \Cref{le:aGAMtoaIGAM}}]
		The two auxiliary functions \fcw\ and \fw\ defined above in \Cref{def:aIGAMtoaGAM,def:aGAMtoaIGAM} will play an important role in this proof.
		
		In the end we want to show the convergence of \flSh\ to $\argmin\FlSm$.
		We do so by proving that every \flSh\ gets closer to the corresponding~\fw\, and that \fw$\to\argmin\FlSm$ as $n\to\infty$. The first convergence is shown in \Cref{le:3}.
		The proof of the second convergence~$\fw\to\argmin\FlSm$ needs more steps---first we show the convergence $\FlSmb{\fw}\to \min\FlSm$ (in multiple steps based on \Cref{le:2,le:4}) to further imply with the help of \Cref{le:7} the convergence~$\fw\to\argmin\FlSm$.
		
		Following this strategy, we prove~\Cref{le:aGAMtoaIGAM} step by step:
		
		\begin{enumerate}[step 1, start=0]
			{\transparent{\meineTranzparenz} \item[step -0.5] In the upcoming stepts, we need the auxiliary \Cref{le:poincare,le:0.1NEW}}
			\item\label{itm:0} \Cref{le:0} shows
			\[		\lim_{n\to\infty} \sobnormmulti[{\fcw-\flS}]=0.\]
			\item\label{itm:1} It is directly clear that \[\FlShmb{\flSh}\leq\FlShmb{\fcw},\] because \flSh\ is optimal (see \Cref{def:adaptedGAM}).
			{\transparent{\meineTranzparenz}\item[step 1.5] The auxiliary \Cref{le:LossContinuous} will be needed for \ref{itm:2} and \ref{itm:4}}
			\item\label{itm:2} \Cref{le:2} shows \[ \lim_{n\to\infty}\FlShmb{\fcw}=\FlSmb{\flS}.\]
			\item\label{itm:3} \Cref{le:3} shows \[\lim_{n\to\infty}\sobnormmulti[{\flSh-\fw}]=0.\]
			\item\label{itm:4} \Cref{le:4} shows \[ \lim_{n\to\infty} \left|\FlShmb{\flSh}-\FlSmb{\fw}\right| =0.\]
			\item\label{itm:5} Because of optimality of $\flS$ we have \[\FlSmb{\flS}\leq\FlSmb{\fw}.\]
 			\item\label{itm:6} Combining \labelcref{itm:4,itm:1,itm:2,itm:5} we directly get:
			\begin{align*}
			\FlSmb{\fw}
			&\overset{\text{\ref{itm:4}}}\approx \texttransparent{\meineTranzparenz}{\FlSmb{\fcw}\pm\epsilon_1\leq}\\
			&\overset{\text{\ref{itm:1}}}\leq\texttransparent{\meineTranzparenz}{\FlSmb{\fcw}\pm\epsilon_1\approx}\\
			&\overset{\text{\ref{itm:2}}}\approx\FlSmb{\flS}\pm\epsilon_1\pm\epsilon_2
			\overset{\text{\ref{itm:5}}}\leq\FlSmb{\fw}\pm\epsilon_1\pm\epsilon_2,
			\end{align*}
			and thus:
			\begin{align*}
			\FlSmb{\fw}
			\overset{\substack{\text{\ref{itm:4}}\\
					\text{\ref{itm:2}}\\ \text{\ref{itm:1}}}}\lessapprox
			\FlSmb{\flS}\pm\epsilon_3 \hphantom{\pm\epsilon_2}
			\overset{\text{\ref{itm:5}}}\leq\FlSmb{\fw}\pm\epsilon_3,\hphantom{\pm\epsilon_2}
			\end{align*}
			which directly implies
			\begin{equation}\label{eq:itm:6}
			\lim_{n\to\infty} \FlSmb{\fw} = \FlSmb{\flS}.
			\end{equation}
			\item\label{itm:7} \Cref{le:7} (and \Cref{le:7convex} in case \Ltr\ is convex) shows
			\[ \lim_{n\to\infty}d_{\Woi[K,\Rdout]}\left(\fw,
	\argmin\FlSm\right) =0\]
			if one applies it on the result~\meqref{eq:itm:6} of \ref{itm:6} and using 
			\Cref{le:f_continuous_in_phi}.
			\item\label{itm:8} Combining \labelcref{itm:3,itm:7} with the triangle inequality directly results in the statement~\meqref{eq:aGAMtoaIGAM} we want show.
		\end{enumerate}
		
	\end{proof}
	
		\begin{lemma}[Poincar\'e-typed inequalities]\label{le:poincare}
		Let $\tilde{K}\subset\R$ be compact and $u:\Sd \to{\Wkp[\tilde{K},\Rdout]{1}{2}}$ s.t. for all $s\in\Sd$ it holds that  $\supp \frac{\partial}{\partial r}u(s)\subset \Kw$\[\lim_{r\to -\infty} u(s)(r)=0=\lim_{r\to -\infty} \frac{\partial}{\partial r}u(s)(r).\]
		\begin{enumerate}
		    \item\label{item:PoincareL2L2} 		Let $u\in\LtwoS[{\Wkp[{\tilde{K}}]{1}{2}}]$. Then, with $u':= \frac{\partial}{\partial r}u\in\LtwoS[{\Ltwo[{\tilde{K}}]{}{}}]$,
		\begin{equation}\label{eq:L2Poincare}
	\LtwonormonSd[u]{\Ltwo[\tilde{K}]}\le |\tilde{K}|\LtwonormonSd[u']{\Ltwo[{\tilde{K}}]}.
		\end{equation}
		\item\label{item:PoincareLiLi}  	Let $u\in\LtwoS[{\Wkp[{\tilde{K}}]{1}{\infty}}]$. Then, with $u':= \frac{\partial}{\partial r}u\in\LtwoS[{\Linfty[{\tilde{K}}]{}{}}]$,
		\begin{equation}\label{eq:WoiL2Poincare}
	\LtwonormonSd[u]{\Linfty[\tilde{K}]}\le |\tilde{K}|\LtwonormonSd[u']{\Linfty[{\tilde{K}}]}.
		\end{equation}
				\item\label{item:PoincareLiL2}  	Let $u\in\LtwoS[{\Wkp[{\tilde{K}}]{1}{\infty}}]$. Then, with $u':= \frac{\partial}{\partial r}u\in\LtwoS[{\Linfty[{\tilde{K}}]{}{}}]$,
		\begin{equation}\label{eq:WoiPoincare}
	\LtwonormonSd[u]{\Linfty[\tilde{K}]}\le \sqrt{|\tilde{K}|}\LtwonormonSd[u']{\Ltwo[{\tilde{K}}]}.
		\end{equation}
		\item\label{item:PoincareLiLiProductSpace}  	Let $u:\Sd\times{\tilde{K}}\to\Rdout$ with (weak) derivative $ u':=\frac{\partial}{\partial r}u\in\Linftymultidim[\Sd\times{\tilde{K}},\Rdout]$. Then
		\begin{equation}\label{eq:LinftyPoincare}
	\norm[u]{\Linfty[\Sd\times\tilde{K}]}\le|\tilde{K}| \norm[u']{\Linfty[{\Sd\times\tilde{K}}]}.
	\end{equation}
		\end{enumerate}

	\end{lemma}
	\begin{proof}
	\begin{enumerate}
	     \item We have that 
		\begin{align*}
		    	\LtwonormonSd[u]{\Ltwo[\tilde{K}]}^2&= \int_{\Sd}\int_{\tilde{K}}\|u(s,r)\|^2\,dr\,ds\\
		    	&=\int_{\Sd}\int_{\tilde{K}}\|\int_{(-\infty,r)\cap\tilde{K}}u'(s,x)\,dx\|^2\,dr\,ds\\
		    	&=\int_{\Sd}\int_{\tilde{K}}\|\int_{\tilde{K}}u'(s,x)\ind_{(-\infty,r)\cap\tilde{K}}(x)\,dx\|^2\,dr\,ds\\
		    	&\le\int_{\Sd}\int_{\tilde{K}}\norm[u'(s,\cdot)]{\Ltwo[{\tilde{K}}]{}}^2\norm[{\ind_{(-\infty,r)\cap\tilde{K}}}]{\Ltwo[{\tilde{K}}]}^2\,dr\,ds\\
		    	&\le |\tilde{K}|^2\int_{\Sd}\norm[u'(s,\cdot)]{\Ltwo[{\tilde{K}}]{}}^2\,ds\\
		    	&= |\tilde{K}|^2\LtwonormonSd[u']{\Ltwo[{\tilde{K}}]}^2.
		\end{align*}
		
		\item Analogously to the proof of item \eqref{item:PoincareL2L2}, it holds that
		\begin{align*}
		    	\LtwonormonSd[u]{\Linfty[\tilde{K}]}^2&= \int_{\Sd}\|u(s,\cdot)\|_{\Linfty[\tilde{K}]}^2\,ds\\
		    	&=\int_{\Sd}\|\int_{(-\infty,\cdot)\cap\tilde{K}}u'(s,x)\,dx\|_{\Linfty[\tilde{K}]}^2\,ds\\
		    	&=\int_{\Sd}\|\int_{\tilde{K}}u'(s, x)\ind_{(-\infty,\cdot)\cap\tilde{K}}(x)\,dx\|_{\Linfty[\tilde{K}]}^2\,ds\\
		   	&\le\int_{\Sd}\int_{\tilde{K}}\|u'(s, x)\|^2\|\ind_{(-\infty,\cdot)\cap\tilde{K}}(x)\|_{\Linfty[\tilde{K}]}^2\,dx\,ds\\
		    	&\le |\tilde{K}|^2\int_{\Sd}\norm[u'(s,\cdot)]{\Linfty[{\tilde{K}}]{}}^2\,ds\\
		    	&= |\tilde{K}|^2\LtwonormonSd[u']{\Linfty[{\tilde{K}}]}^2.
		\end{align*}
		\item Analogously to the proof of item \eqref{item:PoincareLiLi}, it holds that
		\begin{align*}
		    	\LtwonormonSd[u]{\Linfty[\tilde{K}]}^2&= \int_{\Sd}\|u(s,\cdot)\|_{\Linfty[\tilde{K}]}^2\,ds\\
		    	&=\int_{\Sd}\|\int_{(-\infty,\cdot)\cap\tilde{K}}u'(s,x)\,dx\|_{\Linfty[\tilde{K}]}^2\,ds\\
		    	&=\int_{\Sd}\|\int_{\tilde{K}}u'(s, x)\ind_{(-\infty,\cdot)\cap\tilde{K}}(x)\,dx\|_{\Linfty[\tilde{K}]}^2\,ds\\
		   	&\le\int_{\Sd}\sup_{r\in\Kw}\{\|u'(s, \cdot)\|_{\Ltwo[\Kw]{}}^2\|\ind_{(-\infty,r)\cap\tilde{K}}(\cdot)\|_{\Ltwo[\tilde{K}]}^2\}\\\
		   	&=\int_{\Sd}\|u'(s, \cdot)\|_{\Ltwo[\Kw]{}}^2\,ds \,\sup_{r\in\Kw}\{\|\ind_{(-\infty,r)\cap\tilde{K}}(\cdot)\|_{\Ltwo[\tilde{K}]}^2\}\\
		    	&\le |\tilde{K}|\int_{\Sd}\norm[u'(s,\cdot)]{\Ltwo[{\tilde{K}}]{}}^2\,ds\\
		    	&= |\tilde{K}|\LtwonormonSd[u']{\Ltwo[{\tilde{K}}]}^2.
		\end{align*}
			\item Analogously to the proof of item\eqref{item:PoincareLiLi}, it holds that for every $s\in\Sd$
		\begin{align*}
		    	\norm[u(s,\cdot)]{\Linfty[\tilde{K}]}&=\|\int_{(-\infty,\cdot)\cap\tilde{K}}u'(s,x)\,dx\|_{\Linfty[\tilde{K}]}\\
		    	&=\|\int_{\tilde{K}}u'(s, x)\ind_{(-\infty,\cdot)\cap\tilde{K}}(x)\,dx\|_{\Linfty[\tilde{K}]}\\
		   	&\le\int_{\tilde{K}}\|u'(s, x)\|\|\ind_{(-\infty,\cdot)\cap\tilde{K}}(x)\|_{\Linfty[\tilde{K}]}\,dx\\
		    	&\le |\tilde{K}|\norm[u'(s,\cdot)]{\Linfty[{\tilde{K}}]{}}.
		\end{align*}
		Therefore,
		\begin{align*}
		    \norm[u]{\Linfty[\Sd\times\tilde{K}]}\le|\tilde{K}| \norm[u']{\Linfty[{\Sd\times\tilde{K}}]}.
		\end{align*}
	\end{enumerate}
	\end{proof}

	\begin{lemma}[\ref{itm:0}]\label{le:0}
		For any choice of hyper parameter $\lambda>0$, weighting function $g:\Sd\times\R\to\Rpz$ and $K\subset\Rdin$ compact,  with $\bigcup_{s\in\Sd}\Set{\skprod{s}{x}\mid x\in K}\subset\tilde{K}$, the aIGAM2aGAM $\fcw$ with $n=|\Shd|$ directions converges to the aIGAM $\flS$ w.r.t. $\sobnormmulti$ as $n$ tends to infinity\footnote{Recall that throughout this subsection~\labelcref{sec:proofle:aGAMtoaIGAM} we assume $|\mathcal{U}|\overset{n\to\infty}{\longrightarrow}0$.}, i.e.,
		\begin{displaymath}
		\lim_{n\to\infty} \sobnormmulti[{\fcw-\flS}]=0.
	\end{displaymath}
\end{lemma}
\begin{proof}
	By \ref{le:0.1NEW} we have that $\psstar[]\in\LtwoS[\Woimultidim]$ and hence by \Cref{def:aIGAMtoaGAM}, $\link\circ\fcw$ is given as the Bochner integral (\cite{guide2006infinite}) of the simple function 
	\begin{align*}
	&\phi_{\Shd}:\Sd\to{\Woimultidim},\\
	&\phi_{\Shd}(s) =\sum_{\Shd}\psrstar[\skprod{\sh}{\cdot}]{\sh}\ind_{\Ush}(s).
	\end{align*}
	Since by \Cref{le:0.1NEW} item \eqref{item:phistarContOnRdin} $s\mapsto\psrstar[\skprod{s}{\cdot}]{s}$ is continuous the difference quotient 
		\begin{align*}
	    \frac{\|\psrstar[\skprod{s}{\cdot}]{s}-\psrstar[\skprod{\sh}{\cdot}]{\sh}\|_{\Wkp[\tilde{K}]{1}{\infty}}}{\|s-\sh\|} 
	\end{align*}
	is bounded by a constant $C_S$ on $\Sd$. We then get 
	\begin{align*}
	    \int_{\Sd}\sobnormmulti[{\psrstar[{\skprod{s}{\cdot}}]{s}-\phi_{\Shd}(s)}]\dx[s]&=\int_{\Sd}\sobnormmulti[{\psrstar[{\skprod{s}{\cdot}}]{s}-\sum_{\Shd}\psrstar[\skprod{\sh}{\cdot}]{\sh}\ind_{\Ush}(s)}]\dx[s]\\
	    &\le\int_{\Sd}\sum_{\Shd}\sobnormmulti[{\psrstar[{\skprod{s}{\cdot}}]{s}-\psrstar[\skprod{\sh}{\cdot}]{\sh}}]\ind_{\Ush}(s)\dx[s]\\
	    &\le\int_{\Sd}\sum_{\Shd}\|s-\sh\|C_S\ind_{\Ush}(s)\dx[s]\\
	    &\le|\mathcal{U}||\Sd|C_S.
	\end{align*}
	And since $\lim_{n\to\infty}|\U|=0$, we obtain
	\[\lim_{n\to\infty}\int_{\Sd}\sobnormmulti[{\psrstar[{\skprod{s}{\cdot}}]{s}-\phi_{\Shd}(s)}]\dx[s]=0.\]
	Therefore, we have 
	\[\link\circ\fcw\substack{{\Woimultidim}\\
	\longrightarrow\\
	n\to\infty}\int_{\Sd}\psrstar[\skprod{s}{\cdot}]{s}\dx[s]=\link\circ\flS.\]
	Thus, by Lipschitz continuity of $\linkinv$, the result follows.
\end{proof}

	


\begin{lemma}[$\Ltr(f_n)\to\Ltr (f)$]\label{le:LossContinuous}
	Let $K\subset\Rdin$ be s.t. \Cref{as:generalloss} is satisfied, and $(f_n)_{n\in\N}$ be a sequence of continuous functions with piece-wise continuous derivatives which converges w.r.t.\ $\Woi[K, \Rdout]$
to a function $f:\Rdin\to\Rdout$ then the training loss $\Ltr$ of $f_n$ converges to $\Ltrb{f}$ as $n$ tends to infinity, i.e.
	\begin{equation}
	\lim\Ltr(f_n)=\Ltr(f).
	\end{equation}
	
\end{lemma}

\begin{proof}
By \Cref{as:generalloss} there exists a finite Borel measure $\nu$ and some $p>1$ s.t. \Ltr\ is continuous w.r.t.\ $\Wkp[K,\nu]{1}{p}$. Since for each component $f_k$ of $f$
\begin{align*}
    \sobnormop[f_k]&=\left(\int_K (f_k)^p\,d\nu\right)^{\frac{1}{p}}+\left(\int_K (f_k^{'})^p\,d\nu\right)^{\frac{1}{p}}\\
    &\le\left(\sup_{x\in K}|f_k(x)|+\sup_{x\in K}|f_k^{'}(x)|\right)\nu(K)^{\frac{1}{p}}\\
    &\le 2\nu(K)^{\frac{1}{p}}\sobnorm[f_k],
\end{align*}
where the last inequality follows from \Cref{rem:sobnormoi}. Therefore, \Ltr\ is continuous w.r.t.\ $\Woi[K, \Rdout]$ as well and the result follows.
\end{proof}

\begin{lemma}[\ref{itm:2}]\label{le:2}
	For any $\lambda>0$
	we have
	\begin{equation}\label{eq:statement2}
	\lim_{n\to\infty}\FlShmb{\fcw}=\FlSmb{\flS},
	\end{equation}
	with $g$ and $\check{g}$ as defined in \Cref{le:aGAMtoaIGAM}.
\end{lemma}
\begin{proof}
	We start by showing the convergence of the regularization terms, i.e.
	\begin{equation}\label{eq:penalty}
	\lim_{n\to\infty}\sum_{\sh\in\Shd}\int_{\supp(\ghs)}\frac{\twonorm[ {\psrcw[r]{\sh}^{''} }]^2}{\ghs(r)} \dx[r]	=\int_{\Sd}\int_{\supp (\gs)} \frac{\twonorm[ {\psrstar{s}}^{''} ]^2}{\gsr{s}} \dx[r]\dx[s]		. 
	\end{equation}
	From \Cref{def:aIGAMtoaGAM} and the definition of $\check{g}$ we get
	\begin{equation*}\sum_{\sh\in\Shd}\int_{\supp(\ghs)}\frac{\twonorm[ { \psrcw[r]{\sh}^{''} }]^2}{\ghs(r)} \dx[r]=\sum_{\sh\in\Shd}\frac{\int_{\supp(\ghs)}\frac{\twonorm[ { \psrstar[r]{\sh}^{''} }]^2}{g(\sh,r)} \dx[r]}{\frac{\int_{\Ush}p(s)\dx[s]}{\mUsh}}\mUsh,
	\end{equation*}
	which is a Riemann sum of the map
	\begin{align*}
	    \sh\in\Sd\mapsto \frac{\int_{\supp(\ghs)}\frac{\twonorm[ { \psrstar[r]{\sh}^{''} }]^2}{g(\sh,r)} \dx[r]}{\frac{\int_{\Ush}p(s)\dx[s]}{\mUsh}}\in\R.
	\end{align*}
	Thus we have proven the convergence of the regularization terms \eqref{eq:penalty}. Together with \Cref{le:0,le:LossContinuous}, \eqref{eq:statement2} follows.
	
\end{proof}

\begin{lemma}[\ref{itm:3}]\label{le:3}
	For any $\lambda>0$ 
	we have
	\begin{equation}\label{eq:statement3}
\lim_{n\to\infty}\sobnormmulti[{\flSh-\fw}]=0,
	\end{equation}
	with $\ghs$ as defined in \Cref{le:aGAMtoaIGAM}.
\end{lemma}
\begin{proof}
	By \Cref{def:aGAMtoaIGAM}, we have on the one hand
	\begin{equation*}
	    \link\circ\fw=\int_{s\in\Sd}\sum_{\sh\in\Shd}\psrcstar[{\skprod{s}{\cdot}}]{\sh}\frac{{\ind_{\Ush(s)}}}{\mUsh}\dx[s]=\sum_{\sh\in\Shd}\frac{1}{\mUsh}\int_{s\in\Ush}\psrcstar[\skprod{s}{\cdot}]{\sh}\dx[s]
	    \end{equation*}
	On the other,
	\begin{equation*}
	    \link\circ\flSh=\sum_{\sh\in\Shd}\psrcstar[\skprod{\sh}{\cdot}]{\sh}=\sum_{\sh\in\Shd}\frac{1}{\mUsh}\int_{s\in\Ush}\psrcstar[\skprod{\sh}{\cdot}]{\sh}\dx[s].
	\end{equation*}
	Because the map $\Sd\to\Woimultidim: s\mapsto\psrcstar[\skprod{s}{\cdot}]{\sh}$ (see \Cref{le:0.1NEW} \cref{item:phistarhatContOnRdin}) is continuous,
	 the difference quotient 
		\begin{align*}
	    \frac{\|\psrstar[\skprod{s}{\cdot}]{s}-\psrstar[\skprod{\sh}{\cdot}]{\sh}\|_{\Wkp[\tilde{K}]{1}{\infty}}}{\|s-\sh\|} 
	\end{align*}
	is bounded by a constant $C_S$ on $\Sd$. We then get 
	\begin{align*}
	   &\sobnormmulti[{\sum_{\sh\in\Shd}\frac{1}{\mUsh}\int_{s\in\Ush}\psrcstar[\skprod{\sh}{\cdot}]{\sh}-\psrcstar[\skprod{s}{\cdot}]{\sh}\dx[s]}]\\
	   &\le{\sum_{\sh\in\Shd}\frac{1}{\mUsh}\int_{s\in\Ush}\sobnormmulti[{\psrcstar[\skprod{\sh}{\cdot}]{\sh}-\psrcstar[\skprod{s}{\cdot}]{\sh}}]\dx[s]}\\
	    &\le{\sum_{\sh\in\Shd}\frac{1}{\mUsh}\int_{s\in\Ush}\|s-\sh\|C_S\dx[s]}\\
	    &\le|\mathcal{U}||\Sd|C_S\sum_{\sh\in\Shd}\frac{1}{\mUsh}=|\mathcal{U}|C_S.
	\end{align*}
	And since $\lim_{n\to\infty}|\U|=0$, we obtain
	\begin{equation*}
	    \nlim\sobnormmulti[{\sum_{\sh\in\Shd}\frac{1}{\mUsh}\int_{s\in\Ush}\psrcstar[\skprod{\sh}{\cdot}]{\sh}-\psrcstar[\skprod{s}{\cdot}]{\sh}\dx[s]}]=0.
	\end{equation*}
	By Lipschitz continuity of $\linkinv$, \eqref{eq:statement3} follows.
\end{proof}

\begin{lemma}[\ref{itm:4}]\label{le:4}
	For any $\lambda>0$ 
	we have
	\begin{equation}\label{eq:statement4}
	\lim_{n\to\infty} \left|\FlShmb{\flSh}-\FlSmb{\fw}\right| =0,
	\end{equation}
	with $\ghs$ as defined in \Cref{le:aGAMtoaIGAM}.
\end{lemma}
\begin{proof}
	\Cref{le:LossContinuous,le:3} show together that
	\begin{equation*}
	\lim_{n\to\infty}\left|\Ltrb{\flSh}-\Ltrb{\fw}\right|=0.
	\end{equation*}
	So it is sufficient to show:
	\begin{equation}\label{eq:4:penalty}
		\lim_{n\to\infty}\left|\sum_{\sh\in\Shd}\frac{\int_{\supp(\ghs)}\frac{\twonorm[ {\psrcstar[r]{\sh}^{''}}]^2}{g(\sh,r)}\dx[r]}{\int_{\Ush}p(s)\dx[s]}-\int_{s\in\Sd}\frac{\int_{\supp(\gs)}\frac{\twonorm[ {\psrw[r]{s}^{''}}]^2}{g(s,r)}\dx[r]}{p(s)}\dx[s]\right|=0.
	\end{equation}
Since by \Cref{def:aGAMtoaIGAM}
		\[\psrw[r]{s}:=\sum_{\sh\in\Shd}\frac{\psrcstar[r]{\sh}}{\mUsh}\ind_{\Ush}(s), \quad (s,r)\in\Sd\times\Rdin,\]
		we get (the first inequality holds true since $\{\Ush\}_{\sh\in\Shd}$ are disjoint environments)
		\begin{align*}
		    &\left|\sum_{\sh\in\Shd}\frac{\int_{\supp(\ghs)}\frac{\twonorm[ {\psrcstar[r]{\sh}^{''}}]^2}{g(\sh,r)}\dx[r]}{\int_{\Ush}p(s)\dx[s]}-\int_{s\in\Sd}\frac{\int_{\supp(\gs)}\frac{\twonorm[ {\psrw[r]{s}^{''}}]^2}{g(s,r)}\dx[r]}{p(s)}\dx[s]\right|\\
		    =&\left|\sum_{\sh\in\Shd}\frac{\int_{\supp(\ghs)}\frac{\twonorm[ {\psrcstar[r]{\sh}^{''}}]^2}{g(\sh,r)}\dx[r]}{\int_{\Ush}p(s)\dx[s]}-\int_{s\in\Sd}\frac{1}{p(s)}\int_{\supp(\gs)}\frac{\sum_{\sh\in\Shd}\frac{\ind_{\Ush}(s)}{\mUsh^2}\twonorm[ {\psrcstar[r]{\sh}^{''}}]^2}{g(s,r)}\dx[r]\dx[s]\right|\\
		    \le\sum_{\sh\in\Shd}&\left|\frac{\int_{\supp(\ghs)}\frac{\twonorm[ {\psrcstar[r]{\sh}^{''}}]^2}{g(\sh,r)}\dx[r]}{\int_{\Ush}p(s)\dx[s]}-\int_{s\in\Ush}\frac{1}{p(s)\mUsh^2}\int_{\supp(\gs)}\frac{\twonorm[ {\psrcstar[r]{\sh}^{''}}]^2}{g(s,r)}\dx[r]\dx[s]\right|\\
		    \le\sum_{\sh\in\Shd}&\left|\int_{\R}\twonorm[ {\psrcstar[r]{\sh}^{''}}]^2\left[\frac{\ind_{\supp(\ghs)}(r)}{g(\sh,r)\int_{\Ush}p(s)\dx[s]}-\frac{1}{\mUsh^2}\int_{s\in\Ush}\frac{\ind_{\supp(\gs)}(r)}{p(s)g(s,r)}\dx[s]\right]\dx[r]\right|\\
		   =\sum_{\sh\in\Shd}&\left|\int_{\R}\twonorm[ {\psrcstar[r]{\sh}^{''}}]^2\frac{1}{\mUsh}\underbrace{\left[\frac{\ind_{\supp(g(\sh,\cdot))}(r)}{g(\sh,r)\frac{1}{\mUsh}\int_{\Ush}p(s)\dx[s]}-\frac{1}{\mUsh}\int_{s\in\Ush}\frac{\ind_{\supp(g(s,\cdot))}(r)}{p(s)g(s,r)}\dx[s]\right]}_{=:A}\dx[r]\right|.\\
		\end{align*}
	 An application of the mean-value theorem to the integrals in $A$ yields $\lim_{n\to\infty}|A|=0$. Thus, for $n$ large enough, the above can be bounded by
		\begin{align*}
		    \sum_{\sh\in\Shd}\int_{\R}\twonorm[ {\psrcstar[r]{\sh}^{''}}]^2\frac{1}{\mUsh}\dx[r]\,\epsilon(n)\le  \Ltr(0)\epsilon(n),\\
		\end{align*}
	for some $\epsilon(n)>0$ with $\lim\limits_{n\to\infty}\epsilon(n)=0$, since due to optimality $\FlShmb{\flSh}\le\Ltr(0)$ and thus $\PgShm(\flSh)\le\Ltr(0)$. Hence, taking the limit $n\to\infty$ finally proves \cref{eq:statement4}.
	\end{proof}



\begin{lemma}[\ref{itm:7} convex]\label{le:7convex}
	Assume that \Ltr\ is convex. For any $\lambda>0$ 
	for any sequence of functions $\phin\in \Tw:=L^2(\Sd,\Wti[\tilde{K},\Rdout])\cap\T$ for some compact $\tilde{K}\subset\R$ such that
	\begin{equation}\label{eq:condition7}
	\nlim \FlSmb{f^n}=\FlSmb{\flS},
	\end{equation}
	with $\link\circ\fn=\int_{s\in\Sd}\phin_s(\skprod{s}{\cdot})\dx[s]$ then it follows that: 
	\begin{equation}\label{eq:statement7}
	\nlim \int_{\Sd}\sobnormmulti[{\phin_s(\skprod{s}{\cdot})-\psrstar[{\skprod{s}{\cdot}}]{s}}]
	= 0,
	\end{equation}
	with $K\subset\Rdin$ compact.
\end{lemma}
\begin{proof}
	For any $n\in\mathbb{N}$, define the functions
	\begin{equation}\label{eq:7defu}
	\un:=\varphi^* - \phin \quad\in\Tw.
	\end{equation}
	
	Define the regularization term of \FlSm\ as:
	\begin{equation}\label{eq:PgSp}
	\PgSpmb{\phi}:=\gbar\int_{\Sd}\int_{\supp (\gs)} \frac{\twonorm[{{\phi_{s}(r)}^{''} }]^2}{\gsr{s}} \dx[r]\dx[s]	
	\end{equation}
	This penalty~$\PgSp$ is obviously a quadratic form.
	Note that $\frac{\phin+\varphi^*}{2}\in\T$.
	Since the training loss~\Ltr\ is convex, we get the inequality:
	\begin{equation}\label{eq:konvexL}
	\Ltrb{\frac{\fn+\flS }{2}}\leq
	\frac{\Ltrb{\fn}}{2}+\frac{\Ltrb{\flS }}{2}.
	\end{equation}
	Since the penalty~$\PgSp$ is a quadratic form, we get with the help of some algebraic calculations the inequality:
	\begin{equation}\label{eq:quadraticP}
	\PgSpmb{\frac{\phin+\varphi^*}{2}}\leq
	\frac{\PgSpmb{\phin}}{2}+\frac{\PgSpmb{\varphi^*}}{2}-\frac{\PgSpmb{\un}}{4}.
	\end{equation}
	Combining the inequalities~\meqref{eq:konvexL} and~\meqref{eq:quadraticP} results in:
	\begin{equation}\label{eq:7importantIneq}
	\FlSmb{\frac{\fn+\flS}{2}}\leq
	\underbrace{\frac{\FlSmb{\fn}+\FlSmb{\flS}}{2}}_{%
		\textstyle\overset{\meqref{eq:condition7}}{\approx}\FlSmb\flS\pm\epsilon}
	-\frac{\PgSpmb{\un}}{4}.
	\end{equation}
	Together with the optimality of \flS\ this result leads directly to:
	\begin{subequations}\label{eq:7:secondlastIneq}
		\begin{align}
		\FlSmb\flS
		&\overset{\text{optimality}}{\leq}
		\FlSmb{\frac{\fn+\flS}{2}}\\
		&\underset{\hphantom{\text{optimality}}}{\overset{\text{\meqref{eq:7importantIneq}}}{\lessapprox}}
		\FlSmb\flS\pm\epsilon
		-\frac{\PgSpmb{\un}}{4}.
		\end{align}
	\end{subequations}
	By subtracting $\left(\FlSmb\flS-\frac{\PgSpmb{\un}}{4}\right)$ from both sides of ineq.~\meqref{eq:7:secondlastIneq} and multiplying by $4$ we get:
	\begin{equation*}
	\PgSpmb{\un}\overset{\text{\meqref{eq:7:secondlastIneq}}}{\lessapprox}\pm 4\epsilon,
	\end{equation*}
	which implies that
	\begin{equation}\label{eq:7plimunpmtIs0}
	\nlim\PgSpmb{\un}=0.
	\end{equation}
	First we will show that the weak second derivative~${\un}^{''}$ converges to zero:
	\begin{align}\label{eq:7secondDerivativeunpIsLow}
	\int_{\Sd}\Ltwonormons[{\un_s}^{''}]{\R}\dx[s]\leq\PgSpmb\un\max_{\{(s,x):s\in\Sd, x\in\supp(\gs)\}}\gsr[x]{s},
	\end{align}
	because for any $s\in\Sd$, $\un_s$ has zero second derivative outside $\supp(\gs)$.
	Thus, \[\nlim\Ltwonormons[{\un}^{''}]{\Sd,L^2(\R)}=0,\](by combining \cref{eq:7plimunpmtIs0,eq:7secondDerivativeunpIsLow}) and, since $\un\in\Tw$,\[\nlim\Ltwonormons[{\un}^{''}]{\Sd,L^\infty(\tilde{K})}=0,\]
	for any compact $\tilde{K}\subset\R$.
	This can be used to apply two times the Poincar\'e-typed \Cref{le:poincare} (item \eqref{item:PoincareLiLi}, first on ${\un_s}^{'}$ then on ${\un_s}$) to get (after an application of the Cauchy-Schwarz inequality)
	\begin{equation*}
	\nlim\int_{\Sd}\norm[\un_s]{\Woi[\tilde{K},\Rdout]}\dx[s]=0,
	\end{equation*}
	because $\un_s\in\Tw$ satisfies the boundary conditions at $\Cgsl$ and $\Cgsu$ (cp. \Cref{rem:compactSupp}) which in turn follows by compactness of the support of $\gs$, $s\in\Sd$. In particular,
	\begin{equation*}
	\nlim\int_{\Sd}\norm[\un_s]{\Woi[\tilde{K},\Rdout]}\dx[s]=0.
	\end{equation*}
	Due to 
	\[\sobnormmulti[{\un_s(\skprod{s}{\cdot})}]\le\norm[\un_s]{\Woi[\tilde{K},\Rdout]}\sobnormmulti[\skprod{s}{\cdot}]\le C_K\norm[\un_s]{\Woi[\tilde{K},\R]}, \]
	\meqref{eq:statement7} follows.
\end{proof}

\begin{lemma}\label{le:compactnessoflevelset}
  For any $c>0$, the sub-level set $\mathcal{K}:=\left\{\phi\in \Tw:\PgSpmb{\phi}\leq c\right\}$ is sequentially compact w.r.t. $\LtwoS[{\Wkp[\Kw]{1}{\infty}}]$ for $\Tw$ and $\Kw\subset\R$ compact as in \cref{eq:Tw}.
\end{lemma}

\begin{proof}
First note that $ \Wkp[{\Kw}]{2}{2}$ embeds compactly into ${\Wkp[\Kw]{1}{\infty}}$ (see for instance the proof of \cite[Lemma A.19]{ImplRegPart1V3}). By Aubin-Lions Lemma (see for instance \cite[Theorem 6.3.]{Arendt_2018}) we therefore have the compact embedding
\[\LtwoS[{\Hk[\tilde{K}]{2}}]\cap\Wkp[{\Sd,\Woi[\tilde{K}]}]{1}{2}\subset\subset\LtwoS[{\Wkp[\tilde{K}]{1}{\infty}}].\]
The set $\mathcal{K}\subset{\LtwoS[{\Hk[\tilde{K}]{2}}]\cap\Wkp[{\Sd,\Woi[\tilde{K}]}]{1}{2}}$ is bounded: by \cref{eq:W22leP} it is bounded in $\LtwoS[{\Hk[\tilde{K}]{2}}]$, and by \Cref{le:0.1NEW} item \eqref{item:phistarinW12} (for bounding the first weak derivative $\partial\phi/\partial_s:\Sd\to\Woi[\Kw]$) together with \cref{eq:W22leP} and \Cref{le:poincare} \eqref{item:PoincareLiL2} (for bounding the zero\textsuperscript{th} derivative $\phi:\Sd\to\Woi[\Kw]$) it is bounded in $\Wkp[{\Sd,\Woi[\tilde{K}]}]{1}{2}$. Therefore, it is sequentially compact w.r.t. $\LtwoS[{\Wkp[\Kw]{1}{\infty}}]$.
\end{proof}

\begin{lemma}\label{le:f_continuous_in_phi}
Let $K\subset\Rdin$ compact and define $\tilde{K}:=\overline{\bigcup_{s\in\Sd}\Set{{\skprod{s}{x}\mid x\in K}}}$. Then, the map
\begin{align*}
     &\LtwoS[{\Wkp[{\tilde{K}}]{1}{\infty}}]
      \to {\Wkp[{{K}}]{1}{\infty}},\\
         &\phi\mapsto \link^{-1}\left(\int_{s\in\Sd}\phi_s(\skprod{s}{\cdot})\, ds\right),
\end{align*}
is Lipschitz continuous.
\end{lemma}
\begin{proof}
For each $s$ the set $\tilde{K}_s:=\Set{{\skprod{s}{x}\mid x\in K}}\subset\tilde{K}$ and $\tilde{K}$ is compact.
We then get
\begin{align*}
   \supnormmulti[{ \int_{\Sd}\phi_s(\skprod{s}{\cdot})\,ds}]&\le\int_{\Sd}\supnormmulti[{ \phi_s(\skprod{s}{\cdot})}]\,ds\\
   &=\int_{\Sd}\norm[{ \phi_s({\cdot})}]{L^\infty(\tilde{K}_s,\Rdout)}\,ds\\
   &\le \int_{\Sd}\norm[{ \phi_s({\cdot})}]{L^\infty(\tilde{K},\Rdout)}\,ds\\
     &\le\sqrt{\int_{\Sd}\norm[{ \phi_s({\cdot})}]{L^\infty(\tilde{K},\Rdout)}^2\,ds}\\
     &=\LtwonormonSd[\phi]{L^\infty(\tilde{K},\Rdout)}\\
     &\le\LtwonormonSd[\phi]{\Wkp[{\tilde{K},\Rdout}]{1}{\infty}},
\end{align*}
where the third inequality follows by Cauchy-Schwarz.
Moreover, we analogously get
\begin{align*}
   \supnormmulti[{ \int_{\Sd}\phi_s(\skprod{s}{\cdot})\,ds}^{'}]&\le\int_{\Sd}\supnormmulti[{ \phi_s(\skprod{s}{\cdot})^{'}}]\underbrace{\norm[s]{}}_{=1}\,ds\\
   &=\int_{\Sd}\norm[{ \phi_s({\cdot})^{'}}]{L^\infty(\tilde{K}_s,\Rdout)}\,ds\\
   &\le \int_{\Sd}\norm[{ \phi_s({\cdot})^{'}}]{L^\infty(\tilde{K},\Rdout)}\,ds\\
     &\le\sqrt{\int_{\Sd}\norm[{ \phi_s({\cdot})^{'}}]{L^\infty(\tilde{K},\Rdout)}^2\,ds}\\
     &=\LtwonormonSd[\phi^{'}]{L^\infty(\tilde{K},\Rdout)}\\
     &\le\LtwonormonSd[\phi]{\Wkp[{\tilde{K},\Rdout}]{1}{\infty}}.
\end{align*}
Thus in total, 
\begin{align*}
    \LtwoS[{\Wkp[{\tilde{K}}]{1}{\infty}}]
      \to {\Wkp[{{K}}]{1}{\infty}}, \phi\mapsto\int_{s\in\Sd}\phi_s(\skprod{s}{\cdot})\, ds\
\end{align*}
is Lipschitz continuous. Since the inverse of a link function $\link$ is Lipschitz continuous too, the result follows.
\end{proof}
\begin{lemma}[\ref{itm:7}]\label{le:7}
	For any sequence of functions $\phin\in \Tw:=\T\cap\Wkp[{\Sd,\Woi[\tilde{K}]}]{1}{2}$ for some compact $\tilde{K}\subset\R$ such that
	\begin{equation}\label{eq:condition7generalLoss}
	\nlim \FlSmb{f^n}=\min\FlSm,
	\end{equation}
		with $\link\circ\fn=\int_{s\in\Sd}\phin_s(\skprod{s}{\cdot})\dx[s]$, it holds that
	\begin{equation}\label{eq:statement7generalLoss}
\lim_{n\to\infty}d_{\LtwoS}\left(\phin,
	\argmin F^\phi\right) =0,
	\end{equation}
	where $F^\phi$ as in \Cref{def:Philevelproblem},
	i.e., $\forall\epsilon>0:\exists n_0\in\N:\forall n>n_0$
	such that
	\begin{equation}
	    \exists\phi^*\in\argmin F^\phi:\LtwonormonSd[{\phin-\phi^* }]{\Woi[\tilde{K}]}<\epsilon.
	\end{equation}
	
\end{lemma}
\begin{proof}
	
In order to show that \cref{eq:statement7generalLoss} holds true, assume on the contrary that
$\exists\epsilon>0:\forall n_0\in\N:\exists n>n_0$
	such that
	\begin{equation}\label{eq:farAwayFromArgMin}
	     \forall\phi^*\in\argmin F^\phi:\LtwonormonSd[{\phin-\phi^* }]{\Woi[\tilde{K}]}>\epsilon.
	\end{equation}
This implies the existence of a sub-sequence $\phi^{\tilde{n}}$ of which every element fulfills \meqref{eq:farAwayFromArgMin}. 
\Cref{eq:condition7} implies that for every $\epsilon_\text{\meqref{eq:condition7generalLoss}}>0$, $|\FlSmb{f^n}-\min\FlSm|<\epsilon_\text{\meqref{eq:condition7generalLoss}}$ for all $n$ large enough, with $f^n$ again given via $\link\circ\fn=\int_{s\in\Sd}\phin_s(\skprod{s}{\cdot})\dx[s]$. Which is equivalent to $|F^\phi(\phi^n)-\min F^\phi|<\epsilon_\text{\meqref{eq:condition7generalLoss}}$ for all $n$ large enough.
Therefore, also $|F^\phi(\phi^{\tilde{n}})-\min F^\phi|<\epsilon_\text{\meqref{eq:condition7generalLoss}}$ for all $\tilde{n}$ large enough. 
Define the set
$\mathcal{K}:=\{\phi\in \Tw:\PgSpmb{\phi}\le \frac{L(0)+{\epsilon}_\text{\meqref{eq:condition7generalLoss}}}{\lambda}\}$.
It then holds that for large $\tilde{n}$
    that $\phi^{\tilde{n}}\in\mathcal{K}$.
     
     Since $\mathcal{K}$ is  sequentially compact w.r.t.~$\LtwoS$ by \Cref{le:compactnessoflevelset}, there exists a sub-sequence $\phi^{\tilde{\tilde{n}}}$ of $\phi^{\tilde{n}}$ which converges w.r.t.~$\LtwoS$ to a limit $\tilde{\phi}$ in $\tilde{\mathcal{K}}$. 
     By \Cref{le:LossContinuous} and \Cref{le:f_continuous_in_phi} $\Ltr$ is continuous w.r.t.\ $\LtwoS[{\Wkp[{\tilde{K}}]{1}{\infty}}]$ and we define
     \[\lim_{\tilde{\tilde{n}}\to\infty}\Ltr\left(f^{\tilde{\tilde{n}}}\right)=\Ltr\left(\tilde{f}\right)=:L^*,\]
     (Here, analogously, $f^{\tilde\tilde{n}}$ and $\tilde{f}$ are given via $\link\circ f^{\tilde\tilde{n}}=\int_{s\in\Sd}\phi^{\tilde\tilde{n}}_s(\skprod{s}{\cdot})\dx[s]$ and $\link\circ \tilde{f}=\int_{s\in\Sd}\tilde\phi_s(\skprod{s}{\cdot})\dx[s]$.) and
     \[\lim_{\tilde{\tilde{n}}\to\infty}\PgSp\left(\phi^{\tilde{\tilde{n}}}\right)=\min_{\T} F^\phi-L^*=:P^*.\]
     We now proceed to show that $\phi^{\tilde{\tilde{n}}}$ is a $\LtwoS[{\Wkp[{\tilde{K}}]{2}{2}}]$-Cauchy sequence.
     By continuity of $\Ltr$ there exists a $\delta>0$ such that
     \[\Ltr\left(\frac{f^{\tilde{\tilde{n}}}+f^{\tilde{\tilde{m}}}}{2}\right)<L^*+\epsilon_L,\]
     for $\tilde{\tilde{n}},\tilde{\tilde{m}}$ large enough such that 
     \[\left\|\frac{\phi^{\tilde{\tilde{n}}}+\phi^{\tilde{\tilde{m}}}}{2}-\tilde{\phi} \right\|_{\LtwoS[{\Wkp[{\tilde{K}}]{1}{\infty}}]}<\delta_L.\]
     By $M$-strong convexity\footnote{In \meqref{eq:proofStrongyConvexFInalConstant}, one can see the explicit form of $M:=\gbar/\left(\left(|\tilde{K}|^4+|\tilde{K}|^2+1\right)\max_{x\in\supp(g)}g(x)\right)$.} of $\PgSpm$ (see \Cref{le:PfuncStronglyConvex}) we have
     \begin{align*}
     \PgSpm\left(\frac{\phi^{\tilde{\tilde{n}}}+\phi^{\tilde{\tilde{m}}}}{2}\right)&\leq\frac{1}{2}\PgSpm\left(\phi^{\tilde{\tilde{n}}}\right)+\frac{1}{2}\PgSpm\left(\phi^{\tilde{\tilde{m}}}\right)-\frac{M}{8}\left\|\phi^{\tilde{\tilde{n}}}-\phi^{\tilde{\tilde{m}}}\right\|_{\LtwoS[{\Wkp[{\tilde{K}}]{2}{2}}]}\\
     &<\frac{P^*+\epsilon_P}{2}+\frac{P^*+\epsilon_P}{2}-\frac{M}{8}\left\|\phi^{\tilde{\tilde{n}}}-\phi^{\tilde{\tilde{m}}}\right\|_{\LtwoS[{\Wkp[{\tilde{K}}]{2}{2}}]},
     \end{align*}
     where in the last inequality we used $\LtwoS[{\Wkp[{\tilde{K}}]{2}{2}}]$-continuity of $\PgSm$ (see \Cref{le:PfuncStronglyConvex}) (and $\tilde{\tilde{n}},\tilde{\tilde{m}}$ again large enough).
Combining the inequalities above, we obtain
\begin{align*}
    &\Ltr\left(\frac{f^{\tilde{\tilde{n}}}+f^{\tilde{\tilde{m}}}}{2}\right)+\lambda \PgSpm\left(\frac{\phi^{\tilde{\tilde{n}}}+\phi^{\tilde{\tilde{m}}}}{2}\right)<L^*+\lambda P^*+\epsilon_L+\lambda\frac{\epsilon_P+\epsilon_P}{2}-\lambda\frac{M}{8}\left\|\phi^{\tilde{\tilde{n}}}-\phi^{\tilde{\tilde{m}}}\right\|_{\LtwoS[{\Wkp[{\tilde{K}}]{2}{2}}]}.
\end{align*}
Thus, 
\[\left\|\phi^{\tilde{\tilde{n}}}-\phi^{\tilde{\tilde{m}}}\right\|_{\LtwoS[{\Wkp[{\tilde{K}}]{2}{2}}]}<\frac{8}{M}\left(\frac{1}{\lambda}\epsilon_L+\epsilon_P\right)\]
for $\tilde{\tilde{n}},\tilde{\tilde{m}}$ large enough in order not to violate optimality. Thus, $\phi^{\tilde{\tilde{n}}}$ is a $\LtwoS[{\Wkp[{\tilde{K}}]{2}{2}}]$-Cauchy sequence, since we can choose $\epsilon_L$ and $\epsilon_P$ arbitrarily small.

Since $\phi^{\tilde{\tilde{n}}}$ is a $\LtwoS[{\Wkp[{\tilde{K}}]{2}{2}}]$-Cauchy sequence, it converges to $\tilde{\phi}$ in $\LtwoS[{\Wkp[{\tilde{K}}]{2}{2}}]$.
Note that $\left(\tilde{f}_+,\tilde{f}_-\right)\in\T$, since $\T$ is closed w.r.t.~$\Wkp[{\tilde{K}}]{2}{2}$.
This, together with $\Wkp[{\tilde{K}}]{2}{2}$-continuity of $\Ltr$ and $\PgSpm$ implies
\[\min_{\T}F^\phi =\lim_{\tilde{\tilde{n}}\to\infty}\Ltr\left(f^{\tilde{\tilde{n}}}\right)+\lambda\PgSpm\left(\phi^{\tilde{\tilde{n}}}\right)=F^\phi\left(\tilde{\phi}\right),\]
i.e., $\tilde{\phi}\in\argmin_{\T}F^\phi$.
This however is a contradiction to \cref{eq:farAwayFromArgMin} for $\tilde{\tilde{n}}$ large enough.
\end{proof}

\begin{lemma}\label{le:PfuncStronglyConvex}

 Let $\tilde{K}\subset\R$ be a compact interval and consider the Banach space \newline$(\T,||\cdot||_{\LtwoS[{\Wkp[{\tilde{K}}]{2}{2}}]})$. The penalty term of $F^\phi$ from \Cref{def:Philevelproblem}, given by $\PgSpm:\T\to\Rpz$,
	\begin{equation}
	\PgSpm\left(\phi\right)=\gbar \int_{\Sd}\int_{\supp (\gs)} \frac{\twonorm[{{\phi_{s}(r)}^{''} }]^2}{\gsr{s}} \dx[r]\dx[s]
	\end{equation}
	is strongly convex w.r.t. $||\cdot||_{\LtwoS[{\Wkp[{\tilde{K}}]{2}{2}}]}$. Moreover, if $\supp(g)\subset \tilde{K}$, then $\PgSpm$ is continuous w.r.t. $||\cdot||_{\LtwoS[{\Wkp[{\tilde{K}}]{2}{2}}]}$.
\end{lemma}
\begin{proof}
    Let $\tilde{K}\subset\R$ be a compact interval with diameter $|\tilde{K}|$, $\phi^1,\phi^2\in\T$ and define
	\begin{equation}
	u:=\phi^1-\phi^2\in\T
	\end{equation} as the component-wise difference. Note that $t\phi^1+(1-t)\phi^2\in\T$ for every $t\in[0,1]$. Since $\PgSpm$ is a quadratic form, we get for any $t\in[0,1]$ with the help of some algebraic calculations
	\begin{equation}\label{eq:stronglyconvexP}
	\PgSpmb{t\phi^1+(1-t)\phi^2}=
	t\PgSpmb{\phi^1}+(1-t)\PgSpmb{\phi^2}-t(1-t)\PgSpmb{u}.
	\end{equation}
	Moreover, we have
	\begin{align}
	\LtwonormonSd[{u}^{''}]{ \Ltwo[\tilde{K}]}^2\leq\frac{max_{x\in\supp(g)}g(x)}{ \gbar}\PgSpmb{u},
	\end{align}
	since $u\in\T$ has zero second derivative outside $\supp(g)\subset\tilde{K}$.
	Applying the Poincar\'e-typed \Cref{le:poincare} twice (item \eqref{item:PoincareL2L2} first on ${u}^{'}$ then on ${u}^{}$) yields 
	\begin{align}
	\LtwonormonSd[u^{'}]{\Ltwo[\tilde{K}]}^2&\le |\tilde{K}|^2\frac{max_{x\in\supp(g)}g(x)}{ \gbar}\PgSpmb{u},\\
	\LtwonormonSd[u]{\Ltwo[\tilde{K}]}^2&\le |\tilde{K}|^4\frac{max_{x\in\supp(g)}g(x)}{ \gbar}\PgSpmb{u},
	\end{align}
	as $u\in\T$ satisfies the boundary conditions at $\Cgl$ (cp. \Cref{rem:compactSupp}) because of the compact support of $g$.
   Hence
	\begin{align}
	\left\|u\right\|_{\LtwoS[{\Wkp[{\tilde{K}}]{2}{2}}]}^2&\le \left(|\tilde{K}|^4+|\tilde{K}|^2+1\right)\frac{max_{x\in\supp(g)}g(x)}{ \gbar}\PgSpmb{u}.\label{eq:W22leP}
	\end{align}
	Combining this with \meqref{eq:stronglyconvexP} results in
		\begin{align}\label{eq:proofStrongyConvexFInalConstant}
	\PgSpmb{t\phi^1+(1-t)\phi^2}
	\leq& 
	t\PgSpmb{\phi^1}+(1-t)\PgSpmb{\phi^2}\\
	&
	-\frac{t(1-t) \gbar}{\left(|\tilde{K}|^4+|\tilde{K}|^2+1\right)\max_{x\in\supp(g)}g(x)}\left\|u\right\|_{\LtwoS[{\Wkp[{\tilde{K}}]{2}{2}}]}^2.
	\end{align}
	Thus, $\PgSpm$ is strongly convex with parameter $\gbar/\left(|\tilde{K}|^4\max_{x\in\supp(g)}g(x)\right)$.
	To show continuity, note that we have
	\begin{align}\label{eq:upperboundonPOU}
	\PgSpmb{u}&\leq\frac{ \gbar}{\min_{x\in\supp(g)}g(x)}\LtwonormonSd[{u}]{\Wkp[{\tilde{K}}]{2}{2}}^2
	\end{align}
	since $u\in\T$ has zero second derivative outside $\supp(g)$ and $\supp(g)\subset \tilde{K}$. Let $\phi^1\in\T$. Then, for any $\phi^2\in\T$
	\begin{subequations}
	\begin{align}
	    |\PgSpmb{\phi^2}-\PgSpmb{\phi^1}|&=\gbar 
	\int_{\Sd}\int_{\supp (g)} \frac{\left( {\phi_s^{2^{''}}-\phi_s^{1^{''}}} \right)\left( {\phi_s^{2^{''}}+\phi_s^{1^{''}}} \right)(x)}{g(x)} \,dx\, ds
	\\
	&\le \sqrt{\PgSpm\left(\phi^2-\phi^1\right)}\sqrt{\PgSpm\left(\phi^2+\phi^1\right)}\\
	&\le \sqrt{\PgSpmb{u}}\sqrt{8\PgSpmb{\phi^1}+2\PgSpmb{u}}\\
	&\le \sqrt{\frac{ 8\gbar\PgSpmb{\phi^1}}{\min_{x\in\supp(g)}g(x)}}\LtwonormonSd[{u}]{\Wkp[{\tilde{K}}]{2}{2}}\\
	&\quad+\frac{ 2^{\frac{1}{2}}\gbar}{\min_{x\in\supp(g)}g(x)}\LtwonormonSd[{u}]{\Wkp[{\tilde{K}}]{2}{2}}^2\\
	&=c_{\phi^1}\LtwonormonSd[{u}]{\Wkp[{\tilde{K}}]{2}{2}}+c\LtwonormonSd[{u}]{\Wkp[{\tilde{K}}]{2}{2}}^2
	\end{align}
	\end{subequations}
	with positive constants $c_{\phi^1}$ and $ c$, where we employed the Cauchy-Schwarz inequality and \meqref{eq:upperboundonPOU}.
	Thus, for any $\epsilon>0$ we achieve $	|\PgSpmb{\phi^2}-\PgSpmb{\phi^1}|<\epsilon$ for every $\phi^2$ such that 
	\[||\phi^2-\phi^1||_{\LtwoS[{\Wkp[{\tilde{K}}]{2}{2}}]}<\delta\]
	with $\delta := \sqrt{\frac{c_{\phi^1}^2}{4c^2}+\frac{\epsilon}{c}}-\frac{c_{\phi^1}}{2c}$. Hence we have shown continuity at an arbitrary $\phi^1\in\T$.

\end{proof}

\subsubsection{Proof of \Cref{thm:ridgeToaIGAM}}\label{subsubsec:proof:ridgeToaIGAM}
\begin{definition}[FD $\ell_2$-regularized RSN]\label{def:FDridgeNet}
 Let 
 $\lw \in \Rp$, $n\in\mathbb{N}$ and let $\RNRo[\wR]$ be the corresponding \ridgeRSN\ of width $n$, with first-layer parameters $(b_k,v_k)\omb\in\R^{\din+1}$ and kink positions respectively kink directions $\xi_k\omb$ and $s_k\omb$, $k=1,\ldots,n$. Furthermore, consider a set $\Shd\subset\Sd$ of $|\Shd|=n$ finite directions and denote by $\mathcal{U}:=\{U(\sh)\}_{\sh\in\Shd}$ a collection of disjoint environments such that $\dot\bigcup_{\sh\in\Shd}U(\sh)=\Sd$. Define for every $k=1,\ldots,n$ the vectors $\vw_k\in\Rdin$ as
 \[\vw_k\omb:=\sum_{\sh\in\Shd}\sh\twonorm[v_k]\ind_{\Ush}(s_k)\omb.\]
 The \emph{finite direction $\ell_2$-regularized \RSN\ (FD $\ell_2$-regularized RSN)} \FDRNR\ is then defined as the \ridgeRSN\ obtained when the first-layer parameters are given by $(b_k,\vw_k)\omb, k=1,\ldots,n$, i.e.
 \begin{equation}
     \FDRNRo(x):=\linkinv\left(\sum_{k=1}^{n}\wwstar_k\omb\,\sigma\left({b_k\omb}+\sum_{j=1}^{d}{\vw_{k,j}\omb}x_j\right) \right)\quad \faxdg\mycomma
 \end{equation}
 	with $\wwstar$ such that
		\begin{equation}\label{eq:wwstar}
		\wwstar\omb :\in \argmin_{w\in\R^{\dout\times n}}\underbrace{\Ltr\left( \FDRNRo\right) 
		+\lw||w||_2^2}_{\Fnb{\FDRNRo}}\mydot
		\end{equation}
\end{definition}

\begin{remark}\label{rem:FDRSNvnorm}
    The \FDridgeRSN\ w.r.t. an RSN and a certain partition $\U$ of $\Sd$ is obtained by shifting the first-layer parameters $v_{k,\cdot}\omb\in\Rdin$, $k=1,\ldots,n$ towards the nearest mid-points $\vw_{k,\cdot}\in\Rdin$ of $\U$ and optimizing the terminal layer parameters of the resulting network via $\ell_2$-regularized regression.
    In particular, due to
    \begin{align*}
    \twonorm[\vw_k]&=\twonorm[{\sum_{\sh\in\Shd}\sh\twonorm[v_k]\ind_{\Ush}(s_k)}]=\sum_{\sh\in\Shd}\underbrace{\twonorm[\sh]}_{=1}\twonorm[v_k]\ind_{\Ush}(s_k)\\
    &=\twonorm[v_k]\sum_{\sh\in\Shd}\ind_{\Ush}(s_k)=\twonorm[v_k],\quad k=1,\ldots,n,
\end{align*}
    the norms of the first layer parameters remain unchanged.
\end{remark}
\begin{proof}[\hypertarget{proof:thm:ridgeToaIGAM}{Proof of \Cref{thm:ridgeToaIGAM}}]
With the notation from \Cref{le:aGAMtoaIGAM,le:A} we have
\begin{align*}
     d_{\Wkp[K,\Rdout]{1}{\infty}}\left( \RNR, \argmin\FlSm\right) \le&   d_{\Wkp[K,\Rdout]{1}{\infty}}\left( \RNR,\FDRNR\right)\\
    &+d_{\Wkp[K,\Rdout]{1}{\infty}}\left(\FDRNR,\argmin\FlShm\right)\\
    &+d_{\Wkp[K,\Rdout]{1}{\infty}}\left(\argmin \FlShm, \argmin\FlSm\right).
\end{align*}
Moreover by \Cref{le:aGAMtoaIGAM} and \Cref{le:A} respectively, it holds that
\begin{align*}
\nlim & d_{\Wkp[K,\Rdout]{1}{\infty}}\left(\argmin \FlShm, \argmin\FlSm\right)=0, \text{ and}\\
    \plim\ & d_{\Wkp[K,\Rdout]{1}{\infty}}\left( \RNR,\FDRNR\right)=0.\\ 
\end{align*}
Applying \cite[Theorem 3.9]{ImplRegPart1V3} for finitely many directions $\tilde{s}_k$, $k=1,\ldots,n$ (and multiple outputs $\dout$) yields 
\begin{align*}
   	\plim d_{\Wkp[K,\Rdout]{1}{\infty}}\left(\FDRNR,
	\argmin\FlSm\right) =0.
\end{align*}
Therefore, the statement \eqref{eq:ridgeToaIGAM} follows.
\end{proof}

\begin{lemma}[A]\label{le:A}
Let $\RNRo=\RNwo[\wR]$ and $\FDRNRo=\FDRNwo[\wwstar]$ be an \ridgeRSN\ respectively an \FDridgeRSN\ of width $n$.
    We further denote the maximal distance of the partition $\mathcal{U}$ (cp. \Cref{def:FDridgeNet}) w.r.t. the Euclidean distance $d$ by
	    \begin{displaymath}
	    |\mathcal{U}|:=\max_{\sh\in\Shd}\left(\sup_{s\in U(\sh) }d(s,\sh)\right).
	    \end{displaymath}
	    Then, if $|\mathcal{U}|\overset{n\to\infty}{\longrightarrow}0$,
	    we have for any compact $K\subset\Rdin$ that
	    \begin{equation}
	    \plim\sobnormmulti[\RNR-\FDRNR]=0.
	    \end{equation}
\end{lemma}
\begin{proof}

An application of the triangle inequality for the $\Woimultidim$-norm yields

	    \begin{align*}
	    \sobnormmulti[{\RNR-\FDRNR}]&=\sobnormmulti[{\RNw[\wR]-\FDRNw[\wwstar]}]\\
	    &\le \sobnormmulti[{\RNw[\wR]-\FDRNw[\wR]}]\\
	    &\quad+\sobnormmulti[{\FDRNw[\wR]-\FDRNw[\wwstar]}],
	    \end{align*}
with $\FDRNw[\wR]$ as defined in \eqref{eq:FDhilfsRSN}. By \Cref{le:A1}, we have $\plim\Frobnorm[\wR-\wwstar]=0$, and thus, by continuity of the map $(\R^{n\times\dout} ,\Frobnorm )\to\Woimultidim: w\mapsto\FDRNw[w]$,
\begin{equation*}
    \plim\sobnormmulti[{\FDRNw[\wR]-\FDRNw[\wwstar]}]=0.
\end{equation*}
Together with \Cref{le:A2}, the claim follows.
\end{proof}

\begin{lemma}\label{le:BoundsOnParameternorms}
In the setting of \Cref{le:A}, we have 
\begin{align}\label{eq:vdifferencebound}
        \Frobnorm[v-\vw]&\overset{\PP}{\lessapprox}\BigO\left(\sqrt{n}|\mathcal{U}|\right),\\ \label{eq:wRbound}
    \Frobnorm[\wR]&\le \frac{\Ltrb{0}}{\lw}\quad \faog,\\ \label{eq:wwstarbound}
    \Frobnorm[\wwstar]&\overset{{\PP}}{\lessapprox}\frac{2\FlSmb{\flS}}{\lw}.
\end{align}
\end{lemma}
\begin{proof}
An application of the law of large numbers yields
\begin{align*}
    \Frobnorm[v-\vw]^2=\sum_{k=1}^n\twonorm[v_{k,\cdot}-\vw_{k,\cdot}]^2=\sum_{k=1}^n\twonorm[s_k-\tilde{s}_k]^2\twonorm[v_k]^2\overset{\PP}{\lessapprox}|\U|^2n\E[{\twonorm[v_k]^2}].
\end{align*}
Since $\E[{\twonorm[v_k]^2}]<\infty$ by \Cref{item:vk:finiteSecondMoment} of \Cref{as:truncatedg}, \eqref{eq:vdifferencebound} follows.
By optimality of $\wR$, we have $\faog$
\begin{align*}
    \lw\Frobnorm[\wR]\le\Fnb{\RNw[\wR]}\le\Fnb{\RNw[0]}=\Ltrb{0},
\end{align*}
and therefore 
\begin{align*}
    \Frobnorm[\wR]\le \frac{\Ltrb{0}}{\lw}.
\end{align*}
Moreover, applying \cite[Lemma A.17]{ImplRegPart1V3} for finitely many directions $\tilde{s}_k$, $k=1,\ldots,n$ yields $\Fnb{\FDRNw[\wwstar]}{\overset{\underset{n\to\infty}{\PP}}{\approx}}\FlShmb{\flSh}$. (Here, the notation $\overset{\underset{n\to\infty}{\PP}}{\approx}$ corresponds to a mathematically proved exact limit in probability.) Using \Cref{le:4}, we then have
\begin{align*}
    \lw\Frobnorm[\wwstar]\le\Fnb{\FDRNw[\wwstar]}{\overset{\underset{n\to\infty}{\PP}}{\approx}}\FlShmb{\flSh}\underset{\text{\Cref{le:4}}}{\overset{n\to\infty}{\approx}}\FlSmb{\flS},
\end{align*}
which implies
\begin{align*}
    \Frobnorm[\wwstar]\overset{{\PP}}{\lessapprox}\frac{2\FlSmb{\flS}}{\lw}.
\end{align*}
\end{proof}
\begin{lemma}[A.1]\label{le:A1}
In the setting of \Cref{le:A}, we have for any $\lw>0$
	    \begin{equation}
	    \plim\Frobnorm[\wR-\wwstar]=0,
	    \end{equation}
\end{lemma}
with $\wR$ and $\wwstar$ the optimal terminal-layer parameters of \RNR\ and \FDRNR\ respectively. 
\begin{proof}
Since $\wR$ and $\wwstar$ are optimal terminal-layer parameters, both \RNR\ and \FDRNR\ are in the sub-level set $\mathcal{L}:=\{f:\Ltr(f)<\Ltr(0)+\epsilon\}$, for any $\epsilon>0$.
Thus, by \Cref{as:generalloss}, there exists a compact set $K\subset\Rdin$ s.t. $\Ltr$ is  $\Wkp[K,\nu]{1}{p}$-Lipschitz at \RNR\ and \FDRNR. Moreover by \Cref{as:generalloss}, $\Ltr$ is continuous w.r.t. $\Wkp[K,\nu]{1}{p}$. Therefore, for $|\mathcal{U}|$ small enough, i.e., for $n$ large enough, also for the \FDridgeRSN\ $\FDRNw[{\wR}]$ it holds that $\FDRNw[{\wR}]\subset\mathcal{L}$. To see this, note that since
 $\sigma(\cdot)=\max\{\cdot,0\}$ is Lipschitz continuous with constant 1, by Cauchy-Schwarz we have for any $w\in\R^{\dout\times n}$
 \begin{align*}
    \sobnormop[{\RNw[{w}]-\FDRNw[{w}]}]
        =&\sobnormop[\sum_{k=1}^{n}w_{\cdot,k}\left(\sigma\left({b_k}+\langle v_{k},\cdot\rangle \right)-\sigma\left({b_k}+\langle\vw_{k},\cdot\rangle\right)\right) ]\\
        \le& \sum_{k=1}^{n}\twonorm[{w_{\cdot,k}}]\sobnormop[\skprod{v_{k,\cdot}-\vw_{k,\cdot}}{\cdot}]\\
         \le& \sum_{k=1}^{n}\twonorm[{w_{\cdot,k}}]2\nu(K)^{\frac{1}{p}}\sum_{j=1}^\dout\sobnorm[\skprod{v_{k,j}-\vw_{k,j}}{\cdot}]\\
        \le & \left(\max_{x\in K}\twonorm[{x}]+1\right)\Frobnorm[w]\Frobnorm[{v-\vw}].
\end{align*}
Now since we assume $\lim_{n\to\infty}|U|=0$, we have that $\lim_{n\to\infty}\Frobnorm[{v-\vw}]=0$, and therefore for any $\delta>0$ there exists $n$ large enough such that \[ \sobnormop[{\RNw[{\wR}]-\FDRNw[{\wR}]}]<\delta \quad \text{and}\quad \sobnormop[{\RNw[{\wwstar}]-\FDRNw[{\wwstar}]}]<\delta.\]
Thus, by continuity of $\Ltr$, there $\FDRNw[\wR]\in\mathcal{L}$ for $n$ large enough. With $\Wkp[K,\nu]{1}{p}$-Lipschitz continuity of $\Ltr$ on $\mathcal{L}$ we further get
for a constant $C^K>0$,
\begin{equation}\label{eq:lossdiffbound}
     \left|\Ltrb{{\RNw[{w}]}}-\Ltrb{{\FDRNw[{w}]}}\right|\le C^K\Frobnorm[w]\Frobnorm[{v-\vw}], \text{ for } w\in\{\wR,\wwstar\}.
\end{equation}
These upper bounds in \cref{eq:lossdiffbound} can be further refined (cp. \Cref{le:BoundsOnParameternorms}). Indeed, 
combining equations \eqref{eq:lossdiffbound} and \eqref{eq:wwstarbound} we obtain
\begin{align}\label{eq:bound1}
    \Fnb{\RNw[\wwstar]}\overset{{\PP}}{\lessapprox}\Fnb{\FDRNw[\wwstar]}+C^N\frac{2\FlSmb{\flS}}{\lw}\Frobnorm[{v-\vw}].
\end{align}
Moreover, by $2\lw$-\href{https://en.wikipedia.org/wiki/Convex_function#Strongly_convex_functions}{strong convexity} of $w\mapsto\Fnb{\FDRNw[w]}$, we have due to the optimality of $\wwstar$
\begin{equation}\label{eq:strongconvbound}
    \Fnb{\FDRNw[w]}\ge\Fnb{\FDRNw[\wwstar]}+\lw\Frobnorm[w-\wwstar]^2,\quad \forall w\in\R^{\dout\times n}.
\end{equation}
Therefore, together with equations \eqref{eq:lossdiffbound}, \eqref{eq:wRbound} it follows that
\begin{align}\label{eq:bound2}
    \Fnb{\RNw[\wR]}&\overset{\text{\tiny\eqref{eq:lossdiffbound},\eqref{eq:wRbound}}}{\ge}\Fnb{\FDRNw[\wR]}-C^K\frac{\Ltrb{0}}{\lw}\Frobnorm[{v-\vw}]\nonumber\\
    &\underset{\hphantom{\overset{\text{\tiny\eqref{eq:lossdiffbound},\eqref{eq:wRbound}}}{\ge}}}{\overset{\text{\tiny\eqref{eq:strongconvbound}}}{\ge}} \Fnb{\FDRNw[\wwstar]}+\lw\Frobnorm[\wR-\wwstar]^2-C^K\frac{\Ltrb{0}}{\lw}\Frobnorm[{v-\vw}].
\end{align}
Thus overall, it follows that
\begin{align*}
    0<\Fnb{\RNw[\wwstar]}-\Fnb{\RNw[\wR]}\overset{\text{\tiny\eqref{eq:bound1},\eqref{eq:bound2}}}{\overset{{\PP}}{\lessapprox}}\frac{C^K}{\lw}\Frobnorm[{v-\vw}]\left(\Ltrb{0}+2\FlSmb{\flS}\right)-\lw\Frobnorm[\wR-\wwstar]^2,
\end{align*}
which by \Cref{le:BoundsOnParameternorms} implies
\begin{align*}
 \Frobnorm[\wR-\wwstar]^2\overset{{\PP}}{\lessapprox}\frac{C^K}{\lw}\Frobnorm[{v-\vw}]\left(\Ltrb{0}+2\FlSmb{\flS}\right)\overset{\substack{\text{\tiny\eqref{eq:vdifferencebound}}\\{\PP}}}{\lessapprox} \BigO\left(\frac{|\mathcal{U}|}{\sqrt{n}}\right).
\end{align*}

\end{proof}
\begin{lemma}[A.2]\label{le:A2}
In the setting of \Cref{le:A}, define the RSN $\FDRNw[\wR]$ as
\begin{equation}\label{eq:FDhilfsRSN}
    \FDRNw[\wR](x):= \linkinv\left(\sum_{k=1}^{n}\wR_{\cdot,k}\omb\,\sigma\left({b_k\omb}+\sum_{j=1}^{d}{\vw_{k,j}\omb}x_j\right) \right)\quad \faxdg\mydot
\end{equation}
	    Then,
	    we have for any compact $K\subset\Rdin$ that
	    \begin{equation}
	    \plim\sobnormmulti[{\RNR[\wR]-\FDRNw[\wR]}]=0.
	    \end{equation}
\end{lemma}
\begin{proof}
Analogously to the proof of \cite[Lemma A.23]{ImplRegPart1V3} it follows that
\[\nlim\supnormmulti[{\RNR[\wR]-\FDRNw[\wR]}]=0 \quad \faog\mycomma\]
as $\linkinv$ is required to be Lipschitz continuous by \Cref{def:linkfunction} and $|\mathcal{U}|\overset{n\to\infty}{\longrightarrow}0$. In particular,
\begin{equation}\label{eq:zeroderconv}
    \plim\supnormmulti[{\RNR[\wR]-\FDRNw[\wR]}]=0.
\end{equation}
It remains to show
\[\plim\max_{i=1,\dots,\din}\supnorm[{\jacj[{\RNR[\wR]}]_i-\jacj[{\FDRNw[\wR]}]_i}]=0,\]
for every $j\in\{1,\ldots,\dout\}$, where $\jacj[f]$ represents the gradient of the j\textsuperscript{th} component $f^j:~\R^m\to~\R$ of a differentiable function $f:\R^m\to\Rdout$ ($m=\din$ in this case). (Accordingly, $\jacj[f]_i$ represents the i\textsuperscript{th} partial derivative of $f^j$.) Let $j\in\{1,\ldots,\dout\}, i\in\{1,\ldots,\din\}$ be fixed.
We have by the chain rule that
\begin{align*}
    \jacj[{\RNR[\wR]}]_i& =\jacjat[\linkinv]{\RNR[\wR]}\sum_{k=1}^n\wR_{\cdot,k} \sigma^{'}\left(\skprod{v_{k,\cdot}}{x}+b_k\right)v_{k,i}, \\
    \jacj[{\FDRNw[\wR]}]_i&=\jacjat[\linkinv]{\FDRNw[\wR]}\sum_{k=1}^n\wR_{\cdot,k} \sigma^{'}\left(\skprod{\vw_{k,\cdot}}{x}+b_k\right)\vw_{k,i}.
\end{align*}
(Here, $v_{\cdot,k}$ and $v_{k,\cdot}$ denote the k\textsuperscript{th} column and row respectively of a real-valued matrix $v$ with corresponding dimensions. Even though $v_{k,\cdot}$ is a row vector, we keep writing $\skprod{v_{k,\cdot}}{x}$ for the scalar product with a vector $x$ instead of $v_{k,\cdot}x$.) By Lipschitz continuity of $\linkinv$, we have
\begin{align*}
    \supnorm[{\jacjat[\linkinv]{\RNR[\wR]}-\jacjat[\linkinv]{\FDRNw[\wR]}}]\le L^{\linkinv}\supnormmulti[{\RNR[\wR]-\FDRNw[\wR]}].
\end{align*}
Thus, with \eqref{eq:zeroderconv}
\begin{equation}
    \plim\supnorm[{\jacjat[\linkinv]{\RNR[\wR]}-\jacjat[\linkinv]{\FDRNw[\wR]}}]=0.
\end{equation}
Note that by \Cref{rem:FDRSNvnorm}
we have $\xi_k=-b_k/\twonorm[v_k]=-b_k/\twonorm[\vw_k]=:\tilde{\xi}_k$, for all $k=1,\ldots,n$, i.e., the kink positions of $\RNR[\wR]$ and $\FDRNw[\wR]$ coincide. Therefore, it follows that
\begin{align}\label{eq:upperboundoninnerderivative}
    &\supnorm[{\sum_{k=1}^n\wR_{\cdot,k} \sigma^{'}\left(\skprod{v_{k,\cdot}}{\cdot}+b_k\right)v_{k,i}-\sum_{k=1}^n\wR_{\cdot,k} \sigma^{'}\left(\skprod{\vw_{k,\cdot}}{\cdot}+b_k\right)\vw_{k,i}}]\nonumber\\
    =&\sup_{x\in K}\twonorm[{\sum_{\substack{k:\xi_k<\skprod{s_k}{x}\\
    \wedge\xi_k<\skprod{\tilde{s}_k}{x}}}\wR_{\cdot,k} \left(v_{k,i}-\vw_{k,i}\right)+\sum_{\substack{k:\xi_k<\skprod{s_k}{x}\\
    \wedge\xi_k>\skprod{\tilde{s}_k}{x}}}\wR_{\cdot,k} v_{k,i}+\sum_{\substack{k:\xi_k>\skprod{s_k}{x}\\
    \wedge\xi_k<\skprod{\tilde{s}_k}{x}}}\wR_{\cdot,k} \vw_{k,i}}]\nonumber\\
    \le&\sum_{k=1}^n \twonorm[\wR_{\cdot,k}]\left|v_{k,i}-\vw_{k,i}\right|+\sup_{x\in K}\left(\sum_{\substack{k:\xi_k<\skprod{s_k}{x}\\
    \wedge\xi_k>\skprod{\tilde{s}_k}{x}}}\twonorm[\wR_{\cdot,k}] |v_{k,i}|+\sum_{\substack{k:\xi_k>\skprod{s_k}{x}\\
    \wedge\xi_k<\skprod{\tilde{s}_k}{x}}}\twonorm[\wR_{\cdot,k}] |\vw_{k,i}|\right)\nonumber\\
    \le& \Frobnorm[\wR]\Frobnorm[v-\vw]+\Frobnorm[\wR]\left(\sup_{x\in K}\sum_{\substack{k:\xi_k<\skprod{s_k}{x}\\
    \wedge\xi_k>\skprod{\tilde{s}_k}{x}}}\twonorm[v_k]^2+\sup_{x\in K}\sum_{\substack{k:\xi_k>\skprod{s_k}{x}\\
    \wedge\xi_k<\skprod{\tilde{s}_k}{x}}}\twonorm[v_k]^2\right).
\end{align}
By \Cref{le:BoundsOnParameternorms}, we have
\begin{align}
    \Frobnorm[\wR]\Frobnorm[v-\vw]\overset{\PP}{\lessapprox} \frac{\Ltrb{0}}{\lw}\BigO\left(\sqrt{n}|\mathcal{U}|\right)=\BigO\left(\frac{|\mathcal{U}|}{\sqrt{n}}\right),
\end{align}
and hence the first summand vanishes for $n\to\infty$.
With the law of large numbers it follows that
\begin{align}
    \sum_{\substack{k:\xi_k<\skprod{s_k}{x}\\
    \wedge\xi_k>\skprod{\tilde{s}_k}{x}}}\twonorm[v_k]^2\overset{\PP}{\approx}n\E[{\twonorm[v_k]^2\ind_{\xi_k:\skprod{\tilde{s}_k}{x}<\xi_k<\skprod{s_k}{x}}}]. 
\end{align}
Moreover, 
\begin{align}
\plim\ind_{\xi_k:\skprod{\tilde{s}_k}{x}<\xi_k<\skprod{s_k}{x}}=0,
\end{align}
uniformly in $x$. 
Note that by \Cref{item:vk:finiteSecondMoment} of \Cref{as:truncatedg} $v_k\in L^2$, and hence we have
\begin{align}
    \nlim\frac{1}{n}\sup_{x\in K} \sum_{\substack{k:\xi_k<\skprod{s_k}{x}\\
    \wedge\xi_k>\skprod{\tilde{s}_k}{x}}}\twonorm[v_k]^2=\nlim\sup_{x\in K}\E[{\twonorm[v_k]^2\ind_{\xi_k:\skprod{\tilde{s}_k}{x}<\xi_k<\skprod{s_k}{x}}}]=0.
\end{align}
Together with \Cref{le:BoundsOnParameternorms}, this implies
\begin{align*}
    \nlim\Frobnorm[\wR]\sup_{x\in K}\sum_{\substack{k:\xi_k<\skprod{s_k}{x}\\
    \wedge\xi_k>\skprod{\tilde{s}_k}{x}}}\twonorm[v_k]^2\overset{\text{\tiny \eqref{eq:wRbound}}}{\le}\nlim\frac{\Ltrb{0}}{\lw}\sup_{x\in K}\sum_{\substack{k:\xi_k<\skprod{s_k}{x}\\
    \wedge\xi_k>\skprod{\tilde{s}_k}{x}}}\twonorm[v_k]^2=0.
\end{align*}
Analogously,
\begin{align*}
    \nlim\Frobnorm[\wR]\sup_{x\in K}\sum_{\substack{k:\xi_k<\skprod{\tilde{s}_k}{x}\\
    \wedge\xi_k>\skprod{s_k}{x}}}\twonorm[v_k]^2\overset{\text{\tiny \eqref{eq:wRbound}}}{\le}\nlim\frac{\Ltrb{0}}{\lw}\sup_{x\in K}\sum_{\substack{k:\xi_k<\skprod{\tilde{s}_k}{x}\\
    \wedge\xi_k>\skprod{s_k}{x}}}\twonorm[v_k]^2=0.
\end{align*}
With every summand in the upper bound \eqref{eq:upperboundoninnerderivative} converging to zero, we obtain
\begin{align}\label{eq:innerderivconv}
    \plim  \supnorm[{\sum_{k=1}^n\wR_{\cdot,k} \sigma^{'}\left(\skprod{v_{k,\cdot}}{\cdot}+b_k\right)v_{k,i}-\sum_{k=1}^n\wR_{\cdot,k} \sigma^{'}\left(\skprod{\vw_{k,\cdot}}{\cdot}+b_k\right)\vw_{k,i}}]=0.
\end{align}
Hence, by \eqref{eq:zeroderconv} and \eqref{eq:innerderivconv} we have
\[\plim\max_{i=1,\dots,\din}\supnorm[{\jacj[{\RNR[\wR]}]_i-\jacj[{\FDRNw[\wR]}]_i}]=0.\]
\end{proof}

\end{document}